\documentclass[10pt,journal,compsoc]{IEEEtran}

\ifCLASSOPTIONcompsoc
  \usepackage[nocompress]{cite}
\else
  \usepackage{cite}
\fi

\ifCLASSINFOpdf
  \usepackage[pdftex]{graphicx}
\else
  \usepackage[dvips]{graphicx}
\fi
\usepackage{amsmath,amssymb,amsthm}
\ifCLASSOPTIONcompsoc
 \usepackage[caption=false,font=footnotesize,labelfont=sf,textfont=sf]{subfig}
\else
 \usepackage[caption=false,font=footnotesize]{subfig}
\fi
\usepackage{url}

\usepackage{graphicx}
\usepackage{bm}
\usepackage{bbm}
\usepackage{booktabs}
\usepackage[dvipsnames]{xcolor}
\usepackage{xfrac}
\usepackage{ragged2e} 
\usepackage{titletoc}
\usepackage{mathrsfs}
\usepackage{enumitem}
\usepackage{multirow}
\usepackage{pifont}
\usepackage{algorithm,algorithmic}
\usepackage{thm-restate}
\usepackage{makecell}
\usepackage{soul}
\usepackage{tocvsec2}
\usepackage{colortbl}  

\newcommand{\cmark}{\ding{51}}  
\newcommand{\xmark}{\ding{55}}  

\newcommand{\E}{\mathbb{E}}

\newcommand{\norm}[1]{\left\lVert#1\right\rVert}
\newcommand{\abs}[1]{\left|#1\right|}

\newcommand{\indicator}[1]{\mathbbm{1}_{#1}}

\newcommand{\setX}{\mathcal{X}}
\newcommand{\setY}{\mathcal{Y}}
\newcommand{\setD}{\mathcal{D}}
\newcommand{\setZ}{\mathcal{Z}}
\newcommand{\setN}{\mathcal{N}} 
\newcommand{\setP}{\mathcal{P}} 
\newcommand{\setF}{\mathcal{F}} 

\newcommand{\tbeta}{{\tilde{\beta}}}

\def \np {n_+}
\def \nn {n_-}
\def \npa {n_+^\alpha}
\def \nnb {n_-^\beta}
\def \tp {\mathrm{TPAUC}}
\def \op {\mathrm{OPAUC}}
\def \ehat {\hat{\E}}
\def \cmins {(a,b) \in [0,1]^2}
\def \cmin {f, \cmins}

\def \cminl {\cmin, s' \in \Omega_{s'}}
\def \cmax {\gamma \in \Omega_\gamma}
\def \efb {\eta_\beta(f)}
\def \efa {\eta_\alpha(f)}
\def \hefb {\hat{\eta}_\beta(f)}
\def \hefa {\hat{\eta}_\alpha(f)}
\def \ezdz {\underset{\bm{z}\sim \setD_{\setZ}}{\E}}
\def \hezds {\underset{\bm{z}\sim S}{\ehat}}
\def \exdp {\underset{\bm{x}\sim \setD_{\setP}}{\E}}
\def \exdn {\underset{\bm{x}'\sim \setD_{\setN}}{\E}}

\newcommand{\colblue}[1]{{\color{NavyBlue}#1}}
\newcommand{\colbit}[1]{{\color{Bittersweet}#1}}

\newcommand{\ie}{\textit{i.e.}}
\newcommand{\eg}{\textit{e.g.}}

\newcommand{\stt}{\textit{s.t.}}
\newcommand{\wrt}{\textit{w.r.t.}}

\usepackage[breakable,skins]{tcolorbox}

\newtcolorbox{modbox}{
    colframe=cyan!5!white,
    colback =cyan!5!white,
    top=0mm, bottom=0mm, left=1mm, right=1mm,
    arc=1mm,
    fontupper=\color{blue!70!black},
    fonttitle=\bfseries\color{blue!70!black},
    notitle,
    }

\theoremstyle{plain}
\newtheorem{thm}{Theorem}
\newtheorem{lem}{Lemma}
\newtheorem{prop}{Proposition}
\newtheorem{coro}{Corollary}

\newtheorem{rem}{Remark}

\theoremstyle{definition}
\newtheorem{defn}{Definition}

\newcommand{\Eqref}[1]{Eq.~\eqref{#1}}

\newcommand{\thmref}[1]{Thm.~\ref{#1}}

\definecolor{ballblue}{HTML}{338EA7}
\definecolor{darkblue}{HTML}{183D5E}
\definecolor{lightseagreen}{HTML}{759D39}
\definecolor{lightred}{HTML}{DD7769}
\definecolor{softred}{HTML}{FE8A71}
\definecolor{softblue}{HTML}{63ACE5}

\definecolor{orgin}{RGB}{254, 138, 113}
\definecolor{second}{RGB}{99, 172, 229}
\definecolor{org}{HTML}{F8A145}
\definecolor{blu}{HTML}{63ACE5}
\definecolor{ivy}{HTML}{FFFFF0}

\newcommand{\best}[1]{\textcolor{orgin}{\textbf{#1}}}
\newcommand{\secbest}[1]{\textcolor{second}{\underline{#1}}}

\usepackage[colorlinks,linkcolor=ballblue,citecolor=ballblue]{hyperref}

\begin{document}
%
\title{Closing the Approximation Gap of\\ Partial AUC Optimization: \\A Tale of Two Formulations}

%
%
%

\author{Yangbangyan~Jiang,
	Qianqian~Xu*,~\IEEEmembership{Senior Member,~IEEE,}
	Huiyang~Shao, Zhiyong~Yang,\\
	Shilong~Bao,
	Xiaochun~Cao,~\IEEEmembership{Senior Member,~IEEE,}
	and~Qingming~Huang*,~\IEEEmembership{Fellow,~IEEE}
\IEEEcompsocitemizethanks{
\IEEEcompsocthanksitem Yangbangyan Jiang, Zhiyong Yang and Shilong Bao are with the School of Computer Science and Technology, University of Chinese Academy of Sciences, Beijing 101408, China (E-mail: \{jiangyangbangyan,yangzhiyong21,baoshilong\}@ucas.ac.cn).\protect\\
\IEEEcompsocthanksitem Qianqian Xu is with the Key Laboratory of Intelligent Information Processing, Institute of Computing Technology, Chinese Academy of Sciences, Beijing 100190, China (E-mail: xuqianqian@ict.ac.cn).\protect\\
\IEEEcompsocthanksitem Huiyang Shao is with the ByteDance Inc., Beijing 100086, China (E-mail: dlshaocr@163.com).\protect\\
\IEEEcompsocthanksitem Xiaochun Cao is with School of Cyber Science and Technology, Shenzhen Campus of Sun Yat-sen University, Shenzhen 518107, China (E-mail: caoxiaochun@mail.sysu.edu.cn).\protect\\
\IEEEcompsocthanksitem Qingming Huang is with the School of Computer Science and Technology, University of Chinese Academy of Sciences, Beijing 101408, China, and also with the Key Laboratory of Intelligent Information Processing, Institute of Computing Technology, Chinese Academy of Sciences, Beijing 100190, China (E-mail: qmhuang@ucas.ac.cn).
}
\thanks{Manuscript received September 13, 2024; revised August 27, 2025; accepted November 23, 2025.}
\thanks{This work was supported in part by National Natural Science Foundation of China: 62525212, 62236008, 62441232, 62025604, 62406305, U21B2038, U23B2051, 62476068, and 62471013, in part by Youth Innovation Promotion Association CAS, in part by the Strategic Priority Research Program of the Chinese Academy of Sciences, Grant No. XDB0680201, in part by the Fundamental Research Funds for the Central Universities (E4EQ1101), in part by the China Postdoctoral Science Foundation (CPSF) under Grant No. 2023M743441 and 2025M771492, and in part by the Postdoctoral Fellowship Program of CPSF under Grant No. GZB20240729.}
\thanks{(Corresponding authors: Qianqian Xu and Qingming Huang.)}
}

%
%

\markboth{IEEE Transactions on Pattern Analysis and Machine Intelligence}%
{Jiang \MakeLowercase{\textit{et al.}}: Bare Demo of IEEEtran.cls for Computer Society Journals}
%



\IEEEtitleabstractindextext{%
\begin{abstract}
  \justifying
  As a variant of the Area Under the ROC Curve (AUC), the partial AUC (PAUC) focuses on a specific range of false positive rate (FPR) and/or true positive rate (TPR) in the ROC curve. It is a pivotal evaluation metric in real-world scenarios with both class imbalance and decision constraints. However, selecting instances within these constrained intervals during its calculation is NP-hard, and thus typically requires approximation techniques for practical resolution. Despite the progress made in PAUC optimization over the last few years, most existing methods still suffer from uncontrollable approximation errors or a limited scalability when optimizing the approximate PAUC objectives. In this paper, we close the approximation gap of PAUC optimization by presenting two simple instance-wise minimax reformulations: one with an asymptotically vanishing gap, the other with the unbiasedness at the cost of more variables. Our key idea is to first establish an equivalent instance-wise problem to lower the time complexity, simplify the complicated sample selection procedure by threshold learning, and then apply different smoothing techniques. Equipped with an efficient solver, the resulting algorithms enjoy a linear per-iteration computational complexity w.r.t. the sample size and a convergence rate of $O(\epsilon^{-1/3})$ for typical one-way and two-way PAUCs. Moreover, we provide a tight generalization bound of our minimax reformulations. The result explicitly demonstrates the impact of the TPR/FPR constraints $\alpha$/$\beta$ on the generalization and exhibits a sharp order of $\tilde{O}(\alpha^{-1}\np^{-1} + \beta^{-1}\nn^{-1})$. Finally, extensive experiments on several benchmark datasets validate the strength of our proposed methods.
\end{abstract}

\begin{IEEEkeywords}
AUC Optimization, Partial AUC, Binary Classification, Class Imbalance
\end{IEEEkeywords}}

\maketitle

\IEEEdisplaynontitleabstractindextext

%
\IEEEpeerreviewmaketitle

\IEEEraisesectionheading{\section{Introduction}\label{sec:intro}}

\IEEEPARstart{T}{he} Area Under the Receiver Operating Characteristic (ROC) curve, denoted as AUC, is a pivotal metric summarizing the classifier's performance across various thresholds in terms of True Positive Rate (TPR) and False Positive Rate (FPR) \cite{J1982The}. The insensitivity towards the class imbalance has positioned it as a widely used performance measure over imbalance data \cite{J1982The, yang2022auc, shi2025llmformer,Qiu_2020_CVPR,DBLP:journals/pami/TangSQLWYJ17,DBLP:journals/pami/LiTM19,Qiu_2021_ICCV,wang2025enhanced}. Accordingly, the AUC optimization problem has garnered significant interest within the machine learning community to improve the AUC performance \cite{graepel2000large, cortes2003auc, yan2003optimizing, joachims2005support, ying2016stochastic, yang2021deep, yuan2021compositional}, and been applied to a spectrum of real-world applications such as financial fraud detection \cite{huang2022auc}, spam detection \cite{narasimhan2013structural}, and medical diagnosis \cite{yang2022auc, yuan2021large}.

\begin{figure}[!t]
  \centering
  \includegraphics[width=\columnwidth]{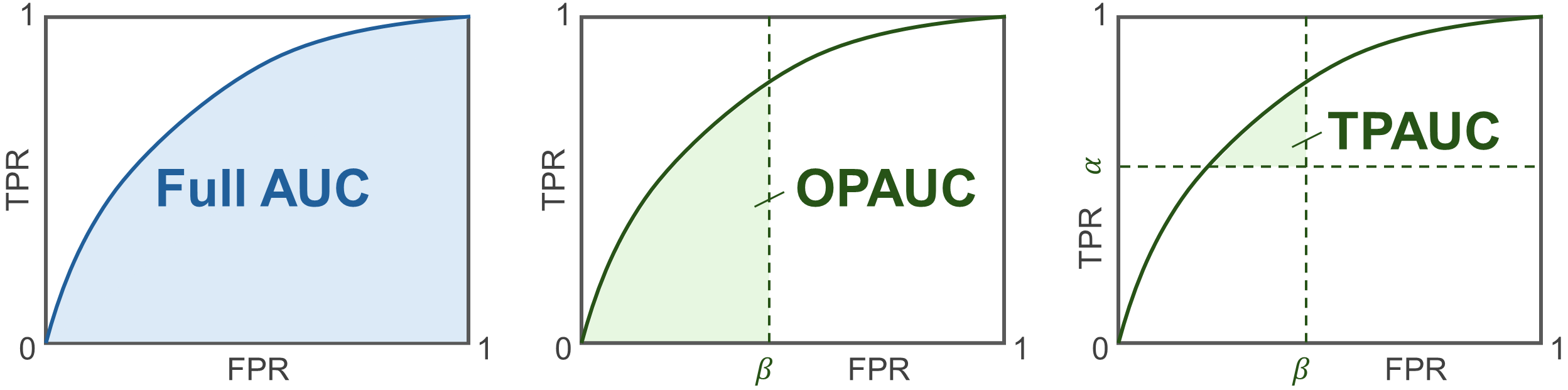}
  \caption{Illustration of AUC and its two typical variants (OPAUC \& TPAUC).}
  \label{fig:aucs}
\end{figure}

\begin{table*}[t]
    \small
    \centering
    \caption{
    Comparison with existing PAUC optimization algorithms. The convergence rate represents the number of iterations after which an algorithm can find an $\epsilon$-stationary point (denoted as `$\epsilon$-sp') or a nearly $\epsilon$-critical point (denoted as `nearly $\epsilon$-cp'). $\triangle$ implies a natural result of non-convex SGD. $n_+^B$ ($n_-^B$ resp.) is the number of positive (negative resp.) instances in each mini-batch $B$.
    }
    \setlength{\tabcolsep}{3pt}
    \begin{tabular}{c|cccc|cc}
        \toprule
        ~ & SOPA \cite{zhu2022auc}&SOPA-S \cite{zhu2022auc}&AGD-SBCD\cite{yao2022large}&TPAUC \cite{yang2021all}& \cellcolor{orange!10}Ours \cite{shao2022asymptotically} & \cellcolor{yellow!15}Ours (journal) \\
        \midrule
        Convergence Rate ($\op$) & $O(\epsilon^{-4})$&$O(\epsilon^{-4})$ & $O(\epsilon^{-6})$ &$O(\epsilon^{-4})^{\triangle}$& \multicolumn{2}{c}{$O(\epsilon^{-3})$} \\
        Convergence Rate ($\tp$) & $O(\epsilon^{-6})$&$O(\epsilon^{-4})$& - &$O(\epsilon^{-4})^{\triangle}$ & \multicolumn{2}{c}{$O(\epsilon^{-3})$} \\
        Convergence Measure & $\epsilon$-sp (non-smooth) & $\epsilon$-sp & nearly $\epsilon$-cp & $\epsilon$-sp & \multicolumn{2}{c}{$\epsilon$-sp} \\
        Smoothness & \xmark & \cmark & \xmark & \cmark & \multicolumn{2}{c}{\cmark} \\
        Unbiasedness & \cmark & \xmark & \cmark & \xmark & \cellcolor{orange!7}\makecell[c]{with bias $O(1/\kappa)$ \\ when $\omega = 0 $} & \cellcolor{yellow!12}\cmark \\
        Per-Iter. Time Complexity & $O(n_+^B n_-^B)$ & $O(n_+^B n_-^B)$ & $O(n_+^B n_-^B)$ & $O(n_+^B n_-^B)$ & \multicolumn{2}{c}{$O(n_+^B + n_-^B)$} \\
        \bottomrule
    \end{tabular}
    \label{tab1}
\end{table*}

In many high-stakes applications, decisions are constrained to a specific region of the ROC curve. For instance, in malware detection a high FPR implies many benign files are undesirably flagged as malware, which is unacceptable even if TPR is high. Analogously, some tasks operate only in a low-TPR regime. The Partial AUC (PAUC), focusing only on the specific region in the ROC curve, has then emerged as a more proper metric in the production environment \cite{xie2024weakly}.
As illustrated in Fig.\ref{fig:aucs}, there are two typical types of PAUC:
\begin{itemize}
    \item One-way PAUC (OPAUC) computes the area within a specified FPR interval ($0 \le \mathrm{FPR} \le \beta$);
    \item Two-way PAUC (TPAUC) computes the area with the constraints $\mathrm{FPR} \le \beta$, $\mathrm{TPR} \ge \alpha$.
\end{itemize}
Directly optimizing PAUC is more challenging than the full AUC. Beyond the inefficiency of the \emph{pairwise} formulation, PAUC requires selecting instances inside the constrained bands, which induces an NP-hard combinatorial component. Hence the development in this field often lags behind that of full AUC optimization.
Early efforts rely on full-batch optimization and approximate sample selection (or ranking) process, suffering from uncontrolled biases and limited efficiency \cite{2003Partial, narasimhan2013structural, narasimhan2017support, kumar2021implicit}.

Most recently, the rise of deep learning has also sparked various stochastic mini-batch end-to-end PAUC optimization algorithms.
\cite{yang2021all,DBLP:journals/pami/YangXBHCH23} propose the early trial in this line of research by introducing surrogate weighting functions to approximate the ranking process, but the estimation of PAUC is still biased. Later, \cite{zhu2022auc} cleverly leverage the distributionally robust optimization framework to learn the sample weights, deriving an exact but nonsmooth objective and an inexact but smooth objective. \cite{yao2022large} cast the problem as a non-smooth difference-of-convex program, naturally inducing an unbiased formulation for OPAUC optimization. However, these approaches primarily focus on deal with the non-differentiable sample selection step, but still keep an explicit positive-negative pairing, leading to a $O(n_+^B n_-^B)$ per-iteration time complexity and a slow convergence rate, as listed in Tab.\ref{tab1}.

\textbf{This paper closes the approximation gap of PAUC optimization by introducing two novel instance-wise formulations}\footnote{We focus on the region $\mathrm{FPR} \le \beta$ for $\op$, while $\mathrm{FPR} \le \beta$ and $\mathrm{TPR} \ge \alpha$ for $\tp$.}. Our approach fundamentally differs from existing methods by first decoupling the pairwise dependencies through a minimax reformulation, transforming the problem into an instance-wise learning task. We then introduce a differentiable, sorting-free sample selection technique to efficiently handle the top/bottom-$k$ ranking procedure, overcoming a key computational bottleneck. This framework subsequently branches into two pathways: a smooth surrogate approximation and an exact unbiased reformulation, both ultimately yielding a standard minimax objective amenable to efficient stochastic optimization. With an efficient training algorithm, both formulations can enjoy a convergence rate of $O(\epsilon^{-3})$. Our methods reduce the per-iteration complexity to a linear $O(n_+^B + n_-^B)$, a comparative summary of which is in Tab.\ref{tab1}. Our contributions can be summarized as follows:
\begin{itemize}
    \item We propose two minimax instance-wise formulations for PAUC ($\op$ and $\tp$) maximization, where one is with an asymptotically vanishing gap while the other is exactly unbiased. The instance-wise design removes explicit pair enumeration and ranking, achieving a $O(n_+^B + n_-^B)$ per-iteration cost and a convergence rate of $O(\epsilon^{-3})$.
    \item Building on the instance-wise reformulation, we establish a tight generalization bound with a much easier proof than prior results \cite{narasimhan2017support, yang2021all, DBLP:journals/pami/YangXBHCH23}. The bound makes explicit the effects of the TPR/FPR constraints $\alpha$/$\beta$ and exhibits a sharp order of $\tilde{O}(\alpha^{-1}\np^{-1} + \beta^{-1}\nn^{-1})$.
    \item We conduct extensive experiments on multiple imbalanced image classification tasks. The results speak to the effectiveness of our proposed methods.
\end{itemize}


A preliminary version of this paper is published in NeurIPS \cite{shao2022asymptotically}, where we propose an efficient instance-wise PAUC maximization algorithm with an asymptotically vanishing approximation bias. In this long version, we have introduced the following major extensions:
\begin{itemize}
    \item \textbf{New Formulation}: We take a step further by proposing a new unbiased formulation to completely eliminate the approximation gap, ensuring that the estimation is as accurate as possible.
    \item \textbf{New Generalization Analysis}: The analysis is improved with the local Rademacher complexity based technique, reaching a much tighter bound than \cite{shao2022asymptotically}.
    \item \textbf{New Experiments}: We involve new empirical results in terms of datasets, competitors, sensitivity analysis and fine-grained visualization to make the evaluation more comprehensive.
\end{itemize}


The rest of this paper is organized as follows. Sec.\ref{sec:related} reviews the related work and Sec.\ref{sec:pre} provides the preliminary knowledge for AUC/PAUC optimization. Then Sec.\ref{sec:opauc} describes the proposed formulations, followed by a stochastic optimization algorithm in Sec.\ref{sec:alg} and a generalization analysis in Sec.\ref{sec:generalization}. Moreover, comprehensive empirical results are illustrated in Sec.\ref{sec:exp}. Finally, Sec.\ref{sec:con} provides the concluding remarks for this paper.

\section{Related work}\label{sec:related}

\subsection{Deep AUC Optimization}
Optimization is central to machine learning research, as algorithmic advances often hinge on efficient reformulations of complex objectives \cite{DBLP:journals/pami/GaoWFHJ23,DBLP:journals/tnn/XiaLHYJG23}. Among these, AUC maximization has been a key object since the late 1990s \cite{herbrich2000large}. In the past few decades, AUC optimization has already achieved remarkable success in long-tailed/imbalanced learning \cite{yang2022auc}. A partial list of the related literature includes \cite{graepel2000large, cortes2003auc, yan2003optimizing, joachims2005support, pepe2000combining, freund2003efficient, rakotomamonjy2004support, ying2016stochastic}. The main challenge in this field is to efficiently optimize the pairwise formulation over positive and negative instances. Early approaches only apply to the full-batch setting and are usually computationally expensive \cite{yan2003optimizing, freund2003efficient}.

In recent age, many studies focused on AUC optimization with stochastic gradient methods. As a milestone, on top of the square surrogate loss, \cite{ying2016stochastic} first proposed a minimax reformulation of the AUC, which is the average of individual data points. Such an instance-wise formulation significantly facilitates the stochastic optimization. With a strongly convex regularizer, \cite{natole2018stochastic} improved the convergence rate of the stochastic algorithm for AUC to $O(1/T)$. Then \cite{liu2019stochastic} further extends the minimax objective to nonconvex deep neural networks. In succession, many efforts are devoted to improving the minimax deep AUC optimization performance. For example, \cite{yuan2021compositional} builds a compositional objective combining both AUC loss and cross-entropy loss to improve the feature representation and get rid of the warm-up stage. \cite{DBLP:conf/iccv/Yuan0SY21} applies the squared hinge function in the objective and induces an AUC margin loss with better robustness.
Attracting more and more attention, AUC optimization is widely applied in various scenarios such as recommender systems \cite{DBLP:journals/pami/BaoX0CH23,DBLP:journals/tist/YeL25}, robust learning \cite{DBLP:conf/kdd/ZhangSLG23,DBLP:journals/pami/YangXHBHCH23}, multi-task learning \cite{DBLP:conf/nips/HuZY22,DBLP:conf/aaai/YangXCH19}, domain adaptation \cite{DBLP:journals/pami/YangXBWHCH23}, federated learning \cite{DBLP:conf/icml/Guo0LY23} and multi-instance learning \cite{DBLP:conf/icml/ZhuWCWSWY23}, semantic segmentation \cite{DBLP:conf/nips/HanX0BWJH24}, and has great potential in other learning settings \cite{DBLP:conf/www/Ye0025,DBLP:journals/pami/JiangLCHXYCH23,DBLP:journals/pami/JiangXZYWCH23,DBLP:conf/nips/ZhangZT0S20} and computer vision tasks \cite{DBLP:conf/cvpr/YangXLHZLFL24,DBLP:journals/pami/XiaYLZFGL24,shi2024unified,shi2022proposalclip,10681116,DBLP:conf/aaai/ZhangCZWZLS25,DBLP:conf/cvpr/0003CWWHLS25}. Recently, \cite{xie2024weakly} transforms AUC optimization in various weakly supervised scenarios into the maximization of a new type of PAUC, further connecting the AUC and PAUC optimization problems.

\subsection{Partial AUC (PAUC) Optimization}
The concept of PAUC could be dated back to 1989 \cite{mcclish1989analyzing}. The additional region constraints make PAUC more difficult to optimize than full AUC, and thus the research progress on PAUC optimization always trails that of AUC optimization. Earlier studies related to PAUC only paid attention to the simplest linear models. In \cite{pepe2000combining}, PAUC is first optimized by a distribution-free rank-based method. \cite{wang2011marker} developed a non-parametric estimate of PAUC, and selected features at each step to build the final classifier. \cite{narasimhan2013structural} develops a cutting plane algorithm to find the most violated constraint instance, decomposing PAUC optimization into subproblems and solving them by an efficient structural SVM-based approach.

However, most of the above approaches often fall into the non-differentiable property or intractable optimization problems, posing a significant obstacle to the end-to-end implementation. In the end-to-end learning direction, \cite{kar2014online} first proposes a na\"ive mini-batch stochastic method for PAUC that applies to deep neural networks, but it is not guaranteed to converge for minimizing the pAUC objective \cite{yang2022auc}. Using the Implicit Function Theorem, \cite{kumar2021implicit} formulated a rate-constrained optimization problem that modeled the quantile threshold as the output of a function of model parameters. As a milestone study, \cite{yang2021all} simplifies the challenging sample-selected problem in a bi-level manner and thus facilitates the end-to-end optimization for PAUC. Concretely, the inner level achieves instance selection, and the outer level minimizes the loss. Nevertheless, their estimation may suffer from an approximation error with the true PAUC. \cite{yao2022large} formulates the problem as a non-smooth difference-of-convex program for one-way PAUC and develops an approximated gradient descent method (AGD-SBCD) based on Moreau envelope smoothing to solve it. \cite{zhu2022auc} proposes both non-smooth and smooth estimators of PAUC based on DRO named SOPA and SOPA-S respectively, and provides sound theoretical convergence guarantee of their algorithms.

As the comparison in Tab.\ref{tab1} shows, these existing methods are limited by either an approximation error (SOPA-S and TPAUC \cite{yang2021all}), or a slow convergence rate in some cases (SOPA and AGD-SBCD). More importantly, they mainly devote to addressing the quantile calculation but still adopt the pairwise formulation. In contrast, based on the instance-wise reformulation and min-max swapping technique, our proposed methods enjoy an asymptotically biased approximation gap or even without the gap, together with a lower computational cost and a higher convergence rate.

\subsection{Generalization Analysis for PAUC Optimization}
The pairwise form of AUC-based losses also causes difficulties for its generalization analysis, since common independent-loss-term-based techniques in PAC learning theory or statistical learning theory are not applicable \cite{mohri2018foundations}. By introducing new notions of data-dependent hypothesis complexities, \cite{agarwal2005generalization, usunier2005data,yang2021learning} propose various studies on the generalization property of AUC optimization. Generalization bounds for such pairwise losses could also established on the concept of uniform stability \cite{lei2020sharper, lei2021generalization} which also takes the optimization algorithm into consideration.

There are fewer generalization analyses for PAUC optimization. \cite{narasimhan2017support} presents the first generalization analysis for $\op$ and derives a uniform convergence generalization bound. Following their work, a recent study \cite{yang2021all} extends this generalization bound to $\tp$. However, limited by the pairwise form of AUC, all of the above studies require complicated decomposition. Moreover, these generalization analyses only hold for hard-threshold functions and VC-dimension. Although \cite{DBLP:journals/pami/YangXBHCH23} improves the order of the bound in \cite{yang2021all}, the result ignores the impact of the constraints $\alpha$/$\beta$. Based on our instance-wise reformulation, we show that the generalization of PAUC is as simple as other instance-wise algorithms and can be tightly bounded for real-valued score functions with a proper dependence on $\alpha$/$\beta$.

\section{Preliminaries}\label{sec:pre}

\subsection{Notations} Let $\setX \subseteq \mathbb{R}^d$ be the input space, and $\setY=\{0, 1\}$ be the label space where $y=1$ ($0$ \textit{resp.}) is for the positive (negative \textit{resp.}) class. Given a set of $n$ training samples $S=\{\bm{z}_i=(\bm{x}_i,y_i)\}_{i=1}^n$ drawn from distribution $\setD_\setZ=\setX\times\setY$, we denote $\setP$ ($\setN$ \textit{resp.}) as the set of positive (negative \textit{resp.}) instances in the dataset, with a size of $n_+$ ($n_-$ \textit{resp.}). The corresponding positive and negative instance distribution are denoted as $\setD_\setP$ and $\setD_\setN$, respectively. The scoring function $f: \setX \mapsto [0,1]$ to be learned is parametrized by $\bm{\theta}$, which can be simply implemented by any deep neural network with sigmoid outputs. The indicator function $\indicator{A}$ is $1$ if the condition $A$ holds, otherwise it is $0$. $\hat{\E}_{z\sim S}[g(\bm{z})] = 1/n \cdot \sum_{i=1}^{n} g(\bm{z}_i)$ is the empirical expectation on $S$.

\subsection{Standard AUC and Partial AUCs}
\noindent\textbf{Standard AUC.} The standard AUC calculates the entire area under the ROC curve. As shown in \cite{J1982The}, $\mathrm{AUC}$ measures the probability of a positive instance having a higher score than a negative instance:
\begin{equation}
    \mathrm{AUC}(f) = \underset{\bm{x}\sim \setD_{\setP},\bm{x}'\sim \setD_{\setN}}{\Pr}\left[f(\bm{x})> f(\bm{x}')\right].\\
\end{equation}

\noindent\textbf{$\mathrm{\mathbf{OPAUC}}$.} According to \cite{2003Partial}, $\op$ is equivalent to the probability of a positive instance $\bm{x}$ being scored higher than a negative instance $\bm{x}'$ within the specific range $f(\bm{x}')\in[\efb, 1]$ \stt~ ${\Pr}_{\bm{x}'\sim\mathcal{\setD_\setN}}[f(\bm{x}')\geq\eta_{\beta}]=\beta$:
\begin{equation}
\begin{aligned}
    \mathrm{OPAUC}(f) &= \underset{\bm{x}\sim \setD_{\setP},\bm{x}'\sim \setD_{\setN}}{\Pr}\left[f(\bm{x})> f(\bm{x}'),  f(\bm{x}')\geq  \efb \right].
\end{aligned}
\label{OPAUC}
\end{equation}
Practically, we do not know the exact data distributions $\setD_{\setP}$, $\setD_{\setN}$ to calculate Eq.\eqref{OPAUC}. Therefore, we turn to its empirical estimation. The empirical $\op$ over a finite dataset $S$ could be expressed as \cite{narasimhan2013structural}:
\begin{equation}
    \begin{aligned}
        \widehat{\mathrm{AUC}}_{\beta}(f, S) = 1-\sum_{i=1}^{n_+} \sum_{j=1}^{n_-^\beta}\frac{\ell_{0,1}{\left(f(\bm{x}_i)- f(\bm{x}_{[j]}')\right)}}{n_+ n_-^\beta},
        \label{OPAUCM}
    \end{aligned}
\end{equation}
where $n_-^\beta=\lfloor n_-\cdot \beta \rfloor$; $\bm{x}'_{[j]}$ has the $j$-th largest score among negative samples; $\ell_{0,1}(t) = \indicator{t<0}$ is the 0-1 loss, which returns $1$ if $t <0$ and $0$ otherwise.

\noindent\textbf{$\mathrm{\mathbf{TPAUC}}$.}  More recently, \cite{yang2019two} argued that an efficient classifier should have low $\mathrm{FPR}$ and high $\mathrm{TPR}$ simultaneously. Therefore, we also study a more general variant called $\tp$, where the restricted regions satisfy $\mathrm{TPR}\geq \alpha$ and $\mathrm{FPR}\leq \beta$. Similar to $\op$, $\tp$ measures the probability that a positive instance $\bm{x}$ ranks higher than a negative instance $\bm{x}'$ where $f(\bm{x})\in[0, \eta_{\alpha}(f)]$ \stt~ ${\Pr}_{\bm{x}\sim\mathcal{D_\setP}}[f(\bm{x})\leq\eta_{\alpha}]=\alpha$, and $f(\bm{x}')\in[\efb, 1]$ \stt~ ${\Pr}_{\bm{x}'\sim\mathcal{D_\setN}}[f(\bm{x}')\geq\eta_{\beta}]=\beta$:
\begin{equation}
\begin{aligned}
    \mathrm{TPAUC}(f) = \underset{\bm{x}\sim \setD_{\setP},\bm{x}'\sim \setD_{\setN}}{\Pr}[ &f(\bm{x})> f(\bm{x}'),\\
    & f(\bm{x})\leq \eta_{\alpha}(f), f(\bm{x}')\geq  \efb ].
    \label{TPAUC}
\end{aligned}
\end{equation}
We can also adopt its empirical estimation \cite{yang2019two, yang2021all}:
\begin{equation}
    \begin{aligned}
        \widehat{\mathrm{AUC}}_{\alpha, \beta}(f, S) = 1- \sum_{i=1}^{n_+^\alpha} \sum_{j=1}^{n_-^\beta}\frac{\ell_{0,1}{\left(f(\bm{x}_{[i]})- f(\bm{x}'_{[j]})\right)}}{n_+^\alpha n_-^\beta},
        \label{TPAUCM}
    \end{aligned}
\end{equation}
where $n_+^\alpha=\lfloor n_+\cdot \alpha \rfloor$ and $\bm{x}_{[i]}$ has the $i$-th smallest score among all positive instances.

\begin{figure*}[t]
    \centering
    \includegraphics[width=0.94\textwidth]{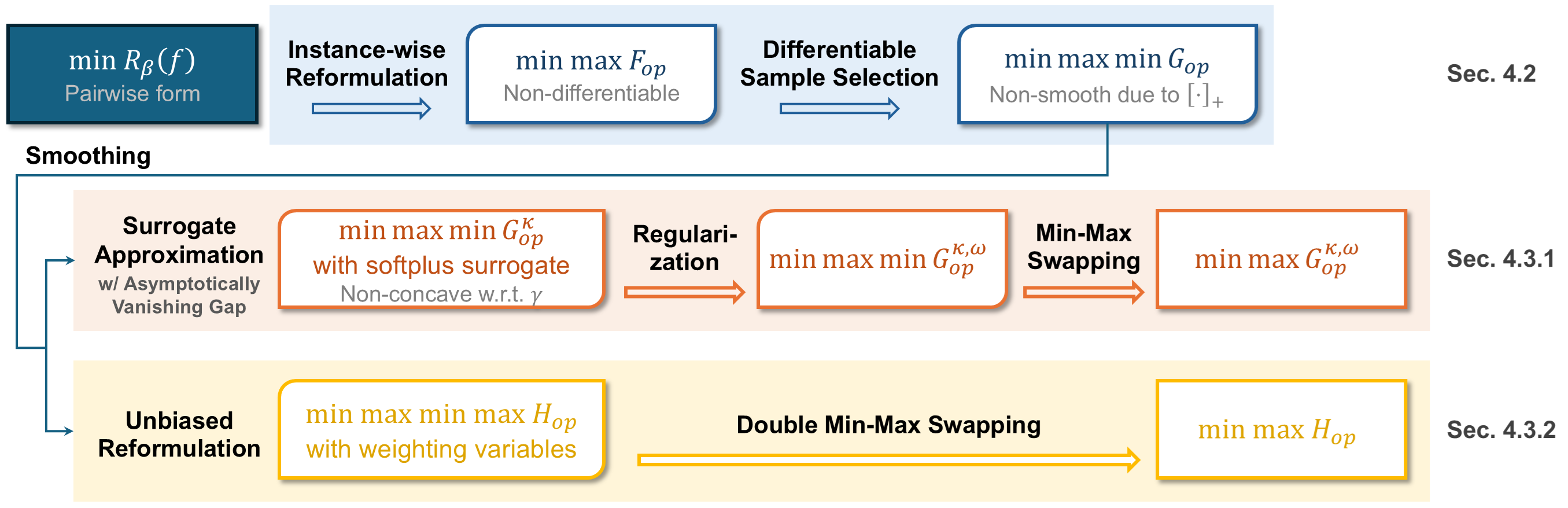}
    \caption{Roadmap of our formulation derivation for OPAUC. TPAUC formulation follows the same steps.}
    \label{fig:workflow}
\end{figure*}

\section{Optimizing Partial AUCs: From OPAUC to TPAUC}\label{sec:opauc}


\subsection{Challenges and Roadmap}
Efficiently optimizing the PAUC variants (Eq.\eqref{OPAUCM} and \eqref{TPAUCM}) presents significant challenges: \textbf{(C1)} the 0-1 loss is indifferentiable and the pairwise formulation imposes a high computational cost of $O(n_+ n_-)$; \textbf{(C2)} it is inherently complicated to determine the positive (negative \textit{resp.}) quantile function $\eta_{\alpha}(f)$ ($\efb$ \textit{resp.}).

Our roadmap is illustrated in Fig.\ref{fig:workflow}. We address \textbf{(C1)} by deriving an \emph{instance-wise} reformulation (Thm.\ref{thm:instance_op}) and \textbf{(C2)} by converting top-$k$ selection into a \emph{sorting-free} threshold learning problem (Thm.\ref{thm:instance_difftopk_op}). Then to mitigate the consequent non-smoothness issue, two smoothing strategies are presented in Sec.\ref{sec:opauc_asym} and Sec.\ref{sec:opauc_nonasym} with and without introducing an asymptotic gap, respectively. The same pipeline applies to $\tp$.


\subsection{Differentiable Instance-Wise Reformulation for OPAUC}\label{sec:instance}
To tackle the challenge (C1), we first adopt a differentiable loss $\ell$ as the surrogate of the 0-1 loss following the convention. Then maximizing $\widehat{\mathrm{AUC}}_{\beta}(f, S)$\footnote{For clarity, we present the empirical form throughout the paper. The population counterpart is analogous by replacing empirical averages with expectations (see Appx.\ref{sec:op_pop_ver}).} on a finite dataset $S$ is equivalent to solving the following surrogate problem:
\begin{equation}
    \underset{f}{\min} \ \hat{\mathcal{R}}_{\beta}(f, S) = \sum_{i=1}^{n_+} \sum_{j=1}^{n_-^\beta}\frac{\ell{\left(f(\bm{x}_i) - f(\bm{x}'_{[j]})\right)}}{n_+ n_-^\beta}.
\end{equation}

In this work, to simplify the subsequent reformulation, we will use the most popular \textbf{surrogate squared loss} $\ell(x) =(1-x)^2$. Sharing a similar merit to the stochastic AUC maximization \cite{ying2016stochastic}, this $\ell$ yields a favorable instance-wise reformulation of the $\op$ optimization problem by introducing some extra global variables, naturally eliminating the issue in (C1) (please see Appx.\ref{sec:proof_instance_op} for the proof):
\begin{thm}\label{thm:instance_op}
\textbf{(Instance-wise Reformulation.)}
For $f(\bm{x})\in[0,1]$ and $\beta\in[0,1]$, $\forall \bm{z}=(\bm{x},y)\in\setD_\setZ$, the instance-wise loss function $F_{op}(f,a,b,\gamma, t, \bm{z})$ is defined as:
\begin{equation}\label{eq:fop}
\begin{aligned}
  F_{op}(f,&a,b,\gamma,t, \bm{z}) = P(f,a,\gamma,\bm{x})\cdot y/p-\gamma^2\\
&+N(f,b,\gamma,\bm{x})\cdot[(1-y)\indicator{f(\bm{x})\geq t}]/[\beta(1-p)],
\end{aligned}
\end{equation}
where $p=\Pr[y=1]$ is the positive class prior probability, and
\begin{equation}
  \begin{aligned}
    P(f,a,\gamma,\bm{x}) =& (f(\bm{x})-a)^2-
    2(1+\gamma)f(\bm{x}), \\
    N(f,b,\gamma,\bm{x}) =& (f(\bm{x})-b)^2+2(1+\gamma)f(\bm{x}).
  \end{aligned}
\end{equation}
Then we have the following equivalent formulation for $\op$:
\begin{equation}\label{eq:minmaxopauc2}
    \begin{aligned}
    \underset{f}{\min}~ &\hat{\mathcal{R}}_{\beta}(f, S) \\
    \Leftrightarrow & \underset{\cmin}{\min}\ \underset{\gamma\in[-1,1]}{\max} \  \colbit{\hezds} \left[F_{op}(f,a,b,\gamma, \colbit{\hefb}, \bm{z} ) \right],
    \end{aligned}
\end{equation}
where $\colbit{\hefb}$ is the empirical quantile of negative instances in $S$.
\end{thm}

\begin{rem}
Here, $P(\cdot)$ and $N(\cdot)$ represent the individual positive and negative losses. The variables $a$ and $b$ serve as reference averages for positives and top quantile negatives, respectively. Minimizing over $(a,b)$ adaptively calibrates the scores for each class. Meanwhile, the dual variable $\gamma$ highlights the score gap between classes. This maximization step emphasizes the discriminability between positives and negatives. In this sense, each sample only interacts with global statistics rather than all pairs.

\end{rem}

Thm.\ref{thm:instance_op} removes explicit positive-negative pairing and provides a support to convert the pairwise loss into instance-wise for $\op$. Once the quantile computation is instance-wise, the entire formulation exhibits a linear computational cost \wrt~the sample size. However, the quantile-based operation $\indicator{f(\bm{x}')\geq\hat{\eta}_{\beta}}$ requires sorting negative instances for selection. A second key step is to transform the top-$k$ selection into a selection threshold learning problem, making the quantile operation differentiable and thus solving the issue (C2).

We denote $N(f,b,\gamma,\bm{x})$ as $\ell_-(\bm{x})$ for short when $f,b,\gamma$ are not discussed. Since $\ell_-(\bm{x})$ is increasing in $f(\bm{x})$ for $\gamma \in [b-1,1]$, the selection based on $f(\bm{x})$ can be equivalently cast in terms of $\ell_-(\bm{x})$. By Lem.1 of \cite{fan2017learning}, the top-$k$ operator admits a sorting-free threshold reformulation:
\begin{equation}
  \begin{aligned}
    &\hat{\E}_{\bm{x}\sim\setN}[\indicator{f(\bm{x})\geq\hat{\eta}_{\beta}(f)}\cdot\ell_-(\bm{x})] \\
    &= \min_{s'} \frac{1}{\beta}\cdot  \hat{\E}_{\bm{x}\sim\setN} [\beta s' + [\ell_-(\bm{x})-s']_+],
  \end{aligned}
\end{equation}
where $s'$ serves as a learnable threshold for the selection.
This yields Thm.\ref{thm:instance_difftopk_op} with a differentiable negative sample selection (please see Appx.\ref{sec:proof_instance_difftopk_op} for the proof):
\begin{thm}\label{thm:instance_difftopk_op}
    \textbf{(Differentiable Sample Selection.)}
For $f(\bm{x})\in[0,1]$ and $\beta\in[0,1]$, $\forall \bm{z}=(\bm{x},y)\in \setD_\setZ$, we have the equivalent optimization problem for $\op$:
\begin{equation}\label{minmaxmin_op}
\begin{aligned}
  &\underset{\cmin}{\min}\ \underset{\gamma\in[-1,1]}{\max} \
\colbit{\hezds}[F_{op}(f,a,b,\gamma, \colbit{\hefb},\bm{z})]\\
&\Leftrightarrow \underset{\cmin}{\min}\
\underset{\gamma\in\Omega_{\gamma} }{\max}
\ \underset{s'\in\Omega_{s'}}{\min}
\ \colbit{\hezds}[G_{op}(f,a,b,\gamma,\bm{z},s')],
\end{aligned}
\end{equation}
where $\Omega_{\gamma}=[b-1,1]$, $\Omega_{s'}=[0,5]$ and
\begin{equation}\label{OPAUC_IB}
\begin{aligned}
G_{op}(f,&a,b,\gamma,\bm{z},s') = P(f,a,\gamma,\bm{x})\cdot y/p-\gamma^2\\
& +
\left(\beta s' +\left[N(f,b,\gamma,\bm{x})-s'\right]_+\right)\cdot (1-y)/[\beta(1-p)].
\end{aligned}
\end{equation}
\label{thm:2}
\end{thm}

So far, we have obtained an instance-wise formulation with a simplified and differentiable quantile calculation, successfully address issues (C1) and (C2). However, this new min-max-min formulation, in turn, gives rise to new difficulties on optimization. An ideal solution to navigate this complexity is to swap the order of $\max_{\gamma}$ and $\min_{s'}$, thereby reframing it as a minimax problem that can be more effectively tackled with advanced optimization techniques. In order to realize this idea, we must address the non-smoothness issue caused by the non-smooth function $[\cdot]_+$.

\subsection{Handling Non-Smooth Selection: Surrogate Approximation vs. Unbiased Reformulation}

At this point, two options emerge for handling the non-smoothness: (i) employ a smooth surrogate, leading to an asymptotically unbiased optimization; (ii) adopt an unbiased reformulation via auxiliary weights, which preserves exactness at the cost of more variables. We next detail each formulation in turn.

\subsubsection{Surrogate Approximation with Asymptotically Vanishing Gap}\label{sec:opauc_asym}
The simplest strategy to deal with the non-smooth function is to replace it with a smooth surrogate.
Using a very common way, we apply the \texttt{softplus} surrogate function \cite{glorot2011deep} to smooth $[\cdot]_+$:
\begin{equation}\label{eq:kappa}
    r_{\colblue{\kappa}}(x) = \frac{\log\left(1+\exp({\colblue{\kappa}} \cdot x)\right)}{{\colblue{\kappa}}}.
\end{equation}
It is easy to show that $r_{\colblue{\kappa}}(x) \overset{\kappa \rightarrow \infty}{\rightarrow} [x]_+$. \textbf{We then proceed to optimize the smooth surrogate objective $G_{op}^{\colblue{\kappa}}(f,a,b,\gamma,\bm{z},s')$ where the $[\cdot]_+$ in $G_{op}(f,a,b,\gamma,\bm{z},s')$ is replaced with $r_{\colblue{\kappa}}(\cdot)$}.

Undoubtedly, the introduction of a surrogate will incorporate a degree of bias into the optimization objective. When $f,b,\gamma$ are fixed, smoothing the loss with $r_k$ results in a different optimal selection threshold $s'$ from using $[\cdot]_+$. This new threshold actually selects the top-$\tbeta$ quantile of negative instances, diverging from the original intent of top-$\beta$ selection. In other words, the surrogate essentially causes a deviation over the proportion of negative instances involved for $\op$ optimization, as illustrated in Fig.\ref{fig:diff_beta}.

Nevertheless, we can prove the approximation gap induced by $r_k$ has a finite convergence rate $O(1/\kappa)$ by the following theorem. Namely, this gap vanishes when $\kappa\rightarrow \infty$, and the quantile deviation will correspondingly diminish. Consequently, our surrogate optimization problem is guaranteed to be asymptotically unbiased towards its non-smooth counterpart. Please see Appx.\ref{sec:uniform_convergence} for the proof.
\begin{thm}\label{thm:bias_converge_op}
    \textbf{(Asymptotically Vanishing Approximation Gap.)}
Denote the bias induced by the approximation as
\begin{align*}
    & \Delta_\kappa^{op} =
    \min_{\cmin} \max_{\cmax} \min_{s^{\prime}\in\Omega_{s'}}
    \hezds \big[
    G_{op}^{\colblue{\kappa}}\left(f, a, b, \gamma, \boldsymbol{z}, s^{\prime}\right)
    \big]\\
    &~~~~~~~~~~-
    \min_{ \cmin} \max_{\cmax} \min_{s^{\prime}\in\Omega_{s'}}
    \hezds \big[
    G_{op}(f,a, b, \gamma, \boldsymbol{z}, s^{\prime})\big].
\end{align*}
For $f(\bm{x}) \in [0,1]$, we have the following convergence result:
\begin{equation*}
    \Delta_\kappa^{op} = O(1/\kappa).
\end{equation*}
\end{thm}

However, $r_k$ itself has brought new obstacles to the min-max swapping --- the concavity of the objective \wrt~$\gamma$, which is necessary for the swapping, does not hold. Fortunately, it is easy to check that $r_\colblue{\kappa}(x)$ has a bounded second-order derivative. In this way, we can regard $G^{\colblue{\kappa}}_{op}(f,a,b,\gamma,\bm{z},s')$ as a weakly-concave function \cite{boyd2004convex} of $\gamma$. By employing an $\ell_2$ regularization, we turn to a regularized form:
\begin{equation}
    G^{{\colblue{\kappa}},\colbit{\omega}}_{op}(f,a,b,\gamma,\bm{z} ,s') = G^{\colblue{\kappa}}_{op}(f,a,b,\gamma,\bm{z},s')  - \colbit{\omega} \cdot \gamma^2.
\end{equation}
With a sufficiently large $\colbit{\omega}$, $G^{{\colblue{\kappa}},\colbit{\omega}}_{op}(f,a,b,\gamma,\bm{z} ,s')$ is strongly-concave \wrt~$\gamma$. Note that the regularization scheme will inevitably introduce some bias. Nonetheless, it is known to be a necessary building block to stabilize the solutions and improve generalization performance.

We then reach a minimax problem in the final step.

\noindent\textbf{Min-Max Swapping.} For fixed $(f,a,b)$, the inner domains $\Omega_\gamma$ and $\Omega_{s'}$ are convex and compact. Moreover, $G_{op}^{{\colblue{\kappa}},\colbit{\omega}}$ is convex in $s'$ and concave in $\gamma$.
Applying the minimax theorem \cite{boyd2004convex}, we can interchange $\max_{\gamma}$ and $\min_{s'}$:
\begin{equation}\label{minmax_op_em}
\begin{aligned}
  &\underset{\cmin}{\min}\ \underset{\gamma\in\Omega_{\gamma}}{\max} \min_{s' \in \Omega_{s'}}
    \ \colbit{\hezds}[{G}_{op}^{{\colblue{\kappa}},\colbit{\omega}}] \\
    &\Leftrightarrow \underset{\cminl}{\min}\ \underset{\cmax}{\max}
\ \colbit{\hezds}[{G}_{op}^{{\colblue{\kappa}},\colbit{\omega}}],
\end{aligned}
\end{equation}
where $G_{op}^{{\colblue{\kappa}},\colbit{\omega}} = G_{op}^{{\colblue{\kappa}},\colbit{\omega}}(f,a,b,\gamma,\bm{z}, s')$. In this sense, we come to a regularized non-convex strongly-concave problem. In Sec.\ref{sec:alg}, we will employ an efficient solver to optimize it.


\begin{figure}[t]
    \centering
    \includegraphics[width=0.9\linewidth]{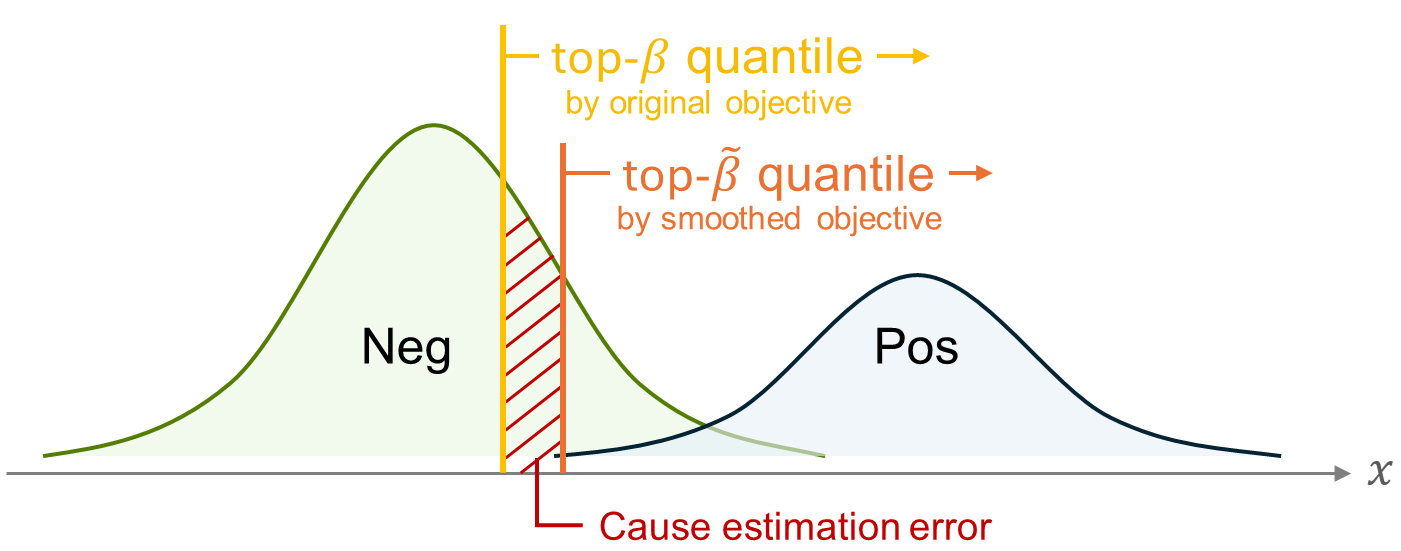}
    \caption{Illustration of the quantile deviation caused by the smoothed objective for OPAUC in the univariate case.}
    \label{fig:diff_beta}
\end{figure}

\subsubsection{Unbiased Reformulation without Asymptotic Gap}
\label{sec:opauc_nonasym}


Although the gap $\Delta_\kappa^{op}$ vanishes asymptotically as discussed in Thm.\ref{thm:bias_converge_op}, the smoothness of the function will also vanish undesirably when $\kappa\to\infty$. To make the optimization efficient, we must consider a limited value of $\kappa$, for which the asymptotic gap cannot be ignored. In this case, the quantile deviation could have a potential impact on the variable estimation to some extent.

We now present an alternative \emph{unbiased reformulation} that exactly preserves the original non-smooth objective. By integrating a continuous auxiliary weighting variable $c\in[0,1]$, we can effectively cast $[\cdot]_+$ as an equivalent maximization problem:
\begin{equation}
    [x]_+ = \max_{c\in[0,1]} c\cdot x.
\end{equation}
It maintains the natural smoothness of the function and preserves the convexity \wrt~$x$.

With this fact at hand, we can further reformulate the sample selection procedure as
\begin{equation}
    \begin{aligned}
      &\min_s \frac{1}{\beta}\cdot  \hat{\E}_{\bm{x}'\sim\setN} [\beta s + [\ell_-(\bm{x}')-s]_+] \\
      =& \min_s \max_{c(\bm{x})\in[0,1]} \frac{1}{\beta}\cdot  \hat{\E}_{\bm{x}'\sim\setN} \big[\beta s + c(\bm{x}') \cdot (\ell_-(\bm{x}')-s)\big].
    \end{aligned}
\end{equation}
Here $c(\bm{x}')$ is instantiated as a set of weights $\bm{c}=\{c_j\in[0,1]\}^{n_-}_{j=1}$. Namely, each negative instance is explicitly assigned a selection weight.

Consequently, we derive the following equivalent optimization problem for Eq.\eqref{minmaxmin_op}:
\begin{equation}\label{minmaxminmax_op}
    \begin{aligned}
      & \underset{\cmin}{\min}\
    \underset{\gamma\in\Omega_{\gamma} }{\max}
    \ \underset{s'\in\Omega_{s'}}{\min}
    \ \colbit{\hezds}[G_{op}(f,a,b,\gamma,\bm{z},s')] \\
    &\Leftrightarrow
    \underset{\cmin}{\min}\
    \underset{\gamma\in\Omega_{\gamma} }{\max}
    \ \underset{s'\in\Omega_{s'}}{\min} \underset{c}{\max}
    \ \colbit{\hezds}[H_{op}(f,a,b,\gamma,\bm{z},s',c)],
    \end{aligned}
\end{equation}
where
\begin{equation}
    \begin{aligned}
    H_{op}&(f,a,b,\gamma,\bm{z},s',c) = P(f,a,\gamma,\bm{x})\cdot y/p-\gamma^2\\
    & +
    \left(\beta s' + c(\bm{x})\cdot \left(N(f,b,\gamma,\bm{x})-s'\right)\right)\cdot (1-y)/[\beta(1-p)].
    \end{aligned}
\end{equation}

\begin{rem}
    One might recall that \cite{yang2021all} and \cite{zhu2022auc} also introduce sample weights to deal with the quantile selection. The fundamental distinction of our approach from these methods lies in the fact that our sample weighting scheme is established on the instance-wise reformulation. This simultaneously ensures a reduced computational complexity and an unbiased estimation.
\end{rem}

Even though the resulting problem exhibits a min-max-min-max form, its convexity/concavity with respect to the inner variables enables a direct swapping of the min and max operations, ultimately converting it into a conventional minimax problem.

\noindent\textbf{Min-Max Swapping.}
Due to the convexity \wrt\ $s'$ and concavity \wrt\ $c,\gamma$ of $H_{op}$ on the convex and compact inner domains (fixing $(f,a,b)$), the minimax theorem can be applied twice to swap the respective inner minimization and maximization operations, yielding:
\begin{align}\label{eq:nonasym_op}
    &\underset{\cmin}{\min}\ \underset{\cmax }{\max} \ \underset{s'\in\Omega_{s'}}{\min}\ \underset{c}{\max} \colbit{\hezds}[H_{op}] \nonumber \\
    \Leftrightarrow&\underset{\cmin}{\min}\ \underset{\cmax, c}{\max} \ \underset{s'\in\Omega_{s'}}{\min} \ \colbit{\hezds}[H_{op}] \nonumber \\
    \Leftrightarrow& \underset{\cmin, s'\in\Omega_{s'}}{\min}\ \underset{\cmax, c}{\max} \ \colbit{\hezds}[H_{op}],
\end{align}
where $H_{op}=H_{op}(f,a,b,\gamma,\bm{z},s',c)$.
\begin{rem}
    In practice, the choice between the two formulations hinges on the trade-off between efficiency and exactness. The surrogate approximation is preferable when training cost or memory is the primary concern, while the unbiased reformulation is more suitable when the elimination of approximation bias is critical.
\end{rem}

\begin{rem}
    Our derivations can extend beyond the squared margin loss to a broad family of convex and continuous surrogates (e.g., logistic, exponential). By Bernstein-polynomial approximation \cite{yang2020stochastic}, such losses can be uniformly approximated while preserving convexity/monotonicity, so the instance-wise minimax structure and linear per-iteration complexity remain intact. Overall, the extension requires no change to the training pipeline; only the surrogate-specific coefficients/updates differ.
\end{rem}

\subsection{Extension to TPAUC}\label{sec:tpauc}
According to Eq.\eqref{TPAUCM}, given a surrogate loss $\ell$ and finite dataset $S$, maximizing $\widehat{\mathrm{AUC}}_{\alpha, \beta}(f, S)$ is equivalent to solving the following problem:
\begin{equation}
\begin{aligned}
\underset{f}{\min} \ \hat{\mathcal{R}}_{\alpha, \beta}(f, S)= \sum_{i=1}^{\npa} \sum_{j=1}^{\nnb}\frac{\ell{\left(f(\bm{x}_{[i]})- f(\bm{x}'_{[j]})\right)}}{n_+^\alpha n_-^\beta}.
\end{aligned}
\end{equation}
The extension to $\tp$ follows the same pipeline developed for $\op$. The key difference lies in incorporating the TPR constraint $\alpha$, which introduces an additional threshold variable $s$ for selecting hard positive instances.
Due to the limited space, we present the result directly, please refer to Appx.\ref{sec:tpauc_reformulation} for more details.

\begin{itemize}
\item \textbf{Surrogate Approximation:}
\begin{equation}
    \min_{\substack{f, (a, b) \in [0, 1]^2, \\(s, s') \in \Omega_{s'}^2}} \max_{\gamma \in \Omega_{\gamma}} \colbit{\underset{\bm{z}\sim S}{\hat{\mathbb{E}}}}\left[G_{t p}^{\colblue{\kappa}, \colbit{\omega}}\left(f, a, b, \gamma, \bm{z}, s, s'\right)\right],
    \label{eq:asym_tp_em}
\end{equation}
where $\Omega_{\gamma}=[\max\{-a,b-1\},1]$ and
\begin{equation}
\begin{aligned}
&G_{tp}^{\colblue{\kappa},\colbit{\omega}}(f,a,b,\gamma,\bm{z},s,s') \\
=&\left(\alpha s + r_{{\colblue{\kappa}}}\left(P(f,a,\gamma,\bm{x})-s\right)\right)\cdot y/(\alpha p) -(\colbit{\omega}+1)\gamma^2\\
& + \left(\beta s' +r_{{\colblue{\kappa}}}\left(N(f,b,\gamma,\bm{x})-s'\right)\right)\cdot(1-y)/[\beta (1-p)].
\label{TPAUC_OB}
\end{aligned}
\end{equation}
When $\alpha=1$, it degenerates to $G_{op}^{\colblue{\kappa},\colbit{\omega}}$.

\item \textbf{Unbiased Reformulation:}
\begin{equation}\label{eq:minimax_unbiased_tp}
    \min_{\substack{f, (a, b) \in [0, 1]^2, \\(s, s') \in \Omega_{s'}^2}} \max_{\gamma \in \Omega_{\gamma}, c} \colbit{\underset{\bm{z}\sim S}{\hat{\mathbb{E}}}}\left[H_{tp}(f,a,b,\gamma,\bm{z},s,s',c)\right],
\end{equation}
where $\Omega_{\gamma}=[\max\{-a,b-1\},1]$,
\begin{equation}
\begin{aligned}
&H_{tp}(f,a,b,\gamma,\bm{z},s,s',c) \\
=&(\alpha s + c(\bm{x})\cdot[P(f,a,\gamma,\bm{x})-s])\cdot y/(\alpha p) - \gamma^2\\
 +& (\beta s' + c(\bm{x})\cdot[N(f,b,\gamma,\bm{x})-s'])\cdot(1-y)/[\beta (1-p)],
\end{aligned}
\end{equation}
and the weight function $c(\bm{x})\in[0,1]$ can be instantiated by setting individual weight variable for each instance.
\end{itemize}

Note that the domain $\Omega_{\gamma}$ is still coupled with $a$ and $b$ in all the above formulations. According to Thm.3.4 in \cite{tsaknakis2021minimax}, we can use Lagrange multipliers to decouple these constraints and obtain standard unconstrained minimax forms. For example, the unbiased formulation can be formalized as follows.
\begin{coro}\label{coro:lagrange_asym}
    \textbf{(Lagrangian Decoupling.)}
Eq.\eqref{eq:minimax_unbiased_tp} is equivalent to the following minimax problem with Lagrange multipliers $\theta_b, \theta_a$:
\begin{equation}\label{eq:nonasym_tp_lag}
\begin{aligned}
    &\min_{\substack{f,(a,b)\in[0,1]^2,\\ (s, s') \in \Omega_{s'}^2}} \max_{\gamma\in[\max\{-a,b-1\},1],c} \colbit{\underset{\bm{z}\sim S}{\hat{\mathbb{E}}}}[H_{tp}]
\\ \Leftrightarrow
&\min_{\substack{f,(a,b)\in[0,1]^2,\\ (s, s') \in \Omega_{s'}^2,\\ \theta_a \in[0,M_2],\theta_b \in[0,M_3]}} \max_{\gamma\in[-1,1],c} \colbit{\underset{\bm{z}\sim S}{\hat{\mathbb{E}}}}[H_{tp}] \\
 &\qquad \qquad \qquad \qquad \quad -\theta_{b}(b-1-\gamma)- \theta_{a}(-a-\gamma).
\end{aligned}
\end{equation}
\end{coro}

\begin{rem}
    For $\op$, $\theta_a$ can be set to 0. The tight constraints $\theta_b \in [0,M_1], \theta_a \in [0,M_2]$ come from the fact that optimum $\theta_b,\theta_a$ are both finite since the objective function is bounded from above. To make sure that $M_1,M_2$ are sufficiently large, we set $M_1 = M_2 =10^9$ in experiments.
\end{rem}

Similar to Coro.\ref{coro:lagrange_asym}, we can also obtain the equivalent Lagrangian form for the approximated formulations $G_{(\cdot)}^{\colblue{\kappa},\colbit{\omega}}$ in Eq.\eqref{minmax_op_em} and Eq.\eqref{eq:asym_tp_em}. For brevity, the Lagrangian form objective functions are denoted as $\breve{H}_{(\cdot)}$ and $\breve{G}_{(\cdot)}^{\colblue{\kappa},\colbit{\omega}}$.

\section{Stochastic Optimization Algorithm}\label{sec:alg}
Our goal is then to solve the resulting empirical non-convex concave minimax optimization problems Eq.\eqref{eq:nonasym_tp_lag}. Following the similar spirit of \cite{zhang2022sapd+}, we optimize the nonconvex strongly-concave surrogate $\breve{H}^\omega_{(\cdot)} = \breve{H}_{(\cdot)} - \omega \cdot (\gamma^2 + \norm{\bm{c}}^2)$ with a small $\omega$ for faster convergence. Then based on the work \cite{huang2022accelerated}, we employ an accelerated stochastic gradient descent ascent (ASGDA) method to solve the minimax problem.

Denote $\bm{\theta}\in\mathbb{R}^d$ as the parameters of function $f$, $\bm{\tau}=\{\bm{\theta},a,b,s,s',\theta_a,\theta_b\}\in \Omega_{\bm{\tau}}$ as the variables for the outer min-problem, $\bm{\gamma} = \{\gamma, \bm{c}\} \in \Omega_{\bm{\gamma}}$ as the variables for the inner max-problem. $\breve{H}^\omega_{(\cdot)}$ on a minibatch $\mathcal{B}$ can be denoted as $\breve{H}^\omega_{(\cdot)}(\bm{\tau},\bm{\gamma};\mathcal{B})$. Then Alg.\ref{alg:1} shows the procedure of ASGDA following \cite{huang2022accelerated}. There are two key steps: (1) \texttt{Line 5-6}: variables $\bm{\tau}_{t+1}$ and $\bm{\gamma}_{t+1}$ are updated in a momentum way. Moreover, the convex combination ensures that they are always feasible given that the initial solution is feasible. (2) \texttt{Line 9-10}: using the momentum-based variance reduction technique, we can estimate the stochastic first-order partial gradients $\bm{v}_t$ and $\bm{w}_t$ in a more stable manner.

\begin{algorithm}[t]
    \caption{ASGDA}
    \begin{algorithmic}[1]
    \label{alg:1}
    \STATE \textbf{Input}:\,Dataset $S$,\,hyperparameters\,$\{\nu, \lambda, k, m, \iota_1, \iota_2, T\}$
    \STATE \textbf{Initialize:} Randomly select $\bm{\tau}_0=\{\bm{\theta}_0$, $a_0$, $b_0$, $s_0$, $s^\prime_0$, $\theta_a$, $\theta_b\}$ from $\Omega_{\bm{\tau}}$, $\bm{v}_0=\bm{0}^{d+6}$, $\bm{\gamma}_0$ from $\Omega_{\bm{\gamma}}$, $\bm{w}_0=\bm{0}$.
    \FOR{$t=0,1,\cdots,T$}
        \STATE Compute the coefficient $\eta_t=\frac{k}{(m+t)^{1/3}}$;
        \STATE Update $\bm{\tau}_{t+1}=(1-\eta_t)\bm{\tau}_t + \eta_t \setP_{\Omega_{\bm{\tau}}}(\bm{\tau}_t-\nu \bm{v}_t)$;
        \STATE Update $\bm{\gamma}_{t+1} = (1-\eta_t)\bm{\gamma}_t+\eta_t \setP_{\Omega_{\bm{\gamma}}}(\bm{\gamma}_t+\lambda \bm{w}_t)$;
        \STATE Compute $\rho_{t+1}=\iota_1\eta_t^2$ and $\xi_{t+1}=\iota_2 \eta_t^2$;
        \STATE Sample a mini-batch of data $\mathcal{B}_{t+1}$ from dataset $S$;  \\ \rightline{$\rhd$ Replace $\breve{H}^\omega_{(\cdot)}$ with $\breve{G}_{(\cdot)}^{\colblue{\kappa},\colbit{\omega}}$ for the approximated form}
        \STATE Update $\bm{v}_{t+1}=\nabla_{\bm{\tau}}\breve{H}^\omega_{(\cdot)}(\bm{\tau}_{t+1},\bm{\gamma}_{t+1};\mathcal{B}_{t+1})+(1-\rho_{t+1})[\bm{v}_{t}-\nabla_{\bm{\tau}}\breve{H}^\omega_{(\cdot)}(\bm{\tau}_t,\bm{\gamma}_t,\mathcal{B}_{t+1})]$;
        \STATE Update
        $\bm{w}_{t+1}=\nabla_{\bm{\gamma}}\breve{H}^\omega_{(\cdot)}(\bm{\tau}_{t+1},\bm{\gamma}_{t+1};\mathcal{B}_{t+1})+(1-\xi_{t+1})[\bm{w}_{t}-\nabla_{\bm{\gamma}}\breve{H}^\omega_{(\cdot)}(\bm{\tau}_t,\bm{\gamma}_t,\mathcal{B}_{t+1})]$;
    \ENDFOR
    \STATE \textbf{Return} $\bm{\tau}_{T+1}$ and $\bm{\gamma}_{T+1}$
    \end{algorithmic}
\end{algorithm}

The convergence rate of Alg.\ref{alg:1} is presented in Thm.\ref{thm:opt}.
\begin{thm}\label{thm:opt}
    (Thm.9 of \cite{huang2022accelerated}) Supposing that $\breve{H}^\omega_{(\cdot)}(\bm{\tau},\bm{\gamma};\mathcal{B})$ has $L_H$-Lipschitz gradients \wrt $\bm{\tau}$ and $\bm{\gamma}$, let $\{\bm{\tau}_t,\bm{\gamma}_t\}$ be a sequence generated by Alg.\ref{alg:1}, and $\mu$ be the strongly-concavity constant of the objective, if the learning rates $\nu,\lambda$ satisfy:
   \begin{equation}
   \begin{aligned}
       &\iota_1\geq \frac{2}{3k^3}+\frac{9\mu^2}{4}, \qquad \iota_2\geq \frac{2}{3k^3}+\frac{75L_H^2}{2}, \qquad k>0\\
       &m\geq \max(2, k^3, (\iota_1 k)^3, (\iota_2 k)^3),~~ \lambda \leq \min\left(\frac{1}{6L_H},\frac{27\mu}{16}\right)\\
       &\nu\leq \min\Big(\frac{\lambda \mu}{2L_H}\sqrt{\frac{2}{8\lambda^2+75(L_H/\mu)^2}},~\frac{m^{1/3}}{2L_H(1+\frac{L_H}{\mu})k}\Big),
   \end{aligned}
   \end{equation}
   then we have:
   \begin{equation}
       \frac{1}{T} \sum_{t=1}^{T} \E\left[\left\|\frac{1}{\nu}\big(\bm{\tau}_t-\setP_{\Omega_{\bm{\tau}}}(\bm{\tau}_{t}-\nu F_{(\cdot)}(\bm{\tau}_t))\big)\right\|\right]
       \leq O\big(\frac{(L_H/\mu)^{3/2}}{T^{1/3}}\big),
   \end{equation}
   where $\|\frac{1}{\nu}(\bm{\tau}_t-\setP_{\Omega_{\bm{\tau}}}(\bm{\tau}_{t}-\nu \nabla F_{(\cdot)}(\bm{\tau}_t)))\|$ is the $l_2$-norm of gradient mapping metric for the outer problem \cite{dunn1987convergence, ghadimi2016mini, razaviyayn2020nonconvex} with $F_{(\cdot)}(\bm{\tau}_t)=\max_{\bm{\gamma}\in\Omega_{\bm{\gamma}}}\breve{H}^\omega_{(\cdot)}(\bm{\tau}_t,\bm{\gamma})$.
   \label{thm:algorithm}
   \end{thm}

\begin{rem}
    By replacing $\breve{H}^\omega_{(\cdot)}$ with $\breve{G}_{(\cdot)}^{\colblue{\kappa},\colbit{\omega}}$\footnote{It is easy to check that they are strongly-concave w.r.t $\gamma$ whenever $\kappa \le 2 + 2\omega$.} and using $\bm{\gamma} = \{\gamma\}$, we could obtain a similar result for the approximated formulation, \ie, the convergence rate is $O(\frac{(L_H/\mu)^{3/2}}{T^{1/3}})$. By $\frac{(L_H/\mu)^{3/2}}{T^{1/3}}\leq \epsilon$, then the iteration number to achieve the $\epsilon$-first-order saddle point satisfies: $T\geq (L_H/\mu)^{4.5}\epsilon^{-3}$.
\end{rem}

\begin{table*}[ht]
    \centering
    \caption{Dataset statistics.}
    \begin{tabular}{lllll}
        \toprule
        Dataset & Pos. Class ID & Pos. Class Name & \# Pos& \#Neg \\
        \midrule
        CIFAR-10-LT-1 & 2 & birds & 1,508 & 8,907 \\
        CIFAR-10-LT-2 & 1 & automobiles & 2,517 & 7,898 \\
        CIFAR-10-LT-3 & 3 & birds & 904 & 9,511 \\
        \midrule
        CIFAR-100-LT-1 & 6,7,14,18,24 & insects & 1,928 & 13,218 \\
        CIFAR-100-LT-2 & 0,51,53,57,83 & fruits and vegatables & 885 & 14,261 \\
        CIFAR-100-LT-3 & 15,19,21,32,38 & large omnivores herbivores  & 1,172 & 13,974 \\
        \midrule
        Tiny-ImageNet-200-LT-1 & 24,25,26,27,28,29 & dogs & 2,100 & 67,900 \\
        Tiny-ImageNet-200-LT-2 & 11,20,21,22 & birds & 1,400 & 68,600 \\
        Tiny-ImageNet-200-LT-3 & 70,81,94,107,111,116,121,133,145,153,164,166 & vehicles & 4,200 & 65,800 \\
        \midrule
        iNaturalist2021 & / & Fungi & 20,460 & 151,560 \\
        \bottomrule
    \end{tabular}
    \label{tab:dataset}
\end{table*}

\section{Generalization Analysis}\label{sec:generalization}


In this section, we theoretically analyze the generalization performance of our proposed estimators (please see Appx.\ref{sec:proof_generalization} for the proof).

For $\op$, the population risk can be defined as
\begin{equation}
    \mathcal{R}_{\beta}(f) = \underset{\bm{x} \sim \setD_\setP, \bm{x}'\sim \setD_\setN}{\E} \left[\indicator{f(\bm{x}') \ge \efb} \cdot \ell(f(\bm{x}) - f(\bm{x}'))\right].
\end{equation}
To prove the uniform convergence result over a hypothesis class $\setF$ of the scoring function $f$, we need to show that:
\begin{equation*}
    \begin{split}
        \sup_{f \in \setF}\left[ \mathcal{R}_{\beta}(f)  - \hat{\mathcal{R}}_{\beta}(f) \right] \le \epsilon,
    \end{split}
\end{equation*}
holds with high probability.

According to Thm.\ref{thm:instance_difftopk_op} in Sec.\ref{sec:instance}, we know that the generalization error of OPAUC with the surrogate loss $\ell$ can be measured as:
\begin{equation}
    \mathcal{R}_{\beta}(f) \propto \underset{\cmins}{\min}\
    \underset{\cmax }{\max}  \underset{s'\in\Omega_{s'}}{\min}
    \ \colblue{\ezdz}[G_{op}(f,a,b,\gamma,\bm{z},s')],
\end{equation}
and
\begin{equation}
    \hat{\mathcal{R}}_{\beta}(f) \propto  \underset{\cmins}{\min}\
    \underset{\cmax }{\max}  \underset{s'\in\Omega_{s'}}{\min}
    \ \colbit{\hezds}[G_{op}(f,a,b,\gamma,\bm{z},s')].
\end{equation}
Consequently, we only need to prove that:
\begin{equation*}
  \begin{split}
    \sup_{f \in \setF} \Big[ & \underset{\cmins}{\min}\
    \underset{\cmax }{\max}  \underset{s'\in\Omega_{s'}}{\min}
    \ \colblue{\ezdz}[G_{op}(f,a,b,\gamma,\bm{z},s')] \\
    & - \underset{\cmins}{\min}\
    \underset{\cmax }{\max}  \underset{s'\in\Omega_{s'}}{\min}
    \ \colbit{\hezds}[G_{op}(f,a,b,\gamma,\bm{z},s')] \Big] \le \epsilon.
  \end{split}
\end{equation*}

After relaxation, the instance-wise loss function $G_{op}$ allows us to easily employ the standard Rademacher complexity based technique \cite{mohri2018foundations} together with the covering numbers, forming a generalization bound usually with the order of $\tilde{O}(\np^{-1/2} + \beta^{-1}\nn^{-1/2})$. Nevertheless, this order is suboptimal in some situations \cite{bartlett2005local}. Here we take a step further by incorporating the local Rademacher complexity measure \cite{bartlett2005local} into the analysis. Since the local Rademacher complexity only considers a small region in the hypothesis class (\eg, the hypothesis function will small empirical errors), it often leads to a sharper bound.

\begin{thm}
Assume there exist three positive constants $R, D$ and $h$ such that the following bound holds for the covering number of $\setF$ \wrt\ $\norm{\cdot}_2$ norm:
\begin{equation}
    \log \setN(\epsilon, \setF, \norm{\cdot}_2) \leq D\log^h(R/\epsilon).
\end{equation}
Then for any $\delta>0$, with probability at least $1-\delta$ over the draw of an i.i.d. sample set $S$ of size $n$ ($n \geq R^{-2}$), for all $f\in\setF$ and $K>1$ we have:
\begin{equation*}
\begin{aligned}
&\underset{\cmins}{\min}
\underset{\cmax }{\max}  \underset{s'\in\Omega_{s'}}{\min}
\ \colblue{\ezdz}[G_{op}(f,a,b,\gamma,\bm{z},s')] \\
\le& \frac{K}{K-1}\ \underset{\cmins}{\min}\  \underset{\cmax }{\max}  \underset{s'\in\Omega_{s'}}{\min}
\ \colbit{\hezds}[G_{op}(f,a,b,\gamma,\bm{z},s')] \\
&+ \tilde{O}(\np^{-1} + \beta^{-1}\nn^{-1}).
\end{aligned}
\end{equation*}
\label{thm:4}
\end{thm}

Similarly, we could derive the generalization bound of $\tp$ which is with the order of $\tilde{O}(\alpha^{-1}\np^{-1} + \beta^{-1}\nn^{-1})$ (please see Appx.\ref{sec:proof_gene_tp} for the details). When $\alpha=1$, it could recover the case of $\op$.

\begin{rem}
    The assumption for the covering number applies to most popular models ranging from linear models to deep neural networks \cite{lei2016local,longgeneralization}. For example, the proof of Lem.2.3 in \cite{longgeneralization} shows that when a network is $(B,d)$-Lipschitz parametrized, our assumption holds by setting $R=3B, D=d, h=1$. \cite{longgeneralization} also demonstrates that such $B$ is usually large for neural networks (which could degenerate to linear models), and thus the condition $n \geq R^{-2}$ is also very easy to satisfy.
\end{rem}

\begin{rem}
    Compared with previous studies \cite{yang2021all, narasimhan2017support, DBLP:journals/pami/YangXBHCH23}, our generalization analysis is simpler and does not require complex error decomposition. Moreover, our results are sharp and hold for all real-valued hypothesis class with outputs in $[0,1]$, while the results in \cite{yang2021all, narasimhan2017support} only hold for hard-threshold functions with a limited order of $\tilde{O}(\np^{-1/2} + \nn^{-1/2})$. On the other hand, unlike \cite{DBLP:journals/pami/YangXBHCH23}, our bound is dependent on $\alpha$/$\beta$, explicitly illustrating the impact of the TPR/FPR constraints on the generalization.
\end{rem}

\begin{rem}
Since $H_{op}$ is an unbiased formulation of $G_{op}$, its generalization gap also takes the order of $\tilde{O}(\np^{-1} + \beta^{-1}\nn^{-1})$.
For the surrogate approximation $G_{op}^{\colblue{\kappa}}$, it is easy to derive that the bound exhibits an extra term of $O(\kappa^{-1})$ which is the approximation bias.
\end{rem}


\begin{table*}[!t]
    \centering
    \caption{
        OPAUC ($\mathrm{FPR}\leq \beta$) on CIFAR-10-LT with different $\beta$. The \best{best} and \secbest{second best} results are both highlighted.
    }
    \label{tab:op_cifar10}
    \begin{tabular}{l|cc|cc|cc}
    \toprule
        \multirow{2}*{Method} & \multicolumn{2}{c|}{CIFAR-10-LT-1} & \multicolumn{2}{c|}{CIFAR-10-LT-2} & \multicolumn{2}{c}{CIFAR-10-LT-3} \\
        \cmidrule{2-7}~ & 0.3 & 0.5 & 0.3 & 0.5 & 0.3 & 0.5 \\
        \midrule
        CE & 0.7417 $\pm$ .0033 & 0.8350 $\pm$ .0046 & 0.9431 $\pm$ .0028 & 0.9551 $\pm$ .0037 & 0.7428 $\pm$ .0007 & 0.8387 $\pm$ .0018 \\
        AUC-M \cite{ying2016stochastic} & 0.7334 $\pm$ .0044 & 0.8267 $\pm$ .0057 & 0.9609 $\pm$ .0009 & 0.9729 $\pm$ .0015 & 0.7442 $\pm$ .0084 & 0.8411 $\pm$ .0097 \\
        MB \cite{kar2014online} & 0.7492 $\pm$ .0079 & 0.8425 $\pm$ .0092 & 0.9648 $\pm$ .0098 & 0.9768 $\pm$ .0109 & 0.7500 $\pm$ .0034 & 0.8469 $\pm$ .0047 \\
        SOPA \cite{zhu2022auc} & 0.7659 $\pm$ .0067 & 0.8481 $\pm$ .0070 & 0.9688 $\pm$ .0045 & 0.9801 $\pm$ .0058 & 0.7651 $\pm$ .0052 & 0.8500 $\pm$ .0063 \\
        SOPA-S \cite{zhu2022auc} & 0.7548 $\pm$ .0042 & 0.8470 $\pm$ .0055 & 0.9674 $\pm$ .0017 & 0.9794 $\pm$ .0029 & 0.7542 $\pm$ .0076 & 0.8491 $\pm$ .0089 \\
        AGD-SBCD \cite{yao2022large} & 0.7526 $\pm$ .0031 & 0.8459 $\pm$ .0044 & 0.9615 $\pm$ .0056 & 0.9735 $\pm$ .0070 & 0.7497 $\pm$ .0005 & 0.8468 $\pm$ .0016 \\
        AUC-poly \cite{yang2021all} & 0.7542 $\pm$ .0075 & \best{0.8475 $\pm$ .0088} & 0.9672 $\pm$ .0031 & 0.9792 $\pm$ .0044 & 0.7538 $\pm$ .0026 & 0.8497 $\pm$ .0039 \\
        AUC-exp \cite{yang2021all} & 0.7347 $\pm$ .0012 & 0.8280 $\pm$ .0025 & 0.9620 $\pm$ .0084 & 0.9740 $\pm$ .0097 & 0.7457 $\pm$ .0097 & 0.8416 $\pm$ .0109 \\
        \midrule
        PAUCI \cite{shao2022asymptotically} (Ours) & \secbest{0.7721 $\pm$ .0086} & 0.8454 $\pm$ .0099 & \best{0.9716 $\pm$ .0092} & \best{0.9838 $\pm$ .0105} & \secbest{0.7746 $\pm$ .0001} & \secbest{0.8505 $\pm$ .0014} \\
        UPAUCI (Ours) & \best{0.7764 $\pm$ .0063} & \secbest{0.8472 $\pm$ .0045} & \secbest{0.9702 $\pm$ .0049} & \secbest{0.9809 $\pm$ .0043} & \best{0.8061 $\pm$ .0036} & \best{0.8692 $\pm$ .0030} \\
    \bottomrule
    \end{tabular}
\end{table*}

\begin{table*}[!t]
    \centering
    \caption{
        OPAUC ($\mathrm{FPR}\leq \beta$) on CIFAR-100-LT with different $\beta$.
    }
    \label{tab:op_cifar100}
    \begin{tabular}{l|cc|cc|cc}
    \toprule
        \multirow{2}*{Method} & \multicolumn{2}{c|}{CIFAR-100-LT-1} & \multicolumn{2}{c|}{CIFAR-100-LT-2} & \multicolumn{2}{c}{CIFAR-100-LT-3} \\
        \cmidrule{2-7}~ & 0.3 & 0.5 & 0.3 & 0.5 & 0.3 & 0.5 \\
        \midrule
        CE & 0.8903 $\pm$ .0046 & 0.9070 $\pm$ .0060 & 0.9695 $\pm$ .0082 & 0.9725 $\pm$ .0074 & 0.8321 $\pm$ .0015 & 0.8553 $\pm$ .0024 \\
        AUC-M \cite{ying2016stochastic} & 0.8996 $\pm$ .0002 & 0.9163 $\pm$ .0066 & 0.9845 $\pm$ .0032 & 0.9875 $\pm$ .0078 & 0.8403 $\pm$ .0060 & 0.8635 $\pm$ .0075 \\
        MB \cite{kar2014online} & 0.9003 $\pm$ .0066 & 0.9170 $\pm$ .0059 & 0.9804 $\pm$ .0041 & 0.9834 $\pm$ .0082 & \best{0.8575 $\pm$ .0057} & \secbest{0.8807 $\pm$ .0030} \\
        SOPA \cite{zhu2022auc} & 0.9108 $\pm$ .0073 & \best{0.9275 $\pm$ .0057} & 0.9875 $\pm$ .0088 & 0.9905 $\pm$ .0076 & 0.8483 $\pm$ .0096 & 0.8715 $\pm$ .0029 \\
        SOPA-S \cite{zhu2022auc} & 0.9033 $\pm$ .0029 & 0.9200 $\pm$ .0067 & 0.9860 $\pm$ .0061 & 0.9890 $\pm$ .0080 & 0.8449 $\pm$ .0087 & 0.8681 $\pm$ .0026 \\
        AGD-SBCD \cite{yao2022large} & 0.9105 $\pm$ .0011 & \secbest{0.9272 $\pm$ .0061} & 0.9814 $\pm$ .0069 & 0.9844 $\pm$ .0090 & 0.8406 $\pm$ .0044 & 0.8638 $\pm$ .0068 \\
        AUC-poly \cite{yang2021all} & 0.9027 $\pm$ .0084 & 0.9194 $\pm$ .0076 & 0.9859 $\pm$ .0072 & 0.9889 $\pm$ .0085 & 0.8441 $\pm$ .0053 & 0.8673 $\pm$ .0034 \\
        AUC-exp \cite{yang2021all} & 0.8987 $\pm$ .0071 & 0.9154 $\pm$ .0062 & 0.9850 $\pm$ .0022 & 0.9880 $\pm$ .0094 & 0.8407 $\pm$ .0038 & 0.8639 $\pm$ .0045 \\
        \midrule
        PAUCI \cite{shao2022asymptotically} (Ours) & \secbest{0.9155 $\pm$ .0014} & 0.9172 $\pm$ .0064 & \secbest{0.9889 $\pm$ .0053} & \secbest{0.9919 $\pm$ .0074} & 0.8492 $\pm$ .0020 & 0.8722 $\pm$ .0039 \\
        UPAUCI (Ours) & \best{0.9225 $\pm$ .0051} & 0.9221 $\pm$ .0051 & \best{0.9905 $\pm$ .0069} & \best{0.9937 $\pm$ .0071} & \secbest{0.8503 $\pm$ .0062} & \best{0.8851 $\pm$ .0049} \\
    \bottomrule
    \end{tabular}
\end{table*}

\begin{table*}[!t]
    \centering
    \caption{
        OPAUC ($\mathrm{FPR}\leq \beta$) on Tiny-ImageNet-200-LT with different $\beta$.
    }
    \label{tab:op_tinyimagenet}
    \begin{tabular}{l|cc|cc|cc}
    \toprule
        \multirow{2}*{Method} & \multicolumn{2}{c|}{Tiny-ImageNet-200-LT-1} & \multicolumn{2}{c|}{Tiny-ImageNet-200-LT-2} & \multicolumn{2}{c}{Tiny-ImageNet-200-LT-3} \\
        \cmidrule{2-7}~ & 0.3 & 0.5 & 0.3 & 0.5 & 0.3 & 0.5 \\
        \midrule
        CE & 0.8023 $\pm$ .0081 & 0.8681 $\pm$ .0056 & 0.8917 $\pm$ .0077 & 0.9296 $\pm$ .0049 & 0.8878 $\pm$ .0004 & 0.9226 $\pm$ .0010 \\
        AUC-M \cite{ying2016stochastic} & 0.8102 $\pm$ .0023 & 0.8760 $\pm$ .0028 & 0.9011 $\pm$ .0078 & 0.9390 $\pm$ .0054 & 0.9043 $\pm$ .0067 & 0.9391 $\pm$ .0053 \\
        MB \cite{kar2014online} & 0.8193 $\pm$ .0018 & \secbest{0.8851 $\pm$ .0007} & 0.9072 $\pm$ .0049 & 0.9451 $\pm$ .0025 & 0.9091 $\pm$ .0063 & 0.9439 $\pm$ .0049 \\
        SOPA \cite{zhu2022auc} & 0.8157 $\pm$ .0035 & 0.8815 $\pm$ .0029 & 0.9037 $\pm$ .0021 & 0.9416 $\pm$ .0011 & 0.9066 $\pm$ .0059 & 0.9414 $\pm$ .0043 \\
        SOPA-S \cite{zhu2022auc} & 0.8180 $\pm$ .0048 & 0.8838 $\pm$ .0022 & 0.9087 $\pm$ .0079 & 0.9466 $\pm$ .0053 & 0.9095 $\pm$ .0099 & 0.9443 $\pm$ .0085 \\
        AGD-SBCD \cite{yao2022large} & 0.8135 $\pm$ .0095 & 0.8793 $\pm$ .0067 & 0.9081 $\pm$ .0043 & 0.9460 $\pm$ .0017 & 0.9057 $\pm$ .0012 & 0.9405 $\pm$ .0004 \\
        AUC-poly \cite{yang2021all} & 0.8185 $\pm$ .0086 & 0.8843 $\pm$ .0058 & 0.9084 $\pm$ .0065 & 0.9463 $\pm$ .0041 & 0.9100 $\pm$ .0030 & 0.9448 $\pm$ .0018 \\
        AUC-exp \cite{yang2021all} & 0.8127 $\pm$ .0073 & 0.8785 $\pm$ .0046 & 0.9026 $\pm$ .0050 & 0.9405 $\pm$ .0026 & 0.9049 $\pm$ .0016 & 0.9397 $\pm$ .0002 \\
        \midrule
        PAUCI \cite{shao2022asymptotically} (Ours) & \secbest{0.8267 $\pm$ .0040} & 0.8825 $\pm$ .0021 & \secbest{0.9214 $\pm$ .0036} & \secbest{0.9492 $\pm$ .0010} & \secbest{0.9217 $\pm$ .0051} & \secbest{0.9465 $\pm$ .0037} \\
        UPAUCI (Ours) & \best{0.8559 $\pm$ .0013} & \best{0.9064 $\pm$ .0028} & \best{0.9515 $\pm$ .0027} & \best{0.9612 $\pm$ .0075} & \best{0.9558 $\pm$ .0053} & \best{0.9616 $\pm$ .0038} \\
    \bottomrule
    \end{tabular}
\end{table*}

\section{Experiments}\label{sec:exp}

\begin{table*}[t]
    \centering
    \caption{
        TPAUC ($\mathrm{TPR} \geq \alpha$, $\mathrm{FPR}\leq \beta$) on CIFAR-10-LT with different $(\alpha, \beta)$.
    }
    \label{tab:tp_cifar10}
    \begin{tabular}{l|cc|cc|cc}
    \toprule
        \multirow{2}*{Method} & \multicolumn{2}{c|}{CIFAR-10-LT-1} & \multicolumn{2}{c|}{CIFAR-10-LT-2} & \multicolumn{2}{c}{CIFAR-10-LT-3} \\
        \cmidrule{2-7}~ & (0.3, 0.3) & (0.5, 0.5) & (0.3, 0.3) & (0.5, 0.5) & (0.3, 0.3) & (0.5, 0.5) \\
        \midrule
        CE & 0.2855 $\pm$ .0034 & 0.6420 $\pm$ .0038 & 0.8811 $\pm$ .0073 & 0.9353 $\pm$ .0067 & 0.3636 $\pm$ .0008 & 0.6798 $\pm$ .0014 \\
        AUC-M \cite{ying2016stochastic} & 0.2955 $\pm$ .0058 & 0.6520 $\pm$ .0064 & 0.8839 $\pm$ .0084 & 0.9381 $\pm$ .0078 & 0.3659 $\pm$ .0077 & 0.6821 $\pm$ .0083 \\
        MB \cite{kar2014online} & 0.2872 $\pm$ .0008 & 0.6437 $\pm$ .0012 & 0.8950 $\pm$ .0095 & 0.9492 $\pm$ .0089 & 0.3751 $\pm$ .0013 & 0.6913 $\pm$ .0017 \\
        SOPA \cite{zhu2022auc} & 0.3531 $\pm$ .0049 & 0.7096 $\pm$ .0052 & \best{0.9051 $\pm$ .0028} & 0.9593 $\pm$ .0024 & 0.4058 $\pm$ .0055 & 0.7220 $\pm$ .0071 \\
        SOPA-S \cite{zhu2022auc} & 0.3038 $\pm$ .0075 & 0.6603 $\pm$ .0081 & 0.8914 $\pm$ .0041 & 0.9456 $\pm$ .0035 & 0.3755 $\pm$ .0013 & 0.6917 $\pm$ .0027 \\
        AUC-poly \cite{yang2021all} & 0.3239 $\pm$ .0043 & 0.6804 $\pm$ .0049 & 0.9001 $\pm$ .0057 & 0.9543 $\pm$ .0051 & 0.3812 $\pm$ .0070 & 0.6974 $\pm$ .0086 \\
        AUC-exp \cite{yang2021all} & 0.3104 $\pm$ .0091 & 0.6669 $\pm$ .0097 & 0.8951 $\pm$ .0017 & 0.9493 $\pm$ .0011 & 0.3768 $\pm$ .0030 & 0.6930 $\pm$ .0044 \\
        \midrule
        PAUCI \cite{shao2022asymptotically} (Ours) & \secbest{0.3627 $\pm$ .0059} & \secbest{0.7192 $\pm$ .0065} & \secbest{0.9021 $\pm$ .0053} & \secbest{0.9663 $\pm$ .0047} & \secbest{0.4143 $\pm$ .0083} & \secbest{0.7305 $\pm$ .0099} \\
        UPAUCI (Ours) & \best{0.3886 $\pm$ .0035} & \best{0.7486 $\pm$ .0047} & 0.9015 $\pm$ .0053 & \best{0.9772 $\pm$ .0035} & \best{0.4386 $\pm$ .0023} & \best{0.7632 $\pm$ .0025} \\
    \bottomrule
    \end{tabular}
\end{table*}

\begin{table*}[t]
    \centering
    \caption{
        TPAUC ($\mathrm{TPR} \geq \alpha$, $\mathrm{FPR}\leq \beta$) on CIFAR-100-LT with different $(\alpha, \beta)$.
    }
    \label{tab:tp_cifar100}
    \begin{tabular}{l|cc|cc|cc}
    \toprule
        \multirow{2}*{Method} & \multicolumn{2}{c|}{CIFAR-100-LT-1} & \multicolumn{2}{c|}{CIFAR-100-LT-2} & \multicolumn{2}{c}{CIFAR-100-LT-3} \\
        \cmidrule{2-7}~ & (0.3, 0.3) & (0.5, 0.5) & (0.3, 0.3) & (0.5, 0.5) & (0.3, 0.3) & (0.5, 0.5) \\
        \midrule
        CE & 0.7074 $\pm$ .0096 & 0.8467 $\pm$ .0080 & 0.9350 $\pm$ .0061 & 0.9603 $\pm$ .0031 & 0.4841 $\pm$ .0127 & 0.7311 $\pm$ .0099 \\
        AUC-M \cite{ying2016stochastic} & 0.7112 $\pm$ .0073 & 0.8505 $\pm$ .0059 & 0.9569 $\pm$ .0048 & 0.9822 $\pm$ .0020 & 0.4854 $\pm$ .0097 & 0.7324 $\pm$ .0071 \\
        MB \cite{kar2014online} & 0.7272 $\pm$ .0038 & 0.8665 $\pm$ .0026 & 0.9424 $\pm$ .0123 & 0.9677 $\pm$ .0095 & \secbest{0.5113 $\pm$ .0067} & \secbest{0.7583 $\pm$ .0041} \\
        SOPA \cite{zhu2022auc} & 0.7321 $\pm$ .0046 & 0.8714 $\pm$ .0034 & \secbest{0.9602 $\pm$ .0047} & 0.9855 $\pm$ .0019 & 0.5015 $\pm$ .0073 & 0.7485 $\pm$ .0046 \\
        SOPA-S \cite{zhu2022auc} & 0.7224 $\pm$ .0060 & 0.8617 $\pm$ .0048 & 0.9559 $\pm$ .0102 & 0.9812 $\pm$ .0074 & 0.4949 $\pm$ .0040 & 0.7419 $\pm$ .0013 \\
        AUC-poly \cite{yang2021all} & 0.7225 $\pm$ .0087 & 0.8618 $\pm$ .0075 & 0.9582 $\pm$ .0096 & 0.9835 $\pm$ .0068 & 0.4961 $\pm$ .0115 & 0.7431 $\pm$ .0089 \\
        AUC-exp \cite{yang2021all} & 0.7220 $\pm$ .0075 & 0.8613 $\pm$ .0063 & 0.9574 $\pm$ .0085 & 0.9827 $\pm$ .0057 & 0.4977 $\pm$ .0062 & 0.7447 $\pm$ .0036 \\
        \midrule
        PAUCI \cite{shao2022asymptotically} (Ours) & \secbest{0.7411 $\pm$ .0053} & \secbest{0.8814 $\pm$ .0037} & 0.9601 $\pm$ .0051 & \secbest{0.9874 $\pm$ .0023} & 0.5027 $\pm$ .0034 & 0.7497 $\pm$ .0018 \\
        UPAUCI (Ours) & \best{0.7444 $\pm$ .0035} & \best{0.8878 $\pm$ .0056} & \best{0.9667 $\pm$ .0045} & \best{0.9882 $\pm$ .0039} & \best{0.5465 $\pm$ .0035} & \best{0.7801 $\pm$ .0046} \\
    \bottomrule
    \end{tabular}
\end{table*}

\begin{table*}[t]
    \centering
    \caption{
        TPAUC ($\mathrm{TPR} \geq \alpha$, $\mathrm{FPR}\leq \beta$) on Tiny-ImageNet-LT with different $(\alpha, \beta)$.
    }
    \label{tab:tp_tinyimagenet}
    \begin{tabular}{l|cc|cc|cc}
    \toprule
        \multirow{2}*{Method} & \multicolumn{2}{c|}{Tiny-ImageNet-200-LT-1} & \multicolumn{2}{c|}{Tiny-ImageNet-200-LT-2} & \multicolumn{2}{c}{Tiny-ImageNet-200-LT-3} \\
        \cmidrule{2-7}~ & (0.3, 0.3) & (0.5, 0.5) & (0.3, 0.3) & (0.5, 0.5) & (0.3, 0.3) & (0.5, 0.5) \\
        \midrule
        CE & 0.4211 $\pm$ .0035 & 0.7223 $\pm$ .0051 & 0.6662 $\pm$ .0073 & 0.8517 $\pm$ .0064 & 0.6541 $\pm$ .0088 & 0.8478 $\pm$ .0072 \\
        AUC-M \cite{ying2016stochastic} & 0.4349 $\pm$ .0080 & 0.7361 $\pm$ .0084 & 0.6662 $\pm$ .0056 & 0.8517 $\pm$ .0048 & 0.6661 $\pm$ .0111 & 0.8598 $\pm$ .0097 \\
        MB \cite{kar2014online} & 0.4336 $\pm$ .0062 & 0.7348 $\pm$ .0069 & 0.6796 $\pm$ .0040 & 0.8651 $\pm$ .0032 & 0.6687 $\pm$ .0069 & 0.8624 $\pm$ .0055 \\
        SOPA \cite{zhu2022auc} & 0.4405 $\pm$ .0039 & 0.7417 $\pm$ .0053 & 0.6826 $\pm$ .0060 & 0.8681 $\pm$ .0045 & 0.6716 $\pm$ .0079 & 0.8650 $\pm$ .0062 \\
        SOPA-S \cite{zhu2022auc} & 0.4342 $\pm$ .0036 & 0.7354 $\pm$ .0056 & 0.6811 $\pm$ .0107 & 0.8666 $\pm$ .0091 & 0.6694 $\pm$ .0044 & 0.8628 $\pm$ .0038 \\
        AUC-poly \cite{yang2021all} & 0.4337 $\pm$ .0057 & 0.7349 $\pm$ .0077 & 0.6821 $\pm$ .0037 & 0.8676 $\pm$ .0021 & 0.6693 $\pm$ .0068 & 0.8627 $\pm$ .0054 \\
        AUC-exp \cite{yang2021all} & 0.4316 $\pm$ .0078 & 0.7328 $\pm$ .0098 & 0.6817 $\pm$ .0029 & 0.8672 $\pm$ .0015 & 0.6692 $\pm$ .0057 & 0.8626 $\pm$ .0043 \\
        \midrule
        PAUCI \cite{shao2022asymptotically} (Ours) & \secbest{0.4606 $\pm$ .0029} & \secbest{0.7618 $\pm$ .0033} & \secbest{0.7020 $\pm$ .0093} & \secbest{0.8875 $\pm$ .0080} & \secbest{0.6923 $\pm$ .0034} & \secbest{0.8860 $\pm$ .0029} \\
        UPAUCI (Ours) & \best{0.4927 $\pm$ .0035} & \best{0.7941 $\pm$ .0036} & \best{0.7480 $\pm$ .0034} & \best{0.8940 $\pm$ .0016} & \best{0.7165 $\pm$ .0024} & \best{0.9047 $\pm$ .0024} \\
    \bottomrule
    \end{tabular}
\end{table*}

\subsection{Datasets}
We adopt three imbalanced binary classification datasets: CIFAR-10-LT \cite{elson2007asirra}, CIFAR-100-LT \cite{krizhevsky2009learning} and Tiny-ImageNet-200-LT following the instructions in \cite{yang2021all}, where the binary datasets are constructed by selecting one super category as positive class and the other categories as negative class. Besides, we also construct a larger binary dataset based on iNaturalist2021\footnote{\url{https://github.com/visipedia/inat_comp/tree/master/2021}} to make a more realistic evaluation. Their statistics are detailed in Tab.\ref{tab:dataset}.

\begin{itemize}
    \item \textbf{Binary CIFAR-10-LT.} The CIFAR-10 dataset contains 60,000 images, each of 32x32 shapes, grouped into 10 classes of 6,000 images. The training and test sets contain 50,000 and 10,000 images, respectively. We construct the binary datasets by selecting one super category as positive class and the other categories as negative class. We generate three binary subsets composed of positive categories, including 1) birds, 2) automobiles, and 3) cats.
    \item \textbf{Binary CIFAR-100-LT.} The original CIFAR-100 dataset has 100 classes, with each containing 600 images. These 100 classes could be divided into 20 superclasses. By selecting a superclass as the positive class, we create CIFAR-100-LT following the same process as CIFAR-10-LT. The positive superclasses are 1) fruits and vegetables, 2) insects, and 3) large omnivores and herbivores, respectively.
    \item \textbf{Binary Tiny-ImageNet-200-LT.} There are 100,000 256x256 colored pictures in the Tiny-ImageNet-200 dataset, divided into 200 categories, with 500 pictures per category. We also choose 3 positive superclasses to create binary subsets: 1) dogs, 2) birds, and 3) vehicles.
    \item \textbf{Binary iNaturalist2021.} The large-scale real-world dataset iNaturalist2021 consists of images of 10,000 species of plants and animals and is naturally long-tailed. We create a binary subset by taking the Fungi super category as the positive class, and the Insects as the negative.
\end{itemize}

All data are divided into training, validation and test sets with a proportion of 0.7 : 0.15 : 0.15. For the first three datasets, sample sizes decay exponentially in each class, and the ratio of sample sizes of the least frequent to the most frequent class is set to 0.01.

\subsection{Competitors}
We compare our algorithm with 8 competitive baselines: (1) the approximation algorithms of PAUC, denoted as \textbf{AUC-poly} \cite{yang2021all} (poly calibrated weighting function) and \textbf{AUC-exp} \cite{yang2021all} (exp weighting function); (2) the DRO formulations of PAUC, \textbf{SOPA} \cite{zhu2022auc} (exact estimator) and \textbf{SOPA-S} \cite{zhu2022auc} (soft estimator); (3) the large-scale OPAUC optimization method \textbf{AGD-SBCD} \cite{yao2022large}; (4) the na\"ive mini-batch version of empirical partial AUC optimization, denoted as \textbf{MB} \cite{kar2014online}; (5) the AUC minimax \cite{ying2016stochastic} optimization, denoted as \textbf{AUC-M}; (6) the binary \textbf{CE} loss. The proposed methods with and without asymptotically vanishing gap are denoted as \textbf{PAUCI} (NeurIPS version) and \textbf{UPAUCI} (journal version), respectively.

\begin{figure*}[t]
	\centering
    \subfloat[PAUCI Subset-1]{\includegraphics[width=0.16\linewidth]{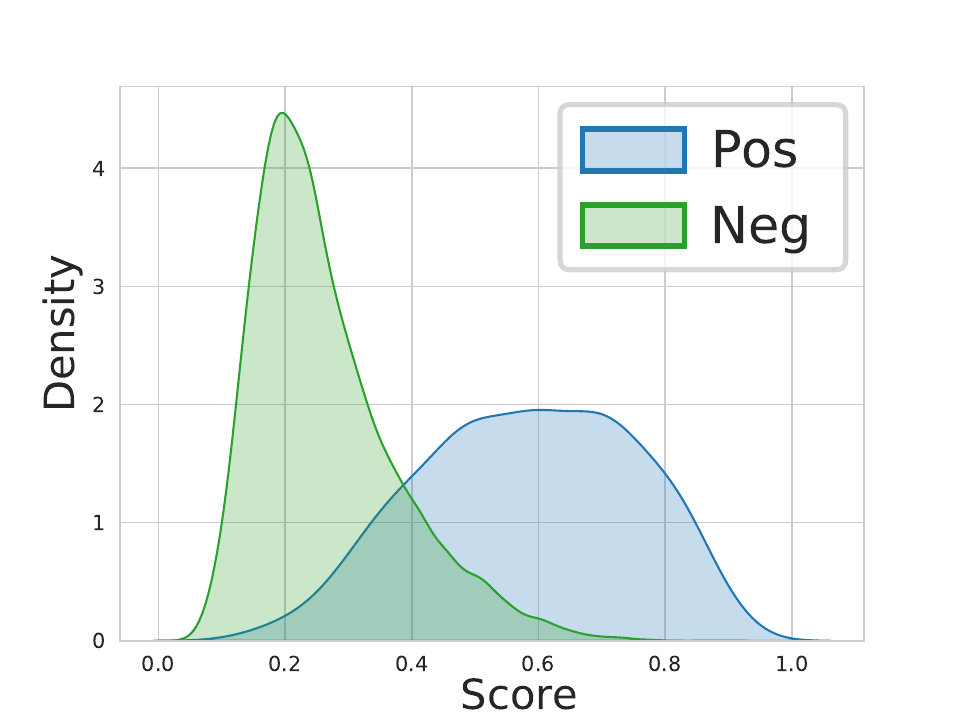}%
    }
    \subfloat[PAUCI Subset-2]{\includegraphics[width=0.16\linewidth]{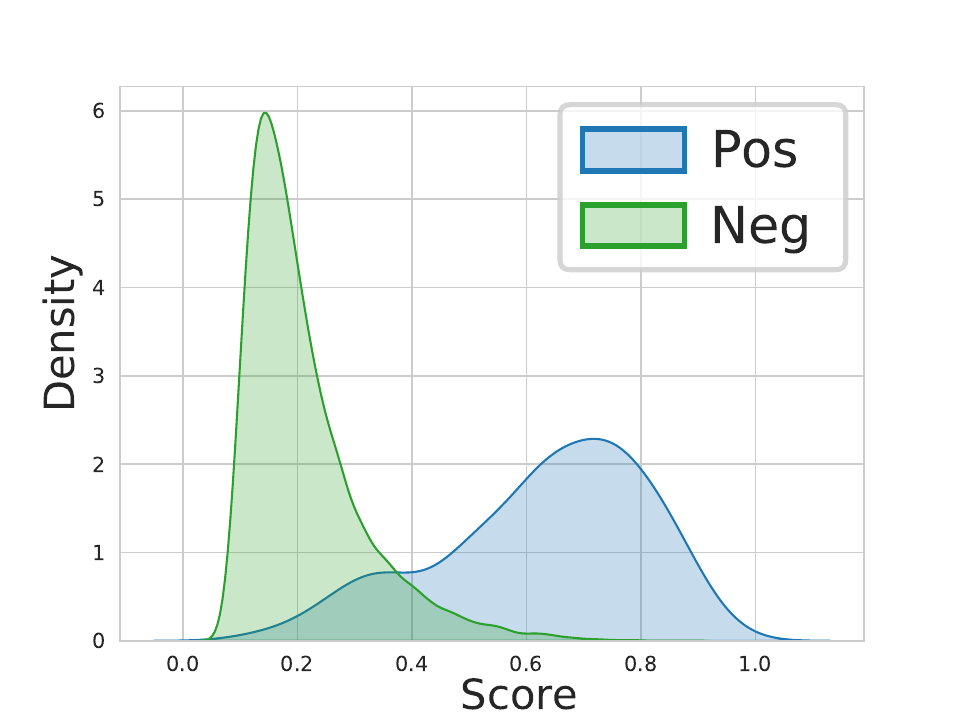}%
    }
    \subfloat[PAUCI Subset-3]{\includegraphics[width=0.16\linewidth]{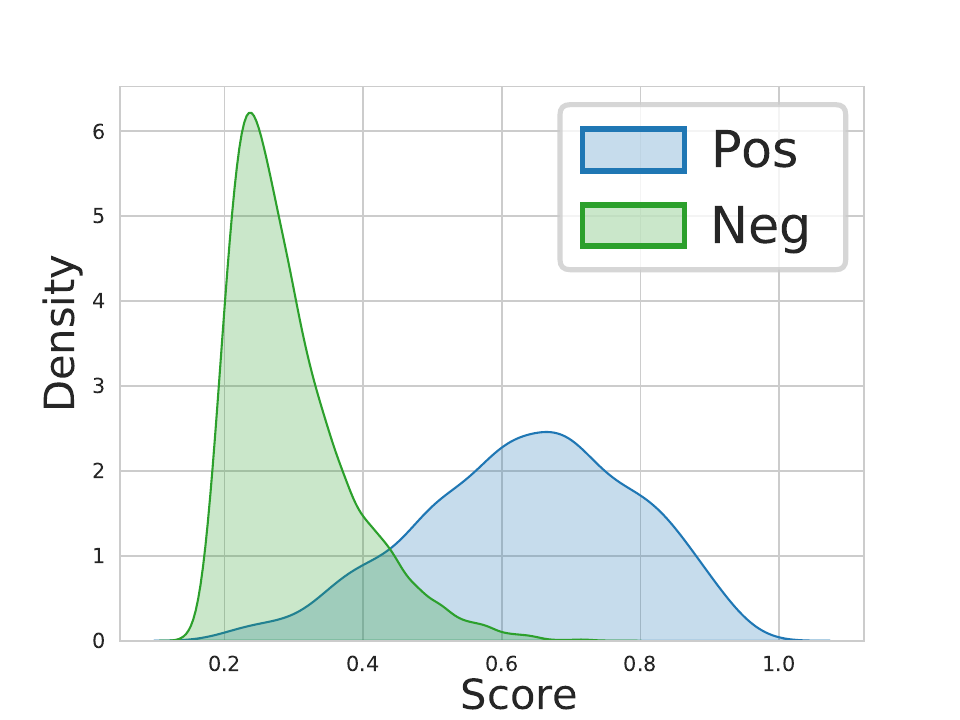}%
    }
	\subfloat[UPAUCI Subset-1]{\includegraphics[width=0.16\linewidth]{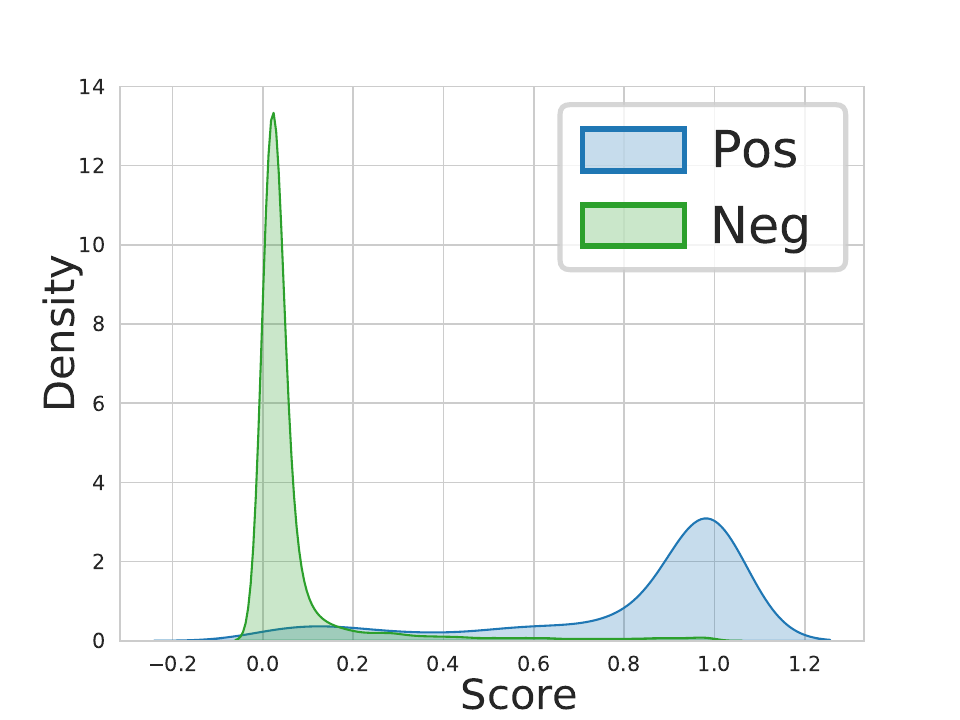}%
    }
    \subfloat[UPAUCI Subset-2]{\includegraphics[width=0.16\linewidth]{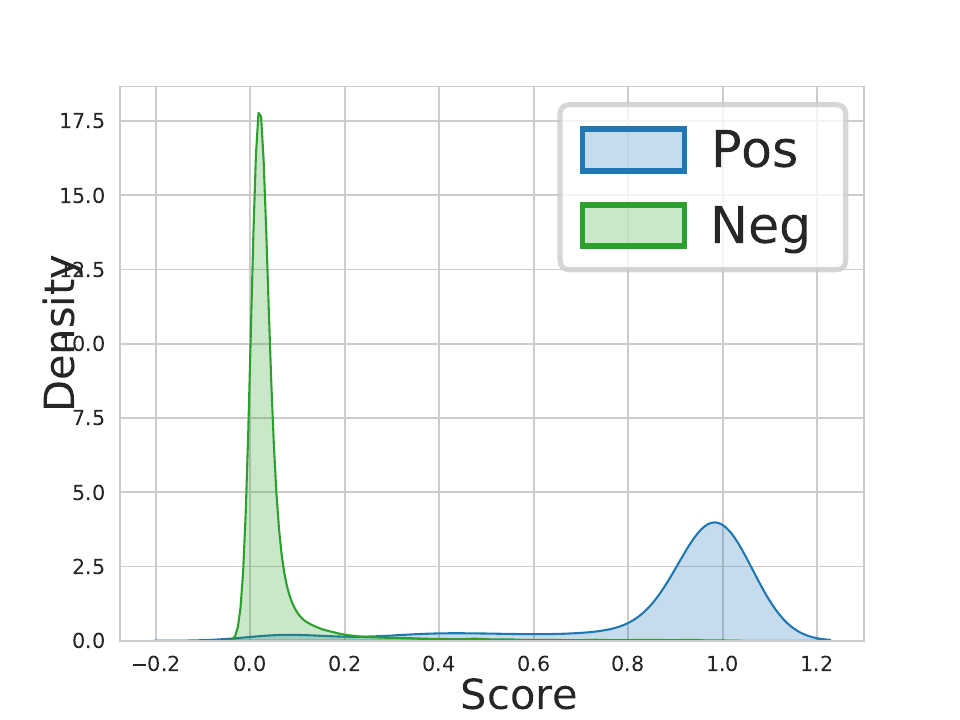}%
    }
    \subfloat[UPAUCI Subset-3]{\includegraphics[width=0.16\linewidth]{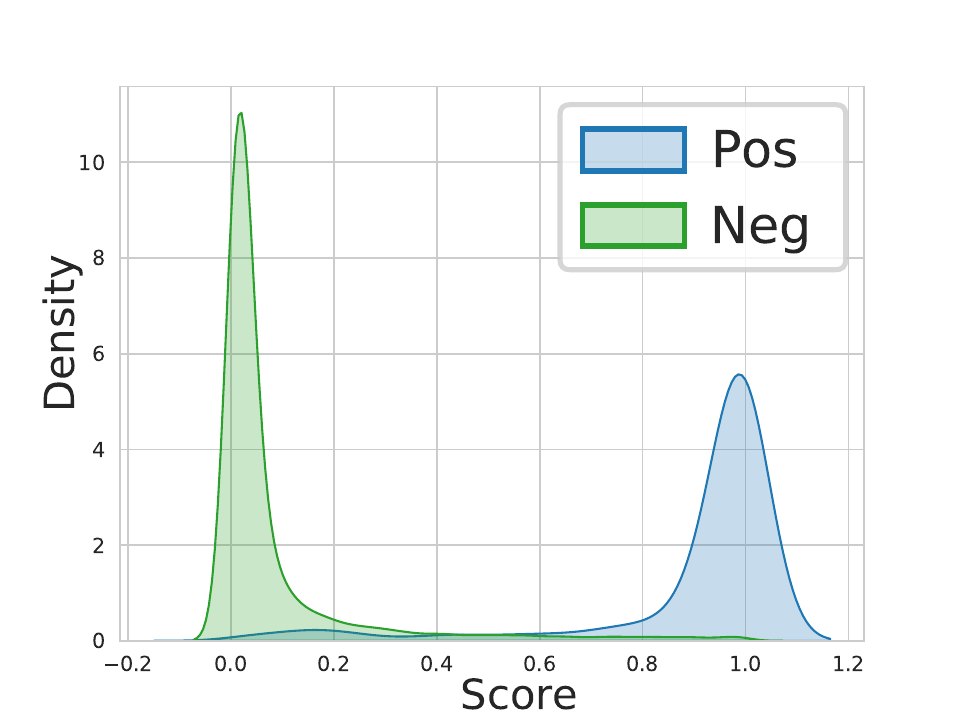}%
    }
	\caption{Score distribution over positive and negative samples by OPAUC optimization ($\mathrm{FPR}\leq 0.3$) on Tiny-ImageNet-200-LT.
    }
	\label{fig:score_op_tiny}
\end{figure*}

\begin{figure*}[t]
	\centering
    \subfloat[PAUCI Subset-1]{\includegraphics[width=0.16\linewidth]{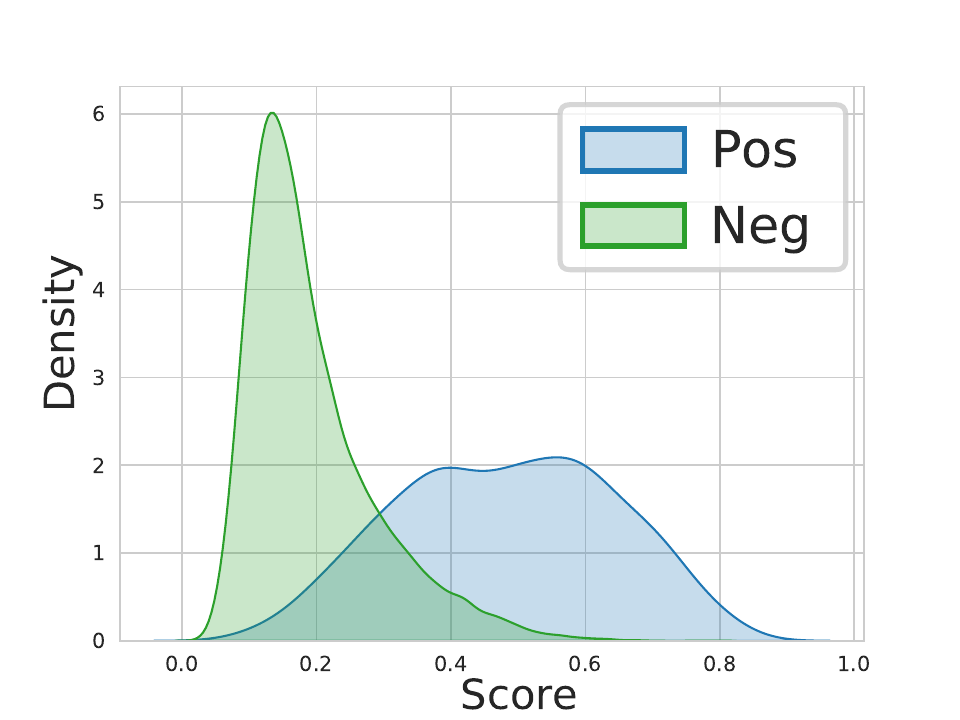}%
    }
    \subfloat[PAUCI Subset-2]{\includegraphics[width=0.16\linewidth]{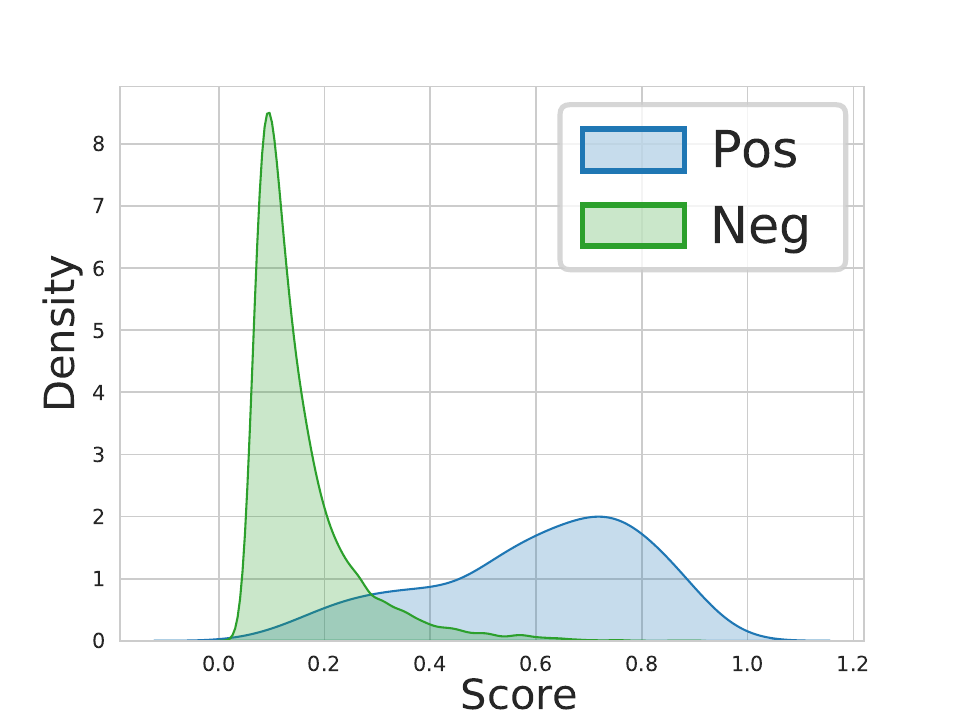}%
    }
    \subfloat[PAUCI Subset-3]{\includegraphics[width=0.16\linewidth]{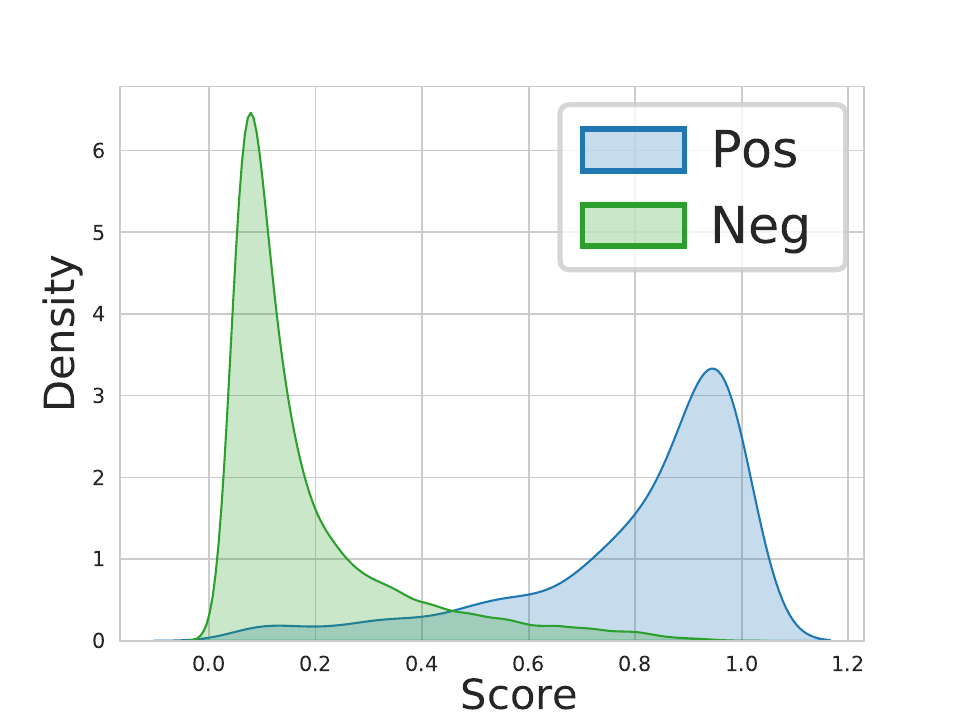}%
    }
	\subfloat[UPAUCI Subset-1]{\includegraphics[width=0.16\linewidth]{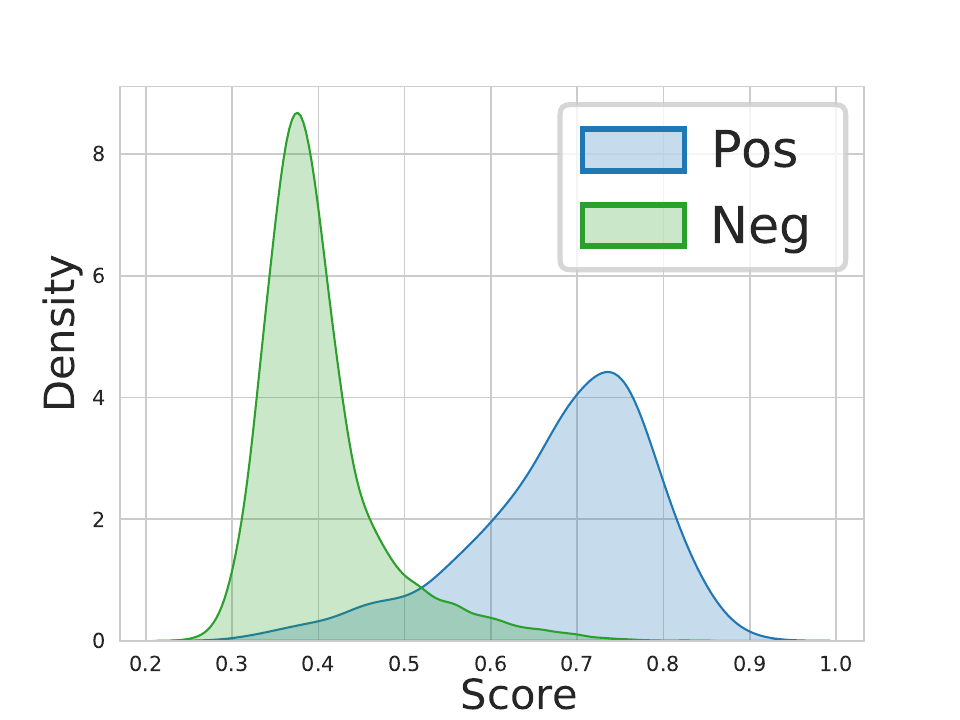}%
    }
    \subfloat[UPAUCI Subset-2]{\includegraphics[width=0.16\linewidth]{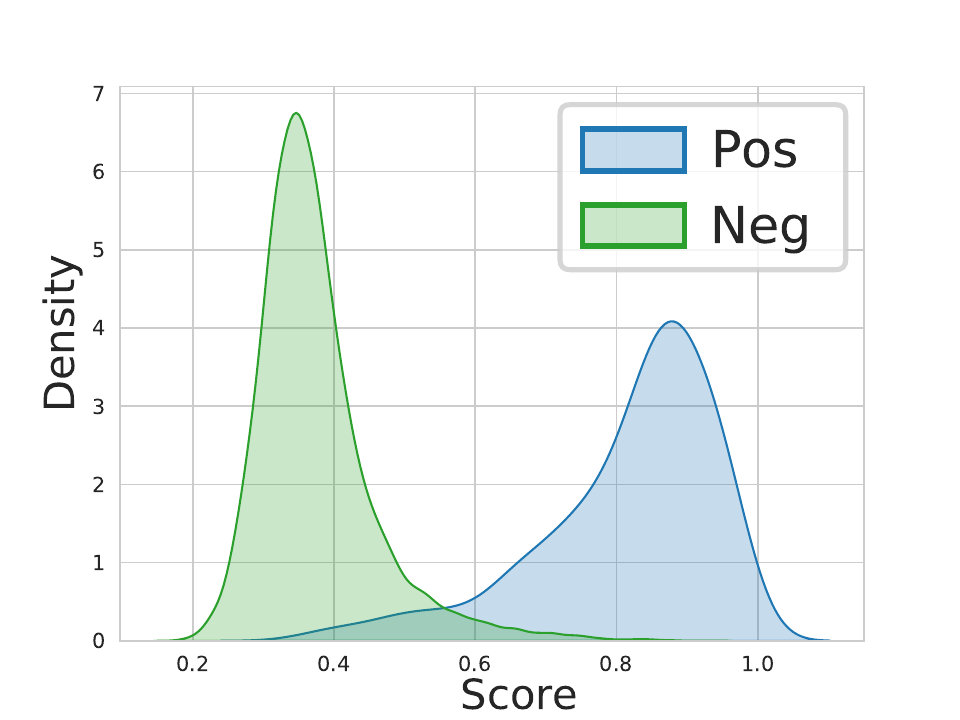}%
    }
    \subfloat[UPAUCI Subset-3]{\includegraphics[width=0.16\linewidth]{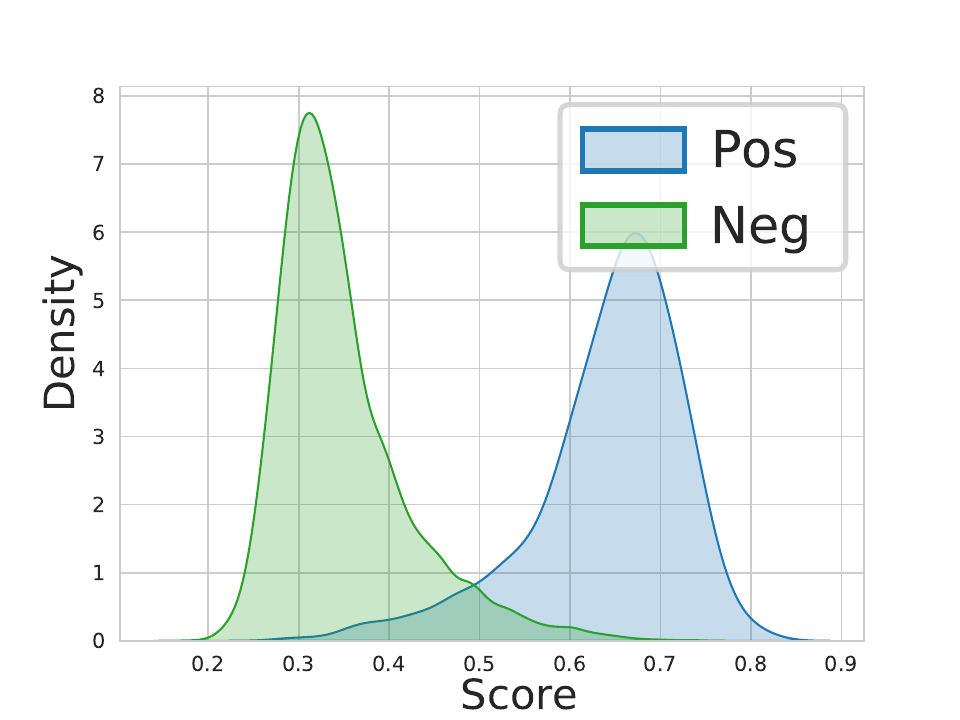}%
    }
	\caption{Score distribution over positive and negative samples by TPAUC optimization ($\mathrm{TPR} \geq 0.3$, $\mathrm{FPR}\leq 0.3$) on Tiny-ImageNet-200-LT.}
	\label{fig:score_tp_tiny}
\end{figure*}

\subsection{Implementation Details}
All experiments are conducted on an Ubuntu 16.04.1 server equipped with an Intel(R) Xeon(R) Silver 4110 CPU and four RTX 3090 GPUs, and all codes are developed in \texttt{Python 3.8} and \texttt{pytorch 1.8.2} environment. We use the ResNet-18 as a backbone. With a \texttt{Sigmoid} function, the output is scaled into $[0,1]$. The batch size is set as $1024$. Following the previous studies \cite{yang2021all, yuan2021large, guo2020communication}, we warm up all algorithms for $10$ epochs with CE loss to avoid overfitting. All competitors are trained using \texttt{SGD} as the basic optimizer. The evaluation metrics in experiments are $\widehat{\mathrm{AUC}}_{\beta}$ and $\widehat{\mathrm{AUC}}_{\alpha, \beta}$ with different $\alpha$ and $\beta$. The experiments are repeated ten times randomly and we report the mean and std values. Please see Appx.\ref{section:experiment_details} for more details of parameter tuning.

\begin{table}[t]
    \centering
    \caption{
        OPAUC on iNaturalist2021 with different $\beta$.
    }
    \label{tab:op_inat}
    \begin{tabular}{l|cc}
        \toprule
        \multirow{2}*{Method} & \multicolumn{2}{c}{iNaturalist2021} \\
        \cmidrule{2-3}~ & 0.3 & 0.5 \\
        \midrule
        CE & 0.9204 $\pm$ .0025 & 0.9358 $\pm$ .0098 \\
        AUC-M \cite{ying2016stochastic} & 0.9283 $\pm$ .0029 & 0.9437 $\pm$ .0071 \\
        MB \cite{kar2014online} & \secbest{0.9374 $\pm$ .0041} & \secbest{0.9527 $\pm$ .0055} \\
        SOPA \cite{zhu2022auc} & 0.9338 $\pm$ .0021 & 0.9492 $\pm$ .0091 \\
        SOPA-S \cite{zhu2022auc} & 0.9361 $\pm$ .0046 & 0.9515 $\pm$ .0073 \\
        AGD-SBCD \cite{yao2022large} & 0.9316 $\pm$ .0018 & 0.9470 $\pm$ .0058 \\
        AUC-poly \cite{yang2021all} & 0.9366 $\pm$ .0052 & 0.9519 $\pm$ .0069 \\
        AUC-exp \cite{yang2021all} & 0.9308 $\pm$ .0039 & 0.9462 $\pm$ .0082 \\
        \midrule
        PAUCI \cite{shao2022asymptotically} (Ours) & 0.9348 $\pm$ .0063 & 0.9501 $\pm$ .0044 \\
        UPAUCI (Ours) & \best{0.9587 $\pm$ .0048} & \best{0.9741 $\pm$ .0061} \\
        \bottomrule
    \end{tabular}
\end{table}

\subsection{Overall Results}

In Tab.\ref{tab:op_cifar10}--\ref{tab:tp_tinyimagenet}, we record the performance on test sets of all the methods on subsets of CIFAR-10-LT, CIFAR-100-LT, and Tiny-Imagent-200-LT. Besides, the results on iNaturalist2021 are recorded in Tab.\ref{tab:op_inat} and \ref{tab:tp_inat}. Each method is tuned independently for $\op$ and $\tp$ metrics. From the results, we make the following remarks:
\begin{enumerate}
    \item[(1)] Our proposed methods outperform all the baselines in most cases for $\op$ and $\tp$. Even for the failure cases (\eg, $\op$ with FPR$\le 0.5$ on CIFAR-100-LT-1, $\tp$ with TPR$\ge 0.3$ and FPR$\le0.3$ on CIFAR-10-LT-2), our methods could attain fairly competitive results compared with the best competitors.
    \item[(2)] We can see that the normal AUC optimization method AUC-M has less reasonable performance under PAUC metric. This demonstrates the necessity of developing the PAUC optimization algorithm to focus on the partial area.
    \item[(3)] Approximation methods SOPA-S, AGD-SBCD, AUC-poly and AUC-exp have lower performance than the unbiased algorithm SOPA and our instance-wise algorithm PAUCI/UPAUCI in most cases. Therefore, optimization based on the unbiased approximation of PAUC is a better choice than biased ones.
    \item[(4)] The proposed UPAUCI achieves higher performance than PAUCI in all but one setting, implying that an unbiased formulation may be more important for optimization. Although the approximation gap in PAUCI vanishes asymptotically, such a bias will inevitably induce biased parameters in practice due to the limited value of $\kappa$. On relatively simple benchmarks such as CIFARs, UPAUCI exhibits more consistent advantages over PAUCI under the $\tp$ metric. This suggests that the approximation bias will be amplified in the presence of both TPR and FPR constraints. Moreover, this issue is particularly evident on Tiny-ImageNet-LT and iNaturalist2021, which exhibit more severe class imbalance or fine-grained categories with small inter-class margins. In such extreme settings, the approximation bias of PAUCI is strongly amplified, whereas UPAUCI's unbiased reformulation remains faithful to the true partial AUC objective, leading to significantly larger gains.
\end{enumerate}
Above all, the experimental results show the effectiveness of our proposed methods.

\begin{table}[t]
    \centering
    \caption{
        TPAUC on iNaturalist2021 with different $(\alpha, \beta)$.
    }
    \label{tab:tp_inat}
    \begin{tabular}{l|cc}
        \toprule
        \multirow{2}*{Method} & \multicolumn{2}{c}{iNaturalist2021} \\
        \cmidrule{2-3}~ & (0.3, 0.3) & (0.5, 0.5) \\
        \midrule
        CE & 0.8075 $\pm$ .0054 & 0.9077 $\pm$ .0153 \\
        AUC-M \cite{ying2016stochastic} & 0.8154 $\pm$ .0058 & 0.9136 $\pm$ .0127 \\
        MB \cite{kar2014online} & \secbest{0.8245 $\pm$ .0070} & \secbest{0.9226 $\pm$ .0111} \\
        SOPA \cite{zhu2022auc} & 0.8209 $\pm$ .0037 & 0.9181 $\pm$ .0147 \\
        SOPA-S \cite{zhu2022auc} & 0.8232 $\pm$ .0075 & 0.9204 $\pm$ .0124 \\
        AUC-poly \cite{yang2021all} & 0.8237 $\pm$ .0081 & 0.9218 $\pm$ .0115 \\
        AUC-exp \cite{yang2021all} & 0.8179 $\pm$ .0068 & 0.9161 $\pm$ .0133 \\
        \midrule
        PAUCI \cite{shao2022asymptotically} (Ours) & 0.8219 $\pm$ .0092 & 0.9195 $\pm$ .0099 \\
        UPAUCI (Ours) & \best{0.8458 $\pm$ .0029} & \best{0.9431 $\pm$ .0056} \\
        \bottomrule
    \end{tabular}
\end{table}

\begin{figure*}[t]
	\centering
	\subfloat[CIFAR-10-LT-1]{\includegraphics[width=0.32\linewidth]{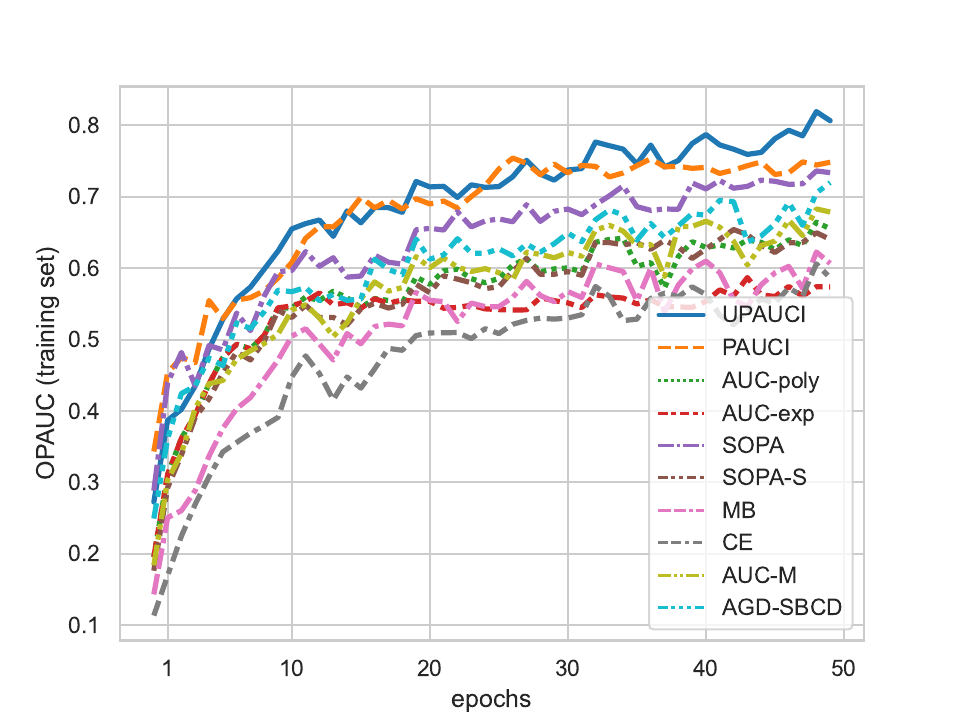}%
  }
  \subfloat[CIFAR-10-LT-2]{\includegraphics[width=0.32\linewidth]{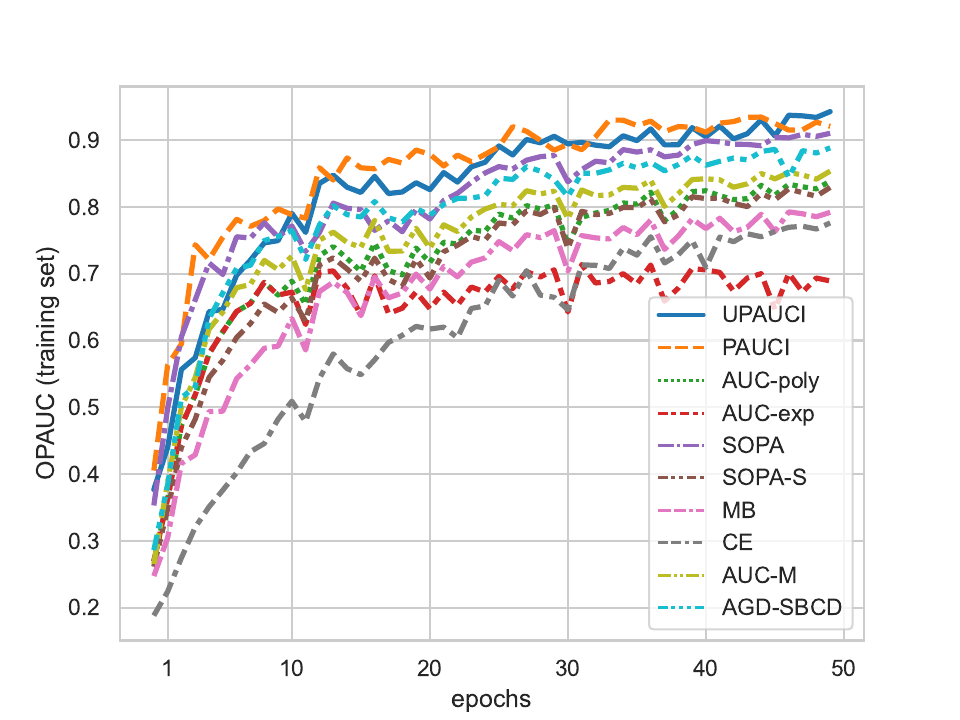}%
  }
  \subfloat[CIFAR-10-LT-3]{\includegraphics[width=0.32\linewidth]{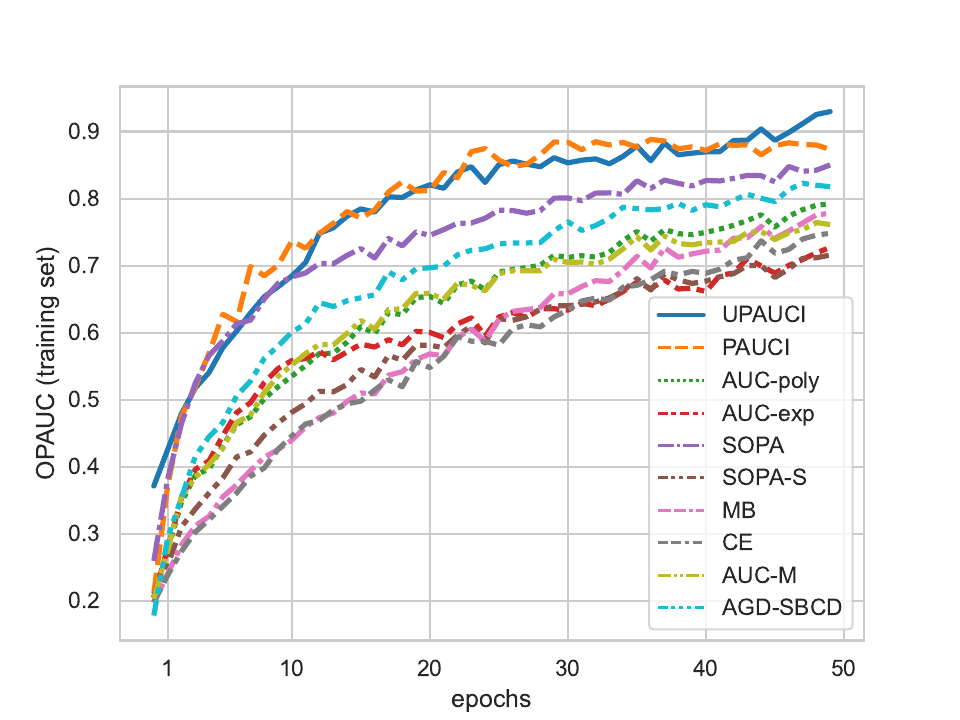}%
  }
	\caption{Convergence of OPAUC optimization($\mathrm{FPR}\leq 0.3$) on CIFAR-10-LT.}
	\label{fig:opaucconvergence}
\end{figure*}
\begin{figure*}[t]
	\centering
	\subfloat[CIFAR-10-LT-1]{\includegraphics[width=0.32\linewidth]{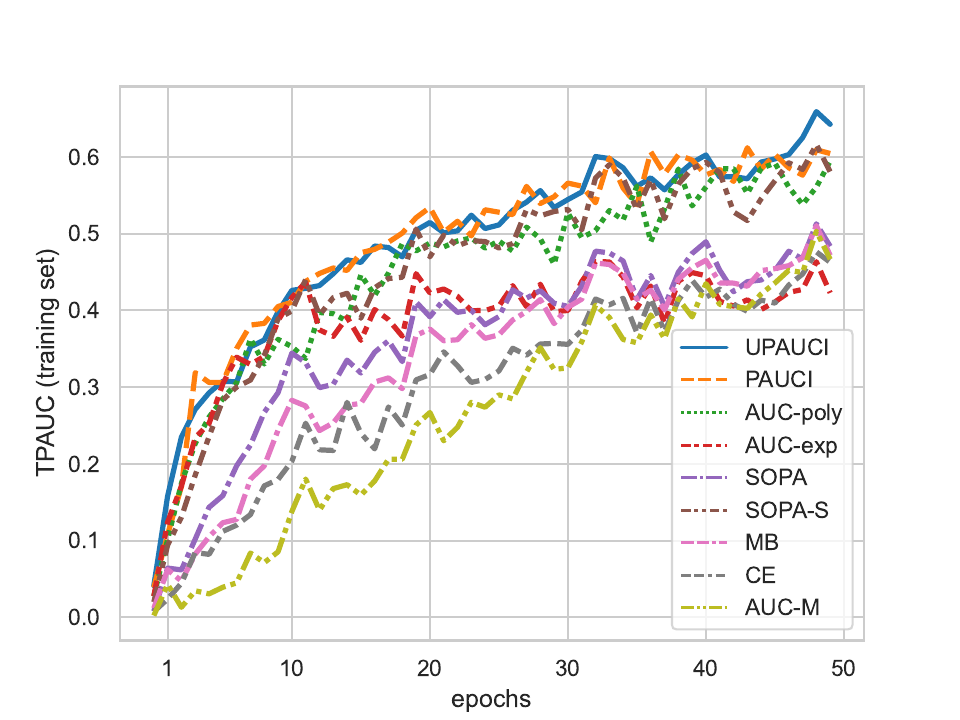}%
  }
  \subfloat[CIFAR-10-LT-2]{\includegraphics[width=0.32\linewidth]{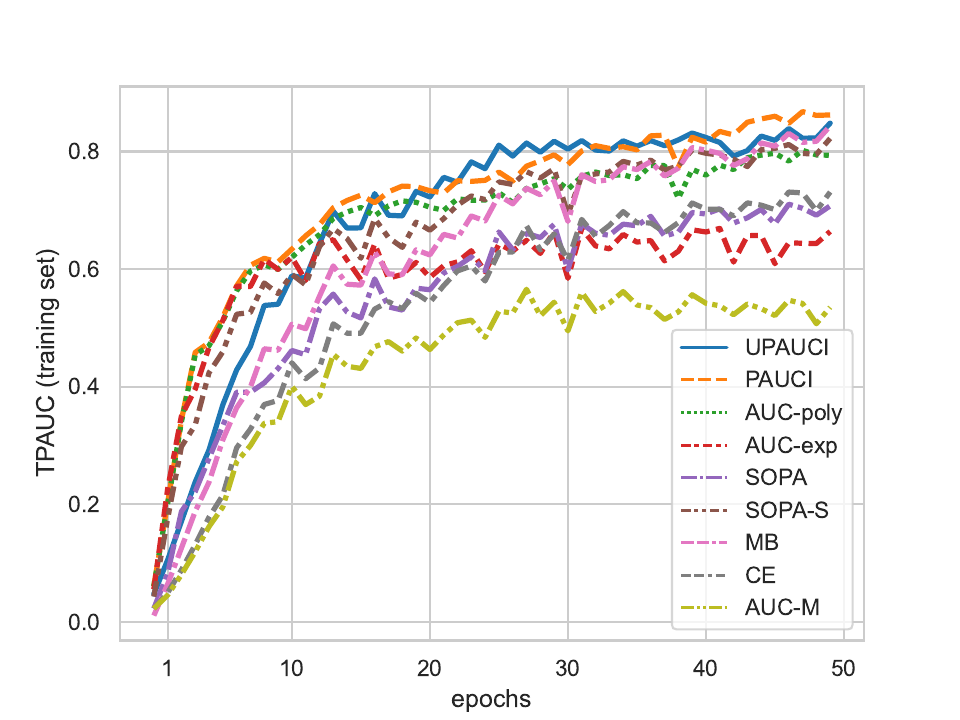}%
  }
  \subfloat[CIFAR-10-LT-3]{\includegraphics[width=0.32\linewidth]{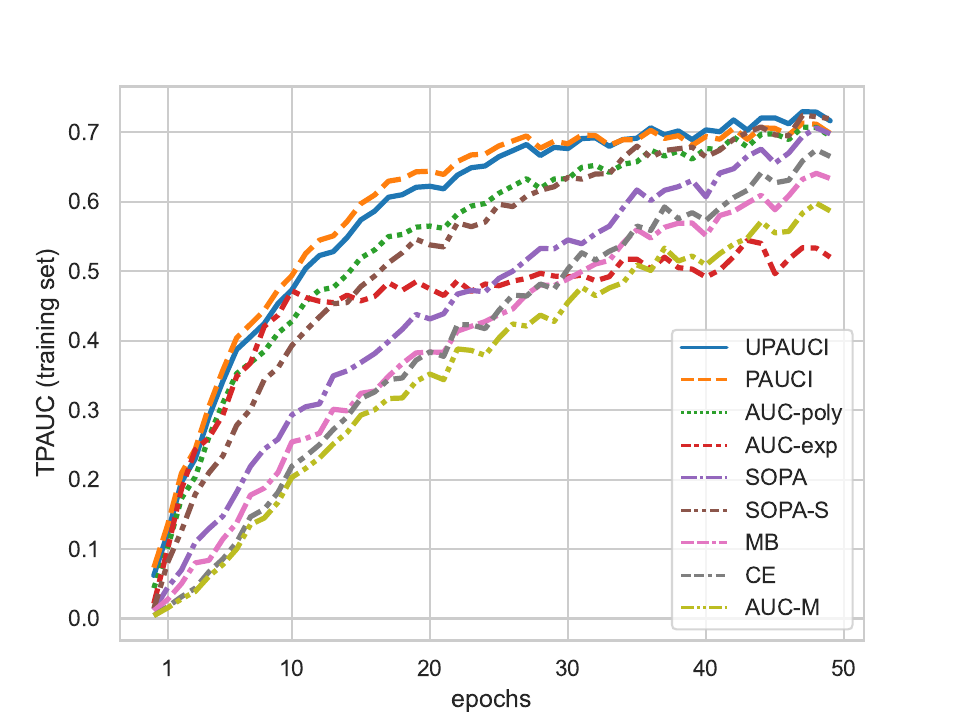}%
  }
	\caption{Convergence of TPAUC ($\mathrm{TPR} \geq 0.3$, $\mathrm{FPR}\leq 0.3$) optimization on CIFAR-10-LT.}
	\label{fig:tpaucconvergence}
\end{figure*}

\subsection{Score distribution of PAUCI and UPAUCI}
For a fine-grained comparison, we visualize the score distributions for positive and negative instances obtained by PAUCI and UPAUCI. Taking the results on Tiny-ImageNet-200-LT as an example. According to Tab.\ref{tab:op_tinyimagenet} and Tab.\ref{tab:tp_tinyimagenet}, the absolute performance gain of UPAUCI is obvious (up to $0.034$ for $\op$ and $0.046$ for $\tp$). Such a gain is also reflected in the learned score distributions in Fig.\ref{fig:score_op_tiny} and Fig.\ref{fig:score_tp_tiny}. By comparing the results of PAUCI and UPAUCI on three subsets, \ie, comparing (a)/(b)/(c) with (d)/(e)/(f), we can observe (1) both the positive and negative score distributions are much sharper in UPAUCI, and (2) UPAUCI is able to significantly reduce the overlap between positive and negative distributions. Subsequently, the partial area under high FPR and low TPR constraints is better optimized.

\begin{figure*}[!t]
	\centering
    \subfloat[$a$]{\includegraphics[width=0.33\linewidth]{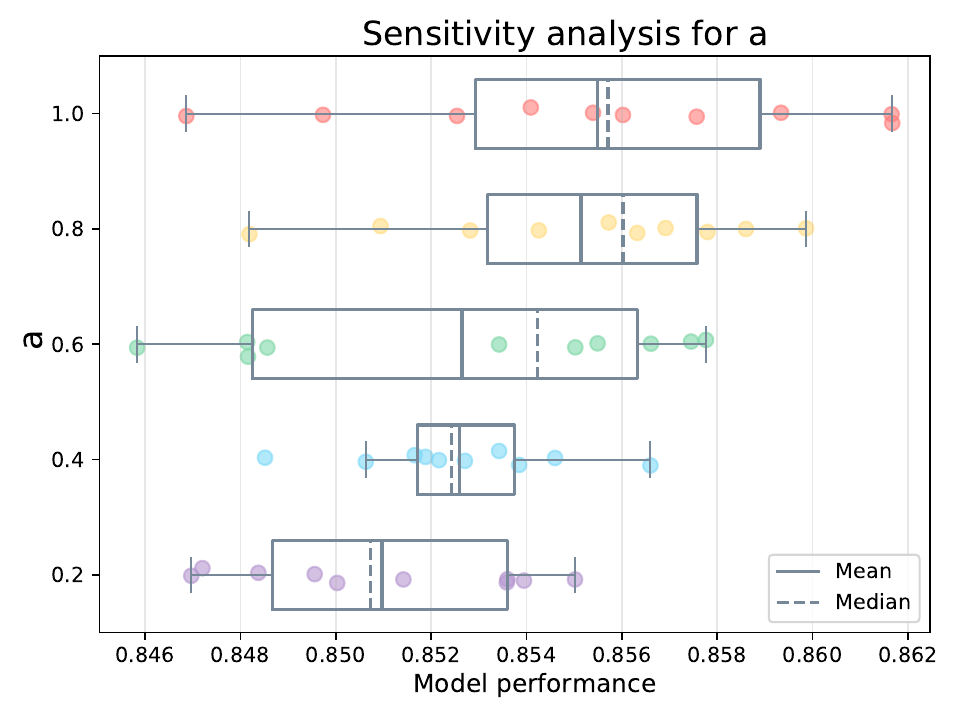}%
    }
    \subfloat[$b$]{\includegraphics[width=0.33\linewidth]{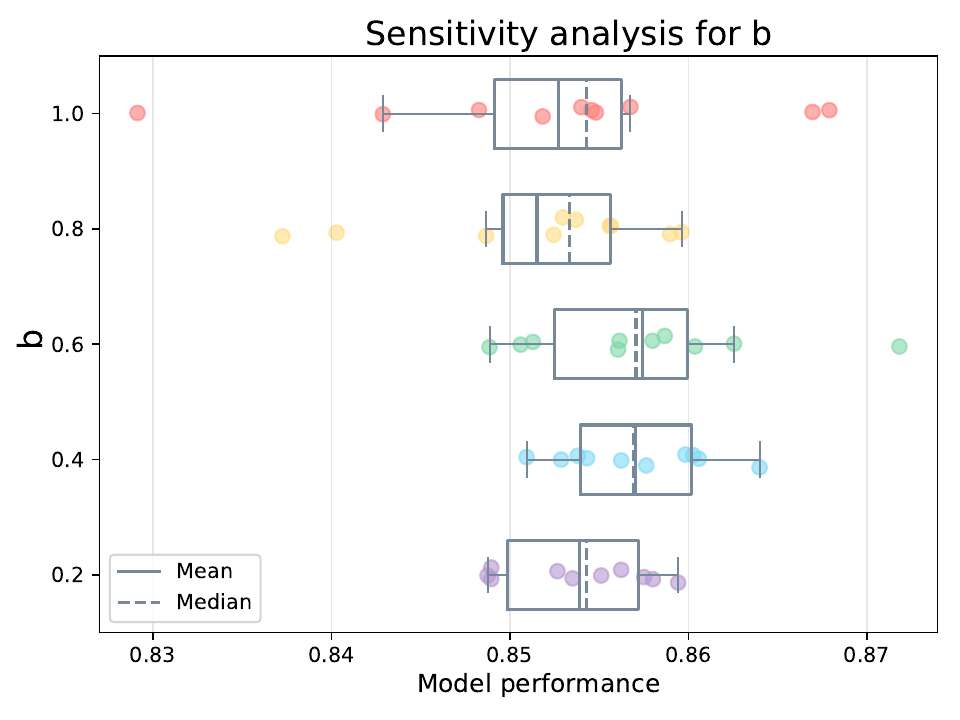}%
    }
    \subfloat[$\gamma$]{\includegraphics[width=0.33\linewidth]{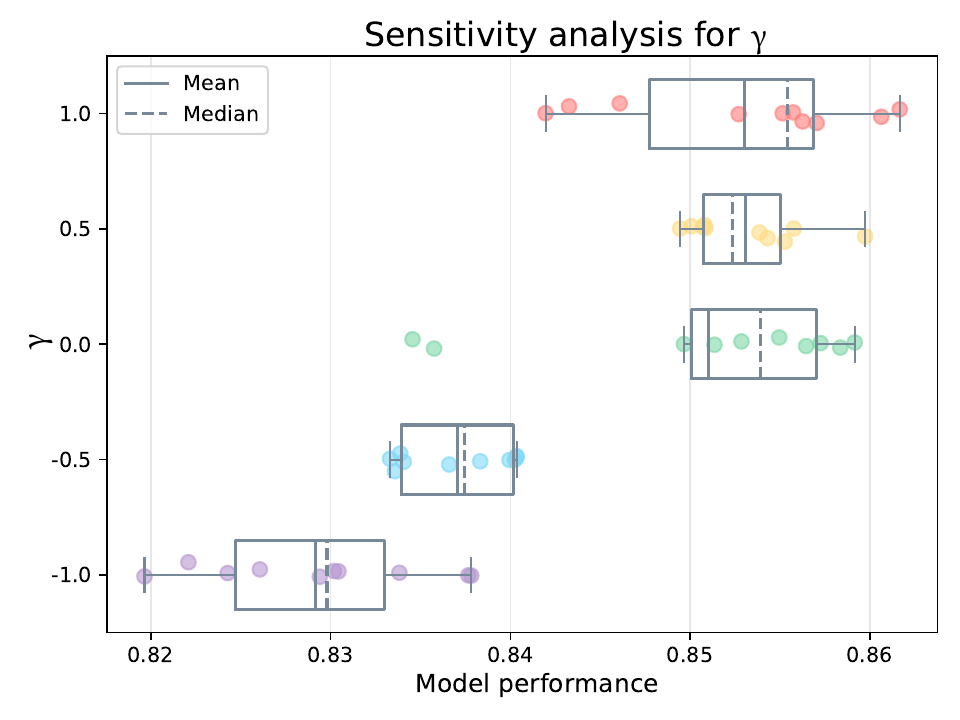}%
    }
    \\
    \subfloat[$s'$]{\includegraphics[width=0.33\linewidth]{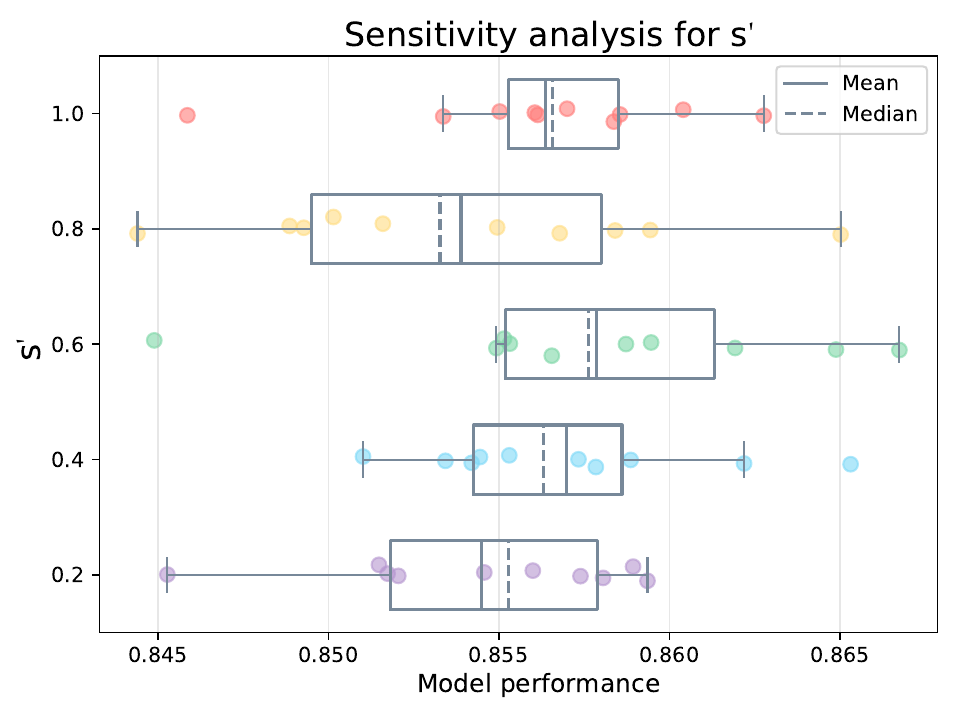}%
    }
    \subfloat[$c$]{\includegraphics[width=0.33\linewidth]{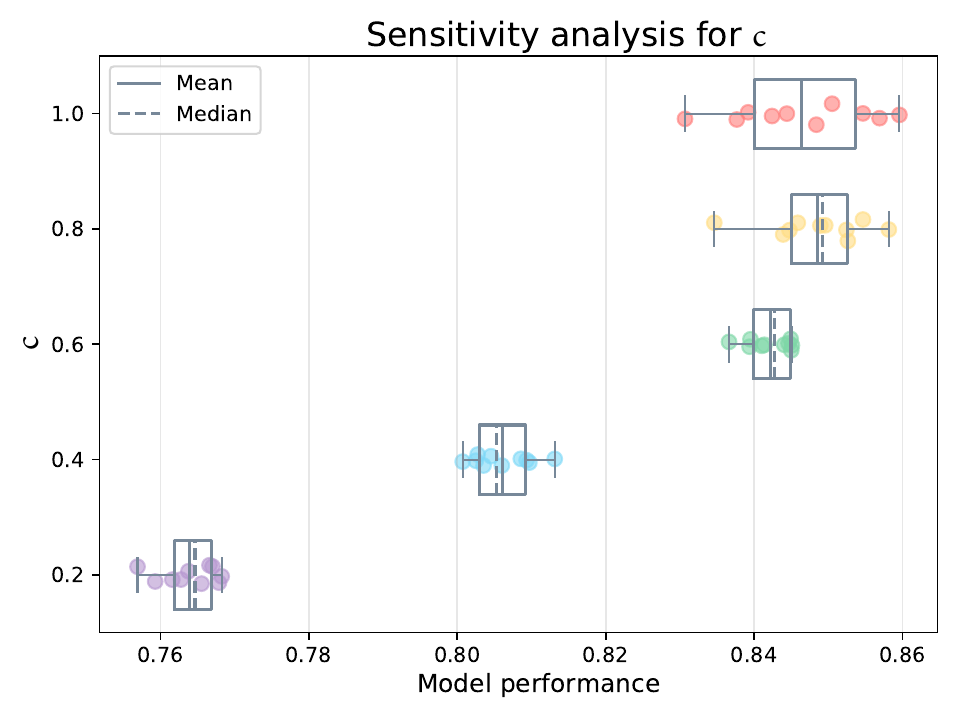}%
    }
	\caption{Sensitivity towards the variable initialization of UPAUCI over OPAUC ($\mathrm{FPR}\leq 0.5$) on CIFAR-10-LT-1.}
	\label{fig:sensi_init_op_0.5_cifar10-1}
\end{figure*}

\begin{figure*}[!t]
	\centering
    \subfloat[$a$]{\includegraphics[width=0.33\linewidth]{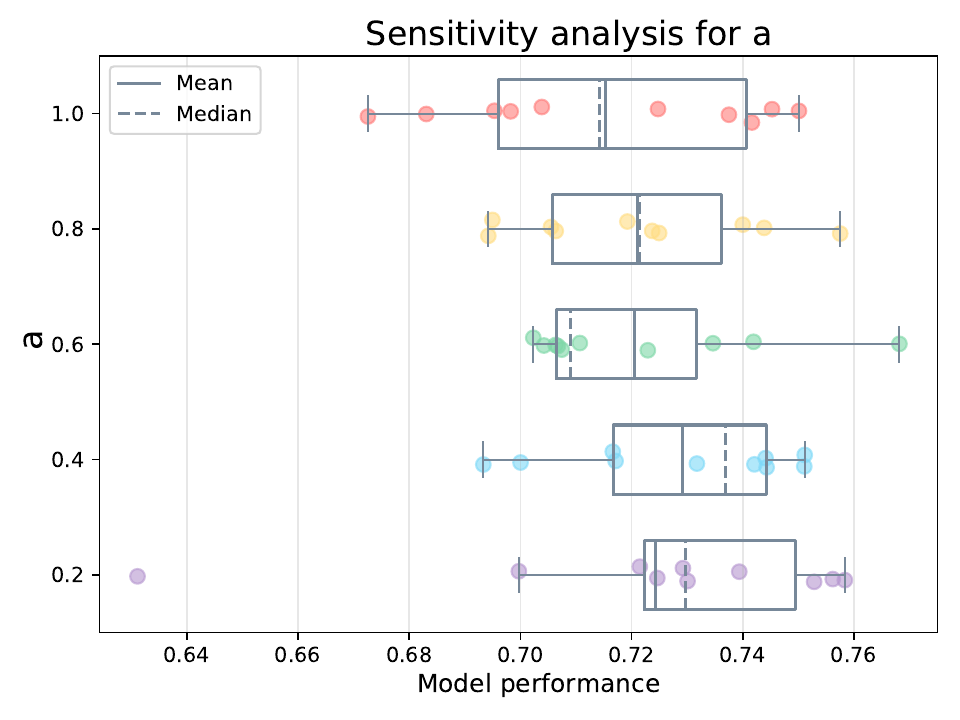}%
    }
    \subfloat[$b$]{\includegraphics[width=0.33\linewidth]{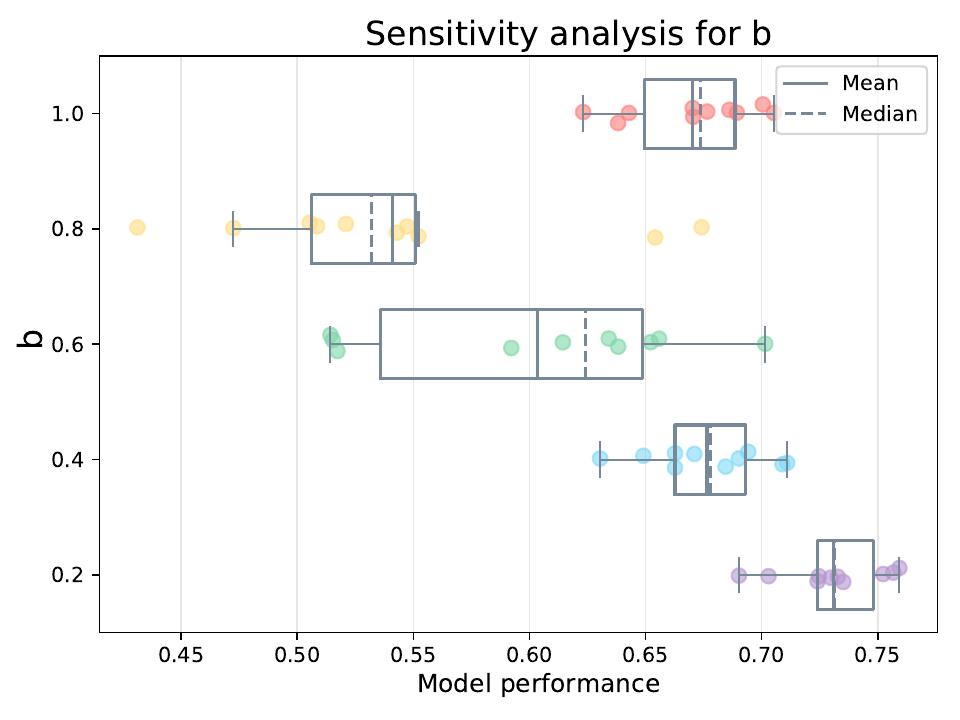}%
    }
    \subfloat[$\gamma$]{\includegraphics[width=0.33\linewidth]{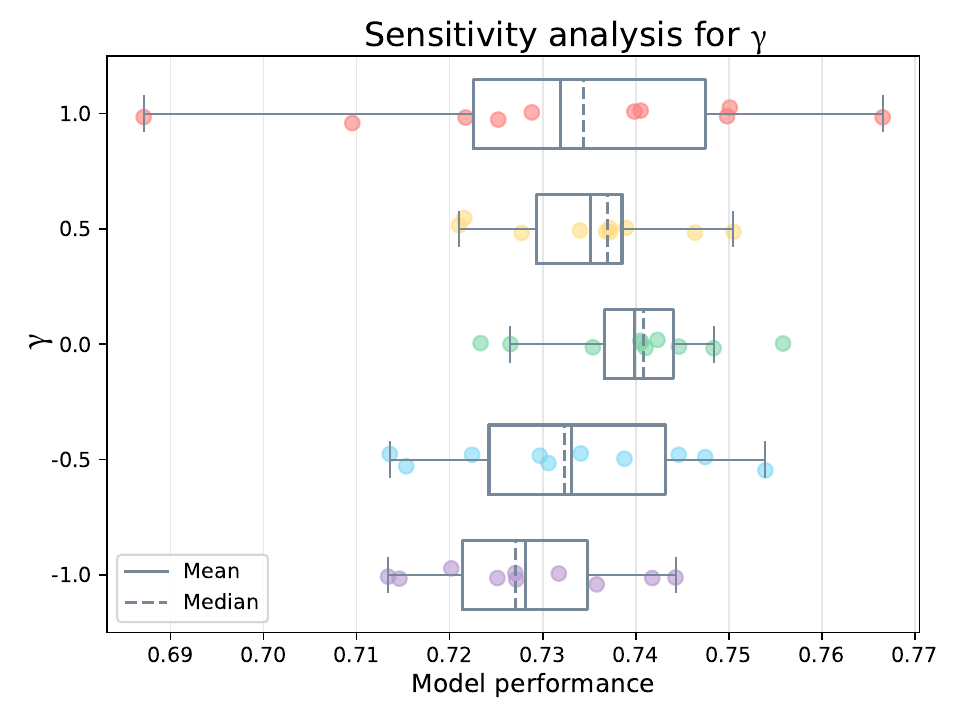}%
    }
    \\
    \subfloat[$s$]{\includegraphics[width=0.33\linewidth]{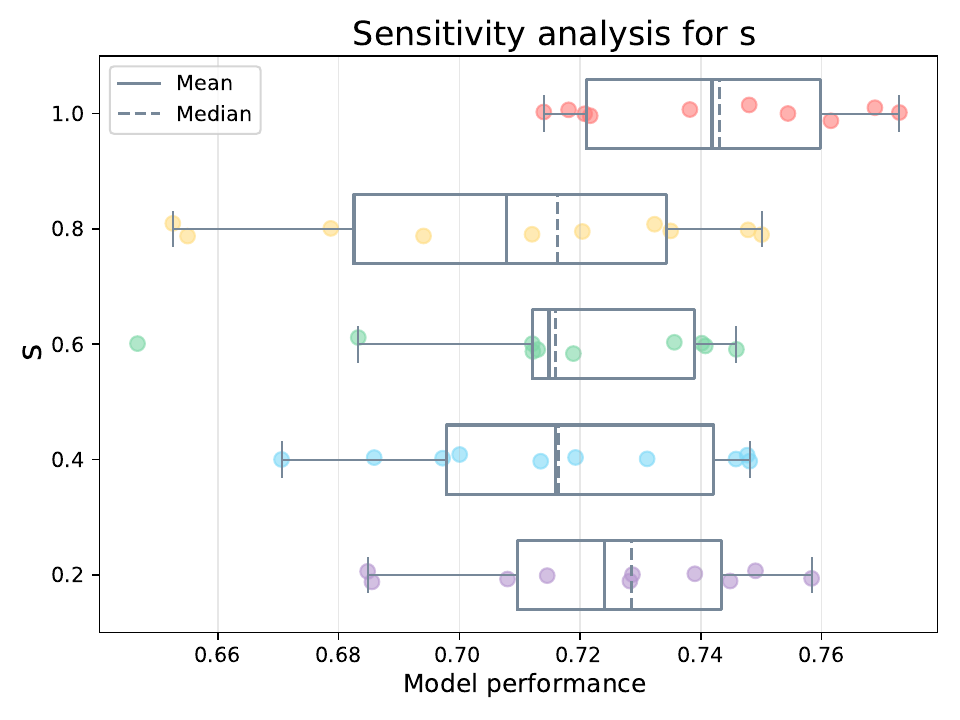}%
    }
    \subfloat[$s'$]{\includegraphics[width=0.33\linewidth]{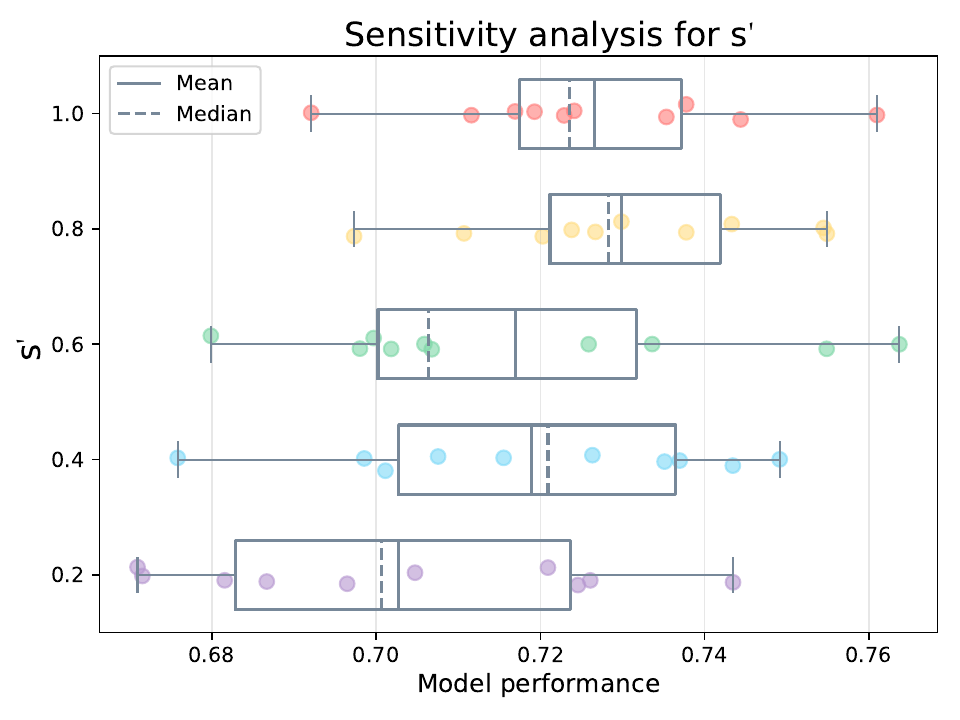}%
    }
    \subfloat[$c$]{\includegraphics[width=0.33\linewidth]{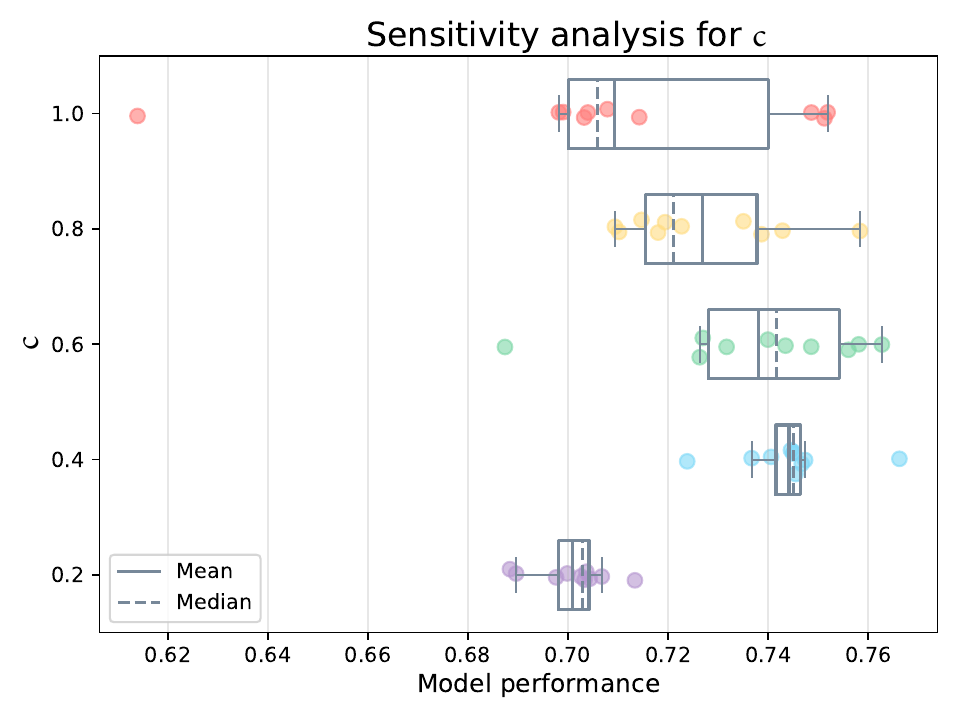}%
    }
	\caption{Sensitivity towards the variable initialization of UPAUCI over TPAUC ($\mathrm{TPR} \geq 0.5$, $\mathrm{FPR}\leq 0.5$) on CIFAR-10-LT-1.}
	\label{fig:sensi_init_tp_0.5_cifar10-1}
\end{figure*}

\begin{figure*}[!t]
	\centering
    \subfloat[OPAUC ($\mathrm{FPR}\leq 0.5$)]{\includegraphics[width=0.33\linewidth]{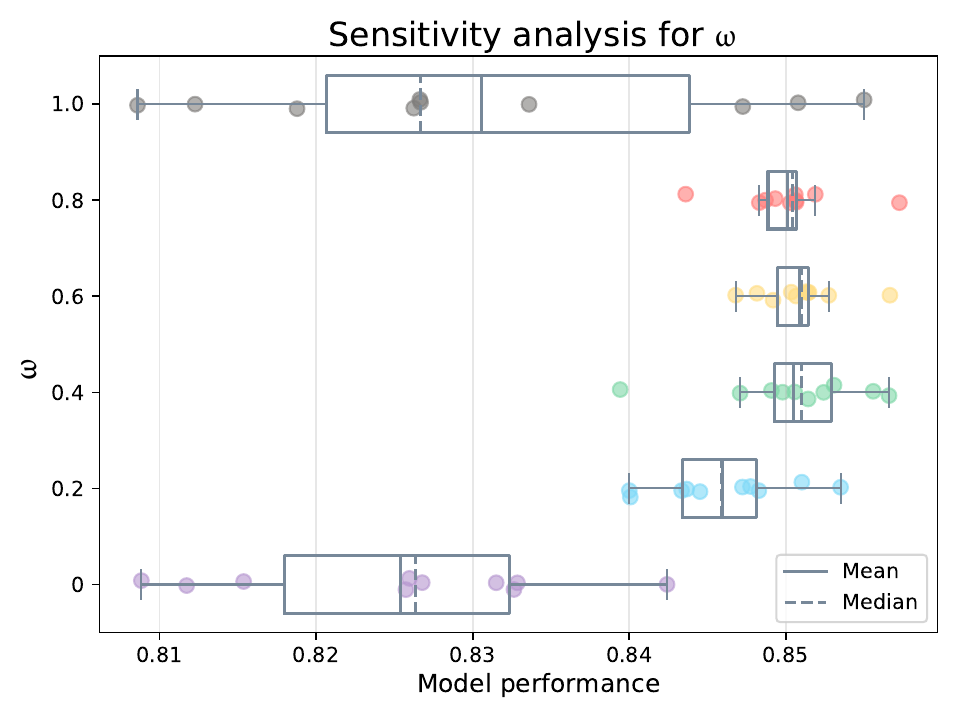}%
    }
    \subfloat[TPAUC ($\mathrm{TPR} \geq 0.3$, $\mathrm{FPR}\leq 0.3$)]{\includegraphics[width=0.33\linewidth]{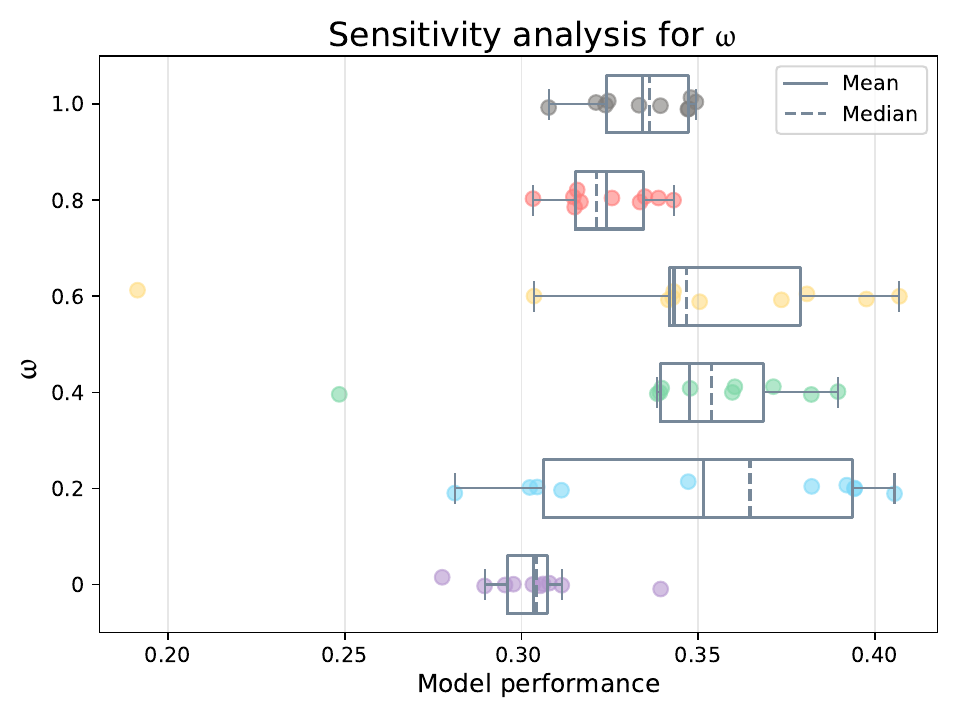}%
    }
    \subfloat[TPAUC ($\mathrm{TPR} \geq 0.5$, $\mathrm{FPR}\leq 0.5$)]{\includegraphics[width=0.33\linewidth]{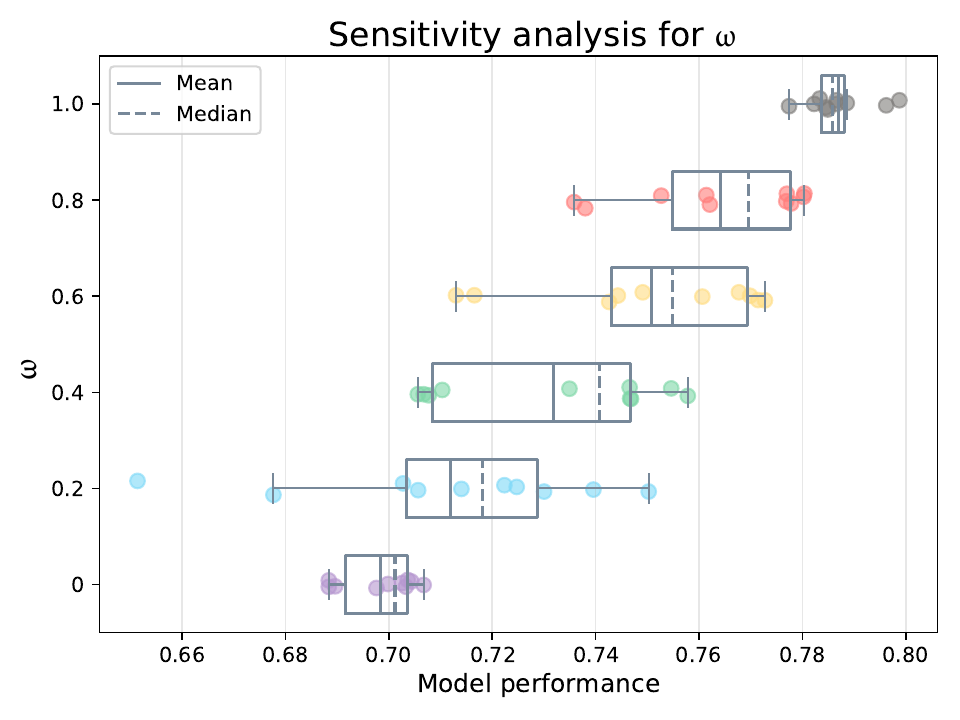}%
    }
	\caption{Sensitivity towards $\omega$ of UPAUCI on CIFAR-10-LT-1.}
	\label{fig:sensi_omega_0.5_cifar10-1}
\end{figure*}

\subsection{Convergence Analysis}
In the convergence analysis, for the sake of fairness, we did not use warm-up. All algorithms use hyperparameters in the performance experiments. We show the plots of training convergence in Fig.\ref{fig:opaucconvergence} and Fig.\ref{fig:tpaucconvergence} on CIFAR-10 for both $\op$ and $\tp$.
According to the figures, we can make the following observations: (1) Our algorithms and SOPA converge faster than other methods for $\op$. However, for $\tp$ optimization, the SOPA converges very slowly due to its complicated algorithm, while our method still shows the best convergence property in most cases. (2) It's notable that our algorithms converge to stabilize after twenty epochs in most cases. That means our methods have better stability in practice.



\subsection{Sensitivity Analysis}
\textbf{Sensitivity towards Variable Initialization.} Except for the network parameters $\bm{\theta}$, the major learnable variables in UPAUCI are (1) $a, b, \gamma$ introduced in the instance-wise reformulation, and (2) $s, s', \bm{c}$ introduced in the differentiable sample selection and the smoothing steps. To facilitate the practical implementation, we need to study how to initialize these values through the sensitivity analysis towards their initialization. Results over OPAUC ($\mathrm{FPR}\leq 0.5$) and TPAUC ($\mathrm{TPR} \geq 0.5$, $\mathrm{FPR}\leq 0.5$) on CIFAR-10-LT-1 are visualized by boxplots in Fig.\ref{fig:sensi_init_op_0.5_cifar10-1} and Fig.\ref{fig:sensi_init_tp_0.5_cifar10-1}, respectively. The sensitivity behavior over OPAUC and TPAUC is different to a certain extent. For example,
among these variables, UPAUCI is most sensitive to the choice of initial $\bm{c}$ over OPAUC, where the performance gap between the initial value of $1$ and $0.2$ is approaching $0.1$.
Also we can see the trend is that the larger $c_i$ is initialized, the better the performance is. Such an inefficiency might be attributed to the limited selected instances with small $\bm{c}$ that make the optimization insufficient. However, in terms of TPAUC, UPAUCI is not as sensitive to $\bm{c}$ as OPAUC, but becomes very sensitive towards the choice of $s,s'$. The reason might be that TPAUC optimization depends on both the top-quantile of negative instances and bottom-quantile of positive instances, which are calculated based on the selection threshold $s,s'$. Moreover, blindly increasing the initial $\bm{c}$ does not necessarily improve the performance on TPAUC, considering the joint effect of positive and negative instances. Besides, it is recommended to choose a large $a$, a small $b$ and $\gamma=0$ for initialization over both OPAUC and TPAUC. This is consistent with the optimal solution for $a, b, \gamma$. Under different FPR or TPR constraints, the method shows a similar tendency. Please see Fig.\ref{fig:sensi_init_op_0.3_cifar10-1} and Fig.\ref{fig:sensi_init_tp_0.3_cifar10-1} in Appx.\ref{sec:add_sensi} for more results.

\noindent\textbf{Necessity of Hyperparameter $\omega$.} Unlike PAUCI which requires the additional $\omega\cdot \gamma^2$ term to achieve strongly-concavity \wrt~$\gamma$, the unbiased objective $H_{op}/H_{tp}$ in UPAUCI is already concave \wrt~the inner variables $\gamma$ and $\bm{c}$. One may think it is not necessary to involve the $\omega$-related term in UPAUCI. However, it is noteworthy that existing minimax optimizers only achieve a convergence rate of $O(\epsilon^{-6})$ under the concavity property, which is much slower than the rate of $O(\epsilon^{-3})$ under strongly-concavity. Considering this, we still add the term with a small coefficient $\omega$ to make the objective strongly-concave \wrt~$\gamma$ and $\bm{c}$ for more effective optimization. The sensitivity boxplots over OPAUC and TPAUC are presented in Fig.\ref{fig:sensi_omega_0.5_cifar10-1}. Apparently, compared with the mere concavity case ($\omega=0$), introducing the squared term significantly improves the performance, especially over TPAUC. Furthermore, when $\omega$ becomes too large (\eg, $\omega=1.0$ over OPAUC), the performance gain will be weakened by the introduced approximation bias. Therefore, a slight $\omega$ is sufficient for our method.

\section{Conclusion}\label{sec:con}
To overcome the efficiency bottleneck of existing approximate PAUC optimization methods, we derive two smooth minimax instance-wise formulations to efficient optimize $\op$ and $\tp$ in this paper. Specifically, the complicated top/bottom instance ranking process is formulated as a differentiable sample selection problem. And the problem is further smoothed by either a smooth surrogate function or an auxiliary variable based reformulation to reach a standard minimax formulation. By employing an efficient stochastic algorithm we can find an $\epsilon$-first order saddle point after $O(\epsilon^{-3})$ iterations, faster than existing results. Moreover, we present a theoretical analysis of the generalization error of our formulation. The result is tight with the order of $\tilde{O}(\alpha^{-1}\np^{-1} + \beta^{-1}\nn^{-1})$ for $\tp$. Finally, extensive empirical studies over a range of long-tailed benchmark datasets speak to the effectiveness of our proposed algorithm.


%



\ifCLASSOPTIONcaptionsoff
  \newpage
\fi


\bibliographystyle{IEEEtran}
\bibliography{egbib}

@string(AAAI = {AAAI})

@article{J1982The,
  title   = {The meaning and use of the area under a receiver operating characteristic (ROC) curve.},
  author  = {J A and Hanley and B J and McNeil},
  journal = {Radiology},
  year    = {1982}
}

@article{dunn1987convergence,
  title   = {On the convergence of projected gradient processes to singular critical points},
  author  = {Dunn, John C},
  journal = {Journal of Optimization Theory and Applications},
  volume  = {55},
  number  = {2},
  pages   = {203--216},
  year    = {1987}
}

@article{mcclish1989analyzing,
  title   = {Analyzing a portion of the ROC curve},
  author  = {McClish, Donna Katzman},
  journal = {Med. Decis. Making},
  volume  = {9},
  number  = {3},
  pages   = {190--195},
  year    = {1989}
}

@incollection{graepel2000large,
  title     = {Large margin rank boundaries for ordinal regression},
  author    = {Graepel, Thore and Obermayer, Klaus and others},
  booktitle = {Advances in Large Margin Classifiers},
  pages     = {115--132},
  year      = {2000}
}

@article{pepe2000combining,
  title   = {Combining diagnostic test results to increase accuracy},
  author  = {Pepe, Margaret Sullivan and Thompson, Mary Lou},
  journal = {Biostat.},
  volume  = {1},
  number  = {2},
  pages   = {123--140},
  year    = {2000}
}

@article{freund2003efficient,
  title   = {An efficient boosting algorithm for combining preferences},
  author  = {Freund, Yoav and Iyer, Raj and Schapire, Robert E and Singer, Yoram},
  journal = {J. Mach. Learn. Res.},
  volume  = {4},
  pages   = {933--969},
  year    = {2003}
}

@inproceedings{cortes2003auc,
  author    = {Corinna Cortes and
               Mehryar Mohri},
  title     = {{AUC} optimization vs. error rate minimization},
  booktitle = {Adv. Neural Inform. Process. Syst.},
  pages     = {313--320},
  year      = {2003}
}

@inproceedings{yan2003optimizing,
  title     = {Optimizing classifier performance via an approximation to the Wilcoxon-Mann-Whitney statistic},
  author    = {Yan, Lian and Dodier, Robert H and Mozer, Michael and Wolniewicz, Richard H},
  booktitle = {Int. Conf. Mach. Learn.},
  pages     = {848--855},
  year      = {2003}
}

@article{rakotomamonjy2004support,
  title   = {Support vector machines and area under ROC curve},
  author  = {Rakotomamonjy, Alain},
  journal = {PSI-INSA de Rouen: Technical Report},
  year    = {2004}
}

@article{agarwal2005generalization,
  title   = {Generalization bounds for the area under the ROC curve},
  author  = {Agarwal, Shivani and Graepel, Thore and Herbrich, Ralf and Har-Peled, Sariel and Roth, Dan and Jordan, Michael I},
  journal = {J. Mach. Learn. Res.},
  volume  = {6},
  pages   = {393--425},
  year    = {2005}
}

@inproceedings{joachims2005support,
  title     = {A support vector method for multivariate performance measures},
  author    = {Joachims, Thorsten},
  booktitle = {Int. Conf. Mach. Learn.},
  pages     = {377--384},
  year      = {2005}
}

@article{2003Partial,
  title   = {Partial AUC Estimation and Regression},
  author  = { Dodd, Lori E.  and  Pepe, Margaret S. },
  journal = {Biometrics},
  volume  = {59},
  number  = {3},
  pages   = {614-623},
  year    = {2003}
}

@book{boyd2004convex,
  title     = {Convex optimization},
  author    = {Boyd, Stephen and Boyd, Stephen P and Vandenberghe, Lieven},
  year      = {2004},
  publisher = {Cambridge university press}
}

@inproceedings{elson2007asirra,
  author    = {Jeremy Elson and
               John R. Douceur and
               Jon Howell and
               Jared Saul},
  title     = {Asirra: a {CAPTCHA} that exploits interest-aligned manual image categorization},
  booktitle = {{ACM} Conf. Comput. Commun. Secur.},
  pages     = {366--374},
  year      = {2007}
}

@article{krizhevsky2009learning,
  title     = {Learning multiple layers of features from tiny images},
  author    = {Krizhevsky, Alex and Hinton, Geoffrey and others},
  year      = {2009},
  publisher = {Citeseer}
}

@article{wang2011marker,
  title   = {Marker selection via maximizing the partial area under the ROC curve of linear risk scores},
  author  = {Wang, Zhanfeng and Chang, Yuan-Chin Ivan},
  journal = {Biostat.},
  volume  = {12},
  number  = {2},
  pages   = {369--385},
  year    = {2011}
}

@inproceedings{glorot2011deep,
  author    = {Xavier Glorot and
               Antoine Bordes and
               Yoshua Bengio},
  title     = {Deep Sparse Rectifier Neural Networks},
  booktitle = {Int. Conf. Artif. Intell. Stat.},
  pages     = {315--323},
  year      = {2011}
}

@inproceedings{narasimhan2013structural,
  title     = {A structural SVM based approach for optimizing partial AUC},
  author    = {Narasimhan, Harikrishna and Agarwal, Shivani},
  booktitle = {Int. Conf. Mach. Learn.},
  pages     = {516--524},
  year      = {2013}
}

@article{ying2016stochastic,
  title   = {Stochastic online auc maximization},
  author  = {Ying, Yiming and Wen, Longyin and Lyu, Siwei},
  journal = {Adv. Neural Inform. Process. Syst.},
  pages   = {451--459},
  year    = {2016}
}

@article{ghadimi2016mini,
  title   = {Mini-batch stochastic approximation methods for nonconvex stochastic composite optimization},
  author  = {Ghadimi, Saeed and Lan, Guanghui and Zhang, Hongchao},
  journal = {Math. Program.},
  volume  = {155},
  number  = {1},
  pages   = {267--305},
  year    = {2016}
}

@article{kar2014online,
  title   = {Online and stochastic gradient methods for non-decomposable loss functions},
  author  = {Kar, Purushottam and Narasimhan, Harikrishna and Jain, Prateek},
  journal = {Adv. Neural Inform. Process. Syst.},
  pages   = {694--702},
  year    = {2014}
}

@article{fan2017learning,
  title   = {Learning with average top-k loss},
  author  = {Fan, Yanbo and Lyu, Siwei and Ying, Yiming and Hu, Baogang},
  journal = {Adv. Neural Inform. Process. Syst.},
  pages   = {497--505},
  year    = {2017}
}

@article{narasimhan2017support,
  title   = {Support vector algorithms for optimizing the partial area under the ROC curve},
  author  = {Narasimhan, Harikrishna and Agarwal, Shivani},
  journal = {Neural Comput.},
  volume  = {29},
  number  = {7},
  pages   = {1919--1963},
  year    = {2017}
}

@inproceedings{natole2018stochastic,
  title     = {Stochastic proximal algorithms for AUC maximization},
  author    = {Natole, Michael and Ying, Yiming and Lyu, Siwei},
  booktitle = {Int. Conf. Mach. Learn.},
  pages     = {3710--3719},
  year      = {2018}
}

@inproceedings{liu2019stochastic,
  author    = {Mingrui Liu and
               Zhuoning Yuan and
               Yiming Ying and
               Tianbao Yang},
  title     = {Stochastic {AUC} maximization with deep neural networks},
  booktitle = {Int. Conf. Learn. Represent.},
  year      = {2020}
}

@article{yang2019two,
  title   = {Two-way partial AUC and its properties},
  author  = {Yang, Hanfang and Lu, Kun and Lyu, Xiang and Hu, Feifang},
  journal = {Statistical Methods in Medical Research},
  volume  = {28},
  number  = {1},
  pages   = {184--195},
  year    = {2019}
}

@article{razaviyayn2020nonconvex,
  title   = {Nonconvex min-max optimization: Applications, challenges, and recent theoretical advances},
  author  = {Razaviyayn, Meisam and Huang, Tianjian and Lu, Songtao and Nouiehed, Maher and Sanjabi, Maziar and Hong, Mingyi},
  journal = {IEEE Signal Process. Mag.},
  volume  = {37},
  number  = {5},
  pages   = {55--66},
  year    = {2020}
}

@inproceedings{guo2020communication,
  title     = {Communication-efficient distributed stochastic auc maximization with deep neural networks},
  author    = {Guo, Zhishuai and Liu, Mingrui and Yuan, Zhuoning and Shen, Li and Liu, Wei and Yang, Tianbao},
  booktitle = {Int. Conf. Mach. Learn.},
  pages     = {3864--3874},
  year      = {2020}
}

@article{tsaknakis2021minimax,
  title   = {Minimax problems with coupled linear constraints: computational complexity and duality},
  author  = {Tsaknakis, Ioannis and Hong, Mingyi and Zhang, Shuzhong},
  journal = {SIAM J. Optim.},
  volume = {33},
  number = {4},
  pages = {2675-2702},
  year = {2023}
}

@inproceedings{yuan2021large,
  title     = {Large-scale robust deep auc maximization: A new surrogate loss and empirical studies on medical image classification},
  author    = {Yuan, Zhuoning and Yan, Yan and Sonka, Milan and Yang, Tianbao},
  booktitle = {{IEEE/CVF} Int. Conf. Comput. Vis.},
  pages     = {3040--3049},
  year      = {2021}
}

@article{yang2021deep,
  title   = {Deep AUC Maximization for Medical Image Classification: Challenges and Opportunities},
  author  = {Yang, Tianbao},
  journal = {arXiv preprint arXiv:2111.02400},
  year    = {2021}
}

@inproceedings{kumar2021implicit,
  title     = {Implicit rate-constrained optimization of non-decomposable objectives},
  author    = {Kumar, Abhishek and Narasimhan, Harikrishna and Cotter, Andrew},
  booktitle = {Int. Conf. Mach. Learn.},
  pages     = {5861--5871},
  year      = {2021}
}

@inproceedings{yang2021all,
  title     = {When All We Need is a Piece of the Pie: A Generic Framework for Optimizing Two-way Partial AUC},
  author    = {Yang, Zhiyong and Xu, Qianqian and Bao, Shilong and He, Yuan and Cao, Xiaochun and Huang, Qingming},
  booktitle = {Int. Conf. Mach. Learn.},
  pages     = {11820--11829},
  year      = {2021}
}

@article{zhu2022auc,
  title   = {When AUC meets DRO: Optimizing Partial AUC for Deep Learning with Non-Convex Convergence Guarantee},
  author  = {Zhu, Dixian and Li, Gang and Wang, Bokun and Wu, Xiaodong and Yang, Tianbao},
  journal={Int. Conf. Mach. Learn.},
  pages={27548--27573},
  year={2022}
}

@article{yao2022large,
  title   = {Large-scale Optimization of Partial AUC in a Range of False Positive Rates},
  author  = {Yao, Yao and Lin, Qihang and Yang, Tianbao},
  journal={Adv. Neural Inform. Process. Syst.},
  pages={31239--31253},
  year={2022}
}

@article{yang2021learning,
  title   = {Learning with Multiclass AUC: Theory and Algorithms},
  author  = {Yang, Zhiyong and Xu, Qianqian and Bao, Shilong and Cao, Xiaochun and Huang, Qingming},
  journal = {{IEEE} Trans. Pattern Anal. Mach. Intell.},
  volume={44},
  number={11},
  pages={7747--7763},
  year={2021}
}

@article{huang2022accelerated,
  title   = {Accelerated zeroth-order and first-order momentum methods from mini to minimax optimization},
  author  = {Huang, Feihu and Gao, Shangqian and Pei, Jian and Huang, Heng},
  journal = {J. Mach. Learn. Res.},
  volume  = {23},
  number  = {36},
  pages   = {1--70},
  year    = {2022}
}

@article{yang2022auc,
  title   = {AUC Maximization in the Era of Big Data and AI: A Survey},
  author  = {Yang, Tianbao and Ying, Yiming},
  journal={ACM Comput. Surv.},
  volume={55},
  number={8},
  pages={1--37},
  year={2022}
}

@inproceedings{yuan2021compositional,
  title     = {Compositional Training for End-to-End Deep AUC Maximization},
  author    = {Yuan, Zhuoning and Guo, Zhishuai and Chawla, Nitesh and Yang, Tianbao},
  booktitle = {Int. Conf. Learn. Represent.},
  year      = {2021}
}

@inproceedings{huang2022auc,
  title     = {AUC-oriented Graph Neural Network for Fraud Detection},
  author    = {Huang, Mengda and Liu, Yang and Ao, Xiang and Li, Kuan and Chi, Jianfeng and Feng, Jinghua and Yang, Hao and He, Qing},
  booktitle = {The {ACM} Web Conf.},
  pages     = {1311--1321},
  year      = {2022}
}

@article{shao2022asymptotically,
  title={Asymptotically unbiased instance-wise regularized partial auc optimization: Theory and algorithm},
  author={Shao, Huiyang and Xu, Qianqian and Yang, Zhiyong and Bao, Shilong and Huang, Qingming},
  journal={Adv. Neural Inform. Process. Syst.},
  pages={38667--38679},
  year={2022}
}

@inproceedings{DBLP:conf/kdd/ZhangSLG23,
  author    = {Chenkang Zhang and
               Wanli Shi and
               Lei Luo and
               Bin Gu},
  title     = {Doubly Robust {AUC} Optimization against Noisy and Adversarial Samples},
  booktitle = {{ACM} {SIGKDD} Int. Conf. Knowl. Discov. Data Min.},
  pages     = {3195--3205},
  year      = {2023},
}

@inproceedings{DBLP:conf/nips/HuZY22,
  author    = {Quanqi Hu and
               Yongjian Zhong and
               Tianbao Yang},
  title     = {Multi-block Min-max Bilevel Optimization with Applications in Multi-task
               Deep {AUC} Maximization},
  pages = {29552--29565},
  booktitle = {Adv. Neural Inform. Process. Syst.},
  year      = {2022},
}

@inproceedings{DBLP:conf/icml/Guo0LY23,
  author    = {Zhishuai Guo and
               Rong Jin and
               Jiebo Luo and
               Tianbao Yang},
  title     = {FeDXL: Provable Federated Learning for Deep X-Risk Optimization},
  booktitle = {Int. Conf. Mach. Learn.},
  pages     = {11934--11966},
  year      = {2023},
}

@inproceedings{DBLP:conf/icml/ZhuWCWSWY23,
  author    = {Dixian Zhu and
               Bokun Wang and
               Zhi Chen and
               Yaxing Wang and
               Milan Sonka and
               Xiaodong Wu and
               Tianbao Yang},
  title     = {Provable Multi-instance Deep {AUC} Maximization with Stochastic Pooling},
  booktitle = {Int. Conf. Mach. Learn.},
  pages     = {43205--43227},
  year      = {2023},
}

@article{DBLP:journals/pami/BaoX0CH23,
  author    = {Shilong Bao and
               Qianqian Xu and
               Zhiyong Yang and
               Xiaochun Cao and
               Qingming Huang},
  title     = {Rethinking Collaborative Metric Learning: Toward an Efficient Alternative
               Without Negative Sampling},
  journal   = {{IEEE} Trans. Pattern Anal. Mach. Intell.},
  volume    = {45},
  number    = {1},
  pages     = {1017--1035},
  year      = {2023},
}

@inproceedings{DBLP:conf/iccv/Yuan0SY21,
  author    = {Zhuoning Yuan and
               Yan Yan and
               Milan Sonka and
               Tianbao Yang},
  title     = {Large-scale Robust Deep {AUC} Maximization: {A} New Surrogate Loss
               and Empirical Studies on Medical Image Classification},
  booktitle = {{IEEE/CVF} Int. Conf. Comput. Vis.},
  pages     = {3020--3029},
  year      = {2021},
}

@article{xie2024weakly,
  title     = {Weakly supervised AUC optimization: a unified partial AUC approach},
  author    = {Xie, Zheng and Liu, Yu and He, Hao-Yuan and Li, Ming and Zhou, Zhi-Hua},
  journal   = {{IEEE} Trans. Pattern Anal. Mach. Intell.},
  year      = {2024},
  volume       = {46},
  number       = {7},
  pages        = {4780--4795}
}

@article{DBLP:journals/pami/YangXBWHCH23,
  author    = {Zhiyong Yang and
               Qianqian Xu and
               Shilong Bao and
               Peisong Wen and
               Yuan He and
               Xiaochun Cao and
               Qingming Huang},
  title     = {AUC-Oriented Domain Adaptation: From Theory to Algorithm},
  journal   = {{IEEE} Trans. Pattern Anal. Mach. Intell.},
  volume    = {45},
  number    = {12},
  pages     = {14161--14174},
  year      = {2023},
}

@article{DBLP:journals/pami/YangXHBHCH23,
  author    = {Zhiyong Yang and
               Qianqian Xu and
               Wenzheng Hou and
               Shilong Bao and
               Yuan He and
               Xiaochun Cao and
               Qingming Huang},
  title     = {Revisiting AUC-Oriented Adversarial Training With Loss-Agnostic Perturbations},
  journal   = {{IEEE} Trans. Pattern Anal. Mach. Intell.},
  volume    = {45},
  number    = {12},
  pages     = {15494--15511},
  year      = {2023},
}

@inproceedings{DBLP:conf/aaai/YangXCH19,
  author    = {Zhiyong Yang and
               Qianqian Xu and
               Xiaochun Cao and
               Qingming Huang},
  title     = {Learning Personalized Attribute Preference via Multi-Task {AUC} Optimization},
  booktitle = {{AAAI} Conf. Artif. Intell.},
  pages     = {5660--5667},
  year      = {2019},
}

@incollection{herbrich2000large,
  title     = {Large Margin Bank Boundaries for Ordinal Regression},
  author    = {Herbrich, Ralf and Graepel, Thore and Obermayer, Klaus},
  booktitle = {Advances in Large Margin Classifiers},
  year      = {2000},
  pages     = {115--132},
  publisher = {MIT Press}
}

@book{mohri2018foundations,
  title     = {Foundations of machine learning},
  author    = {Mohri, Mehryar and Rostamizadeh, Afshin and Talwalkar, Ameet},
  year      = {2018},
  publisher = {MIT press}
}

@inproceedings{usunier2005data,
  title        = {A data-dependent generalisation error bound for the AUC},
  author       = {Usunier, Nicolas and Amini, Massih-Reza and Gallinari, Patrick},
  booktitle    = {Int. Conf. Mach. Learn. Worksh.},
  year         = {2005}
}

@article{lei2020sharper,
  title   = {Sharper generalization bounds for pairwise learning},
  author  = {Lei, Yunwen and Ledent, Antoine and Kloft, Marius},
  journal = {Adv. Neural Inform. Process. Syst.},
  pages   = {21236--21246},
  year    = {2020}
}

@article{lei2021generalization,
  title   = {Generalization guarantee of SGD for pairwise learning},
  author  = {Lei, Yunwen and Liu, Mingrui and Ying, Yiming},
  journal = {Adv. Neural Inform. Process. Syst.},
  pages   = {21216--21228},
  year    = {2021}
}

@article{bartlett2005local,
  title     = {Local Rademacher Complexities},
  author    = {Bartlett, Peter L and Bousquet, Olivier and Mendelson, Shahar},
  journal   = {Annals of Statistics},
  pages     = {1497--1537},
  year      = {2005}
}

@article{lei2016local,
  title     = {Local rademacher complexity bounds based on covering numbers},
  author    = {Lei, Yunwen and Ding, Lixin and Bi, Yingzhou},
  journal   = {Neurocomputing},
  volume    = {218},
  pages     = {320--330},
  year      = {2016}
}

@article{zhang2022sapd+,
  title   = {Sapd+: An accelerated stochastic method for nonconvex-concave minimax problems},
  author  = {Zhang, Xuan and Aybat, Necdet Serhat and Gurbuzbalaban, Mert},
  journal = {Adv. Neural Inform. Process. Syst.},
  pages   = {21668--21681},
  year    = {2022}
}

@article{DBLP:journals/pami/YangXBHCH23,
  author    = {Zhiyong Yang and
               Qianqian Xu and
               Shilong Bao and
               Yuan He and
               Xiaochun Cao and
               Qingming Huang},
  title     = {Optimizing Two-Way Partial {AUC} With an End-to-End Framework},
  journal   = {{IEEE} Trans. Pattern Anal. Mach. Intell.},
  volume    = {45},
  number    = {8},
  pages     = {10228--10246},
  year      = {2023},
}

@inproceedings{longgeneralization,
  title     = {Generalization bounds for deep convolutional neural networks},
  author    = {Long, Philip M and Sedghi, Hanie},
  booktitle = {Int. Conf. Learn. Represent.},
  year      = {2020}
}

@article{yang2020stochastic,
  title   = {Stochastic AUC optimization with general loss},
  author  = {Yang, Zhenhuan and Shen, Wei and Ying, Yiming and Yuan, Xiaoming},
  journal = {Communications on Pure \& Applied Analysis},
  volume  = {19},
  number  = {8},
  year    = {2020},
  pages   = {4191--4212}
}

@article{DBLP:journals/pami/TangSQLWYJ17,
  author    = {Jinhui Tang and
               Xiangbo Shu and
               Guo{-}Jun Qi and
               Zechao Li and
               Meng Wang and
               Shuicheng Yan and
               Ramesh C. Jain},
  title     = {Tri-Clustered Tensor Completion for Social-Aware Image Tag Refinement},
  journal   = {{IEEE} Trans. Pattern Anal. Mach. Intell.},
  volume    = {39},
  number    = {8},
  pages     = {1662--1674},
  year      = {2017},
}

@article{DBLP:journals/pami/LiTM19,
  author    = {Zechao Li and
               Jinhui Tang and
               Tao Mei},
  title     = {Deep Collaborative Embedding for Social Image Understanding},
  journal   = {{IEEE} Trans. Pattern Anal. Mach. Intell.},
  volume    = {41},
  number    = {9},
  pages     = {2070--2083},
  year      = {2019}
}

@inproceedings{DBLP:conf/nips/ZhangZT0S20,
  author    = {Dong Zhang and
               Hanwang Zhang and
               Jinhui Tang and
               Xian{-}Sheng Hua and
               Qianru Sun},
  title     = {Causal Intervention for Weakly-Supervised Semantic Segmentation},
  booktitle = {Adv. Neural Inform. Process. Syst.},
  pages={655--666},
  year      = {2020},
}

@inproceedings{DBLP:conf/www/Ye0025,
  author    = {Shanshan Ye and
               Jie Lu and
               Guangquan Zhang},
  title     = {Towards Safe Machine Unlearning: {A} Paradigm that Mitigates Performance
               Degradation},
  booktitle = {The {ACM} Web Conf.},
  pages     = {4635--4652},
  year      = {2025},
}

@article{DBLP:journals/tist/YeL25,
  author    = {Shanshan Ye and
               Jie Lu},
  title     = {Robust Recommender Systems with Rating Flip Noise},
  journal   = {{ACM} Trans. Intell. Syst. Technol.},
  volume    = {16},
  number    = {1},
  pages     = {11:1--11:19},
  year      = {2025},
}

@article{DBLP:journals/pami/GaoWFHJ23,
  author    = {Zhi Gao and
               Yuwei Wu and
               Xiaomeng Fan and
               Mehrtash Harandi and
               Yunde Jia},
  title     = {Learning to Optimize on Riemannian Manifolds},
  journal   = {{IEEE} Trans. Pattern Anal. Mach. Intell.},
  volume    = {45},
  number    = {5},
  pages     = {5935--5952},
  year      = {2023}
}

@article{DBLP:journals/tnn/XiaLHYJG23,
  author    = {Jingyuan Xia and
               Shengxi Li and
               Junjie Huang and
               Zhixiong Yang and
               Imad M. Jaimoukha and
               Deniz G{\"{u}}nd{\"{u}}z},
  title     = {Metalearning-Based Alternating Minimization Algorithm for Nonconvex
               Optimization},
  journal   = {{IEEE} Trans. Neural Networks Learn. Syst.},
  volume    = {34},
  number    = {9},
  pages     = {5366--5380},
  year      = {2023}
}

@inproceedings{DBLP:conf/cvpr/YangXLHZLFL24,
  author    = {Zhixiong Yang and
               Jingyuan Xia and
               Shengxi Li and
               Xinghua Huang and
               Shuanghui Zhang and
               Zhen Liu and
               Yaowen Fu and
               Yongxiang Liu},
  title     = {A Dynamic Kernel Prior Model for Unsupervised Blind Image Super-Resolution},
  booktitle = {{IEEE/CVF} Conf. Comput. Vis. Pattern Recog.},
  pages     = {26046--26056},
  year      = {2024},
}

@article{DBLP:journals/pami/XiaYLZFGL24,
  author    = {Jingyuan Xia and
               Zhixiong Yang and
               Shengxi Li and
               Shuanghui Zhang and
               Yaowen Fu and
               Deniz G{\"{u}}nd{\"{u}}z and
               Xiang Li},
  title     = {Blind Super-Resolution via Meta-Learning and Markov Chain Monte Carlo
               Simulation},
  journal   = {{IEEE} Trans. Pattern Anal. Mach. Intell.},
  volume    = {46},
  number    = {12},
  pages     = {8139--8156},
  year      = {2024}
}

@article{DBLP:journals/pami/JiangLCHXYCH23,
  author    = {Yangbangyan Jiang and
               Xiaodan Li and
               Yuefeng Chen and
               Yuan He and
               Qianqian Xu and
               Zhiyong Yang and
               Xiaochun Cao and
               Qingming Huang},
  title     = {MaxMatch: Semi-Supervised Learning With Worst-Case Consistency},
  journal   = {{IEEE} Trans. Pattern Anal. Mach. Intell.},
  volume    = {45},
  number    = {5},
  pages     = {5970--5987},
  year      = {2023}
}

@article{DBLP:journals/pami/JiangXZYWCH23,
  author    = {Yangbangyan Jiang and
               Qianqian Xu and
               Yunrui Zhao and
               Zhiyong Yang and
               Peisong Wen and
               Xiaochun Cao and
               Qingming Huang},
  title     = {Positive-Unlabeled Learning With Label Distribution Alignment},
  journal   = {{IEEE} Trans. Pattern Anal. Mach. Intell.},
  volume    = {45},
  number    = {12},
  pages     = {15345--15363},
  year      = {2023}
}

@inproceedings{DBLP:conf/nips/HanX0BWJH24,
  author    = {Boyu Han and
               Qianqian Xu and
               Zhiyong Yang and
               Shilong Bao and
               Peisong Wen and
               Yangbangyan Jiang and
               Qingming Huang},
  title     = {AUCSeg: AUC-oriented Pixel-level Long-tail Semantic Segmentation},
  booktitle = {Adv. Neural Inform. Process. Syst.},
  year      = {2024}
}

@article{shi2025llmformer,
  title     = {LLMFormer: Large language model for open-vocabulary semantic segmentation},
  author    = {Shi, Hengcan and Dao, Son Duy and Cai, Jianfei},
  journal   = {Int. J. Comput. Vis.},
  volume    = {133},
  number    = {2},
  pages     = {742--759},
  year      = {2025}
}

@article{shi2024unified,
  title     = {Unified open-vocabulary dense visual prediction},
  author    = {Shi, Hengcan and Hayat, Munawar and Cai, Jianfei},
  journal   = {{IEEE} Trans. Multim.},
  volume    = {26},
  pages     = {8704--8716},
  year      = {2024}
}

@inproceedings{shi2022proposalclip,
  title     = {Proposalclip: Unsupervised open-category object proposal generation via exploiting clip cues},
  author    = {Shi, Hengcan and Hayat, Munawar and Wu, Yicheng and Cai, Jianfei},
  booktitle = {{IEEE/CVF} Conf. Comput. Vis. Pattern Recog.},
  pages     = {9611--9620},
  year      = {2022}
}

@inproceedings{Qiu_2021_ICCV,
  author    = {Qiu, Heqian and Li, Hongliang and Wu, Qingbo and Cui, Jianhua and Song, Zichen and Wang, Lanxiao and Zhang, Minjian},
  title     = {CrossDet: Crossline Representation for Object Detection},
  booktitle = {{IEEE/CVF} Conf. Comput. Vis. Pattern Recog.},
  year      = {2021},
  pages     = {3195-3204}
}

@inproceedings{Qiu_2020_CVPR,
  author    = {Qiu, Heqian and Li, Hongliang and Wu, Qingbo and Shi, Hengcan},
  title     = {Offset Bin Classification Network for Accurate Object Detection},
  booktitle = {{IEEE/CVF} Conf. Comput. Vis. Pattern Recog.},
  pages        = {13185--13194},
  year      = {2020}
}

@article{10681116,
  author   = {Qiu, Heqian and Wang, Lanxiao and Zhao, Taijin and Meng, Fanman and Wu, Qingbo and Li, Hongliang},
  journal  = {{IEEE} Trans. Circuits Syst. Video Technol.},
  title    = {MCCE-REC: MLLM-Driven Cross-Modal Contrastive Entropy Model for Zero-Shot Referring Expression Comprehension},
  year     = {2025},
  volume   = {35},
  number   = {1},
  pages    = {754-768}
}

@inproceedings{DBLP:conf/aaai/ZhangCZWZLS25,
  author    = {Junxuan Zhang and
               Zhengxue Cheng and
               Yan Zhao and
               Shihao Wang and
               Dajiang Zhou and
               Guo Lu and
               Li Song},
  title     = {{L3TC:} Leveraging {RWKV} for Learned Lossless Low-Complexity Text
               Compression},
  booktitle = {{AAAI} Conf. Artif. Intell.},
  pages     = {13251--13259},
  year      = {2025},
}

@inproceedings{DBLP:conf/cvpr/0003CWWHLS25,
  author    = {Donghui Feng and
               Zhengxue Cheng and
               Shen Wang and
               Ronghua Wu and
               Hongwei Hu and
               Guo Lu and
               Li Song},
  title     = {Linear Attention Modeling for Learned Image Compression},
  booktitle = {{IEEE/CVF} Conf. Comput. Vis. Pattern Recog.},
  pages     = {7623--7632},
  year      = {2025}
}

@inproceedings{wang2025enhanced,
  title     = {Enhanced Semantic Extraction and Guidance for UGC Image Super Resolution},
  author    = {Wang, Yiwen and Liang, Ying and Zhang, Yuxuan and Chai, Xinning and Cheng, Zhengxue and Qin, Yingsheng and Yang, Yucai and Xie, Rong and Song, Li},
  booktitle = {{IEEE/CVF} Conf. Comput. Vis. Pattern Recog.},
  pages     = {1421--1430},
  year      = {2025}
}

%

\begin{IEEEbiography}[{\includegraphics[width=1in,height=1.25in,clip,keepaspectratio]{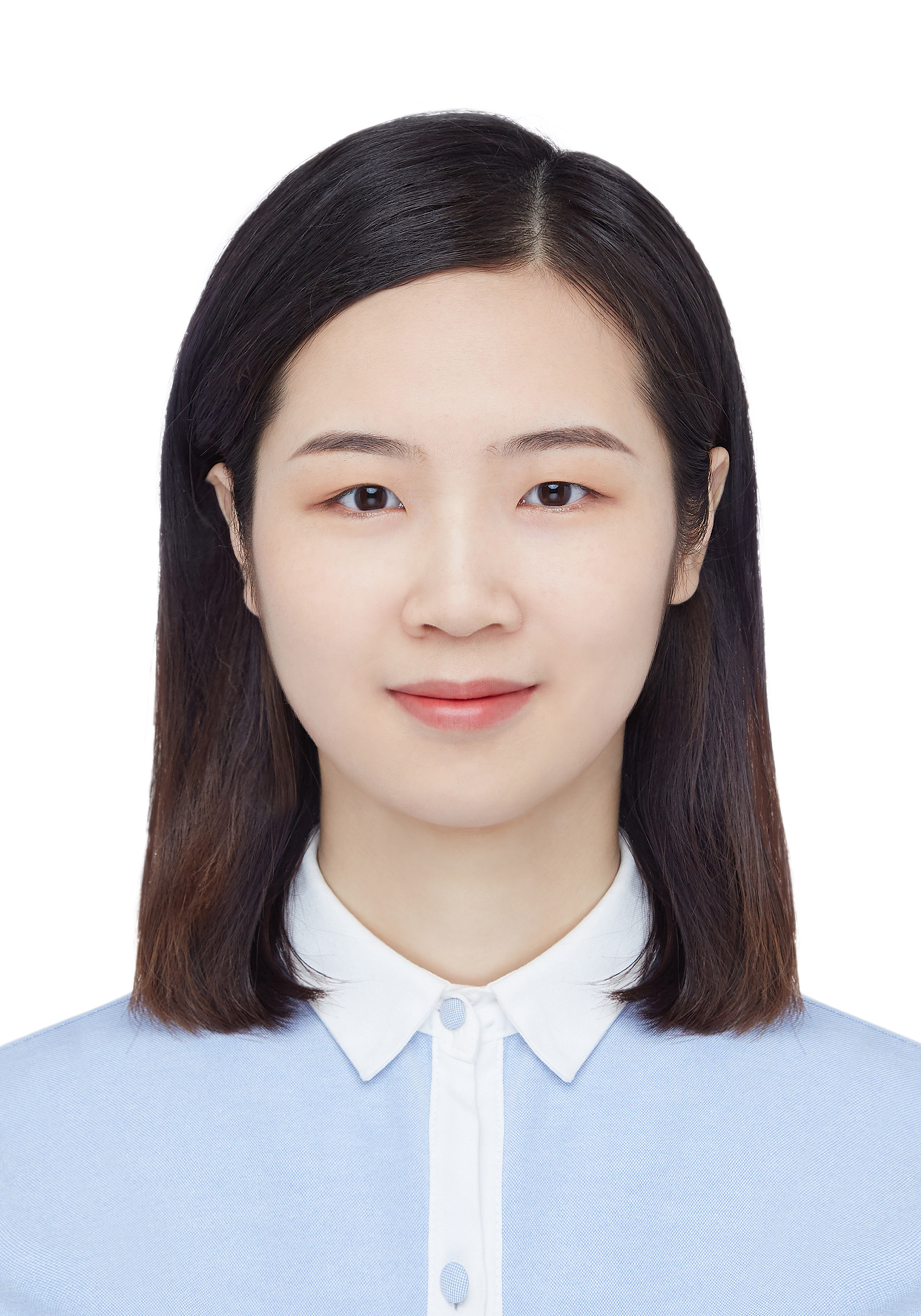}}]
	{Yangbangyan Jiang} received the B.S. degree in instrumentation and control from Beihang University in 2017 and the Ph.D. degree in computer science from University of Chinese Academy of Sciences in 2023. She is currently a postdoctoral research fellow with University of Chinese Academy of Sciences. Her research interests include machine learning and computer vision. She has authored or coauthored 20+ academic papers in international journals and conferences including T-PAMI, NeurIPS, CVPR, AAAI, ACM MM, {etc}. She served as a reviewer for several top-tier conferences such as ICML, NeurIPS, ICLR, CVPR, ICCV, AAAI.
\end{IEEEbiography}
\begin{IEEEbiography}[{\includegraphics[width=1in,height=1.25in,clip,keepaspectratio]{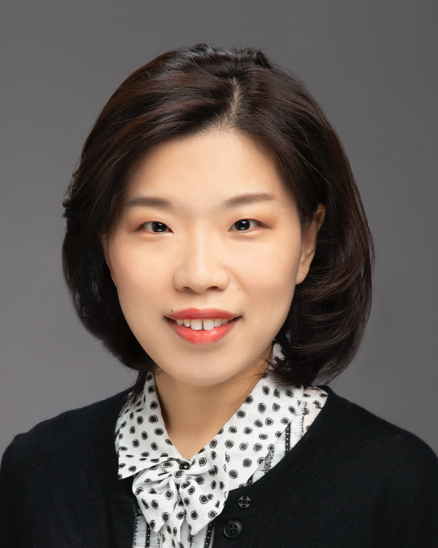}}]
  {Qianqian Xu} received the B.S. degree in computer science from China University of Mining and Technology in 2007 and the Ph.D. degree in computer science from University of Chinese Academy of Sciences in 2013. She is currently a Professor with the Institute of Computing Technology, Chinese Academy of Sciences, Beijing, China. Her research interests include statistical machine learning, with applications in multimedia and computer vision. She has authored or coauthored 100+ academic papers in prestigious international journals and conferences (including T-PAMI, IJCV, T-IP, NeurIPS, ICML, CVPR, AAAI, etc). Moreover, she serves as an associate editor of IEEE Transactions on Circuits and Systems for Video Technology, IEEE Transactions on Multimedia, and ACM Transactions on Multimedia Computing, Communications, and Applications.
\end{IEEEbiography}
\begin{IEEEbiography}
    [{\includegraphics[width=1in,height=1.25in,clip,keepaspectratio]{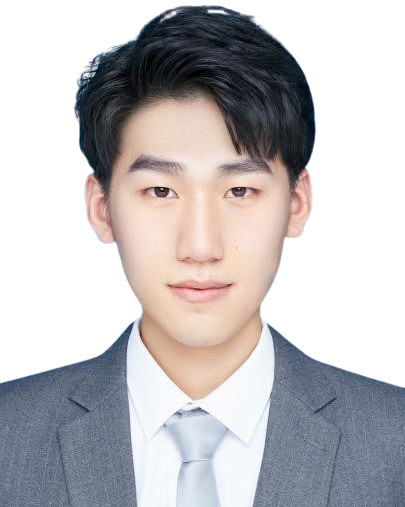}}]
    {Huiyang Shao} received the B.S. degree in software engineering from Liao Ning University in 2021 and the M.Sc. degree in computer science from University of Chinese Academy of Sciences (UCAS) in 2024. He is currently an algorithm engineer in ByteDance Inc. His research interests lie in machine learning and learning theory, with special focus on AUC optimization. He has authored or co-authored 5 academic papers in prestigious conferences (including NeurIPS, ICML, ICCV, AAAI, etc). Moreover, He serves as a reviewer for several top-tier conferences such as NeurIPS and ICLR.
    \end{IEEEbiography}
\begin{IEEEbiography}[{\includegraphics[width=1in,height=1.25in,clip,keepaspectratio]{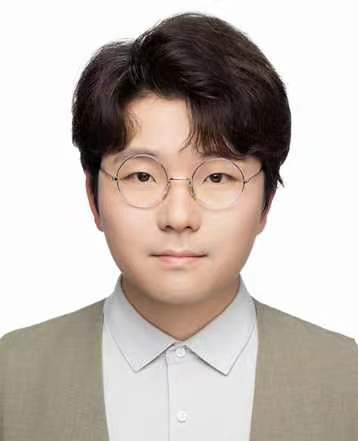}}]
{Zhiyong Yang} received his M.Sc. degree in computer science and technology from the University of Science and Technology Beijing (USTB) in 2017, and Ph.D. degree from the University of Chinese Academy of Sciences (UCAS) in 2021. He is currently an Associate Professor at the University of Chinese Academy of Sciences. His research interests include trustworthy machine learning, long-tail learning, and optimization frameworks for complex metrics. He is one of the key developers of the X-curve learning framework (https://xcurveopt.github.io/), designed to address decision biases between model trainers and users. His work has been recognized with various awards, including Top 100 Baidu AI Chinese Rising Stars Around the World, Top-20 Nomination for the Baidu Fellowship,  Asian Trustworthy Machine Learning (ATML) Fellowship, and the China Computer Federation (CCF) Doctoral Dissertation Award. He has authored or co-authored over 60 papers in top-tier international conferences and journals, including more than 30 papers in T-PAMI, ICML, and NeurIPS. He has also served as an Area Chair (AC) for NeurIPS 2024/ICLR 2025, a Senior Program Committee (SPC) member for IJCAI 2021, and as a reviewer for several prestigious journals and conferences, such as T-PAMI, IJCV, TMLR, ICML, NeurIPS, and ICLR.
\end{IEEEbiography}
\begin{IEEEbiography}
	[{\includegraphics[width=1in,height=1.25in,clip,keepaspectratio]{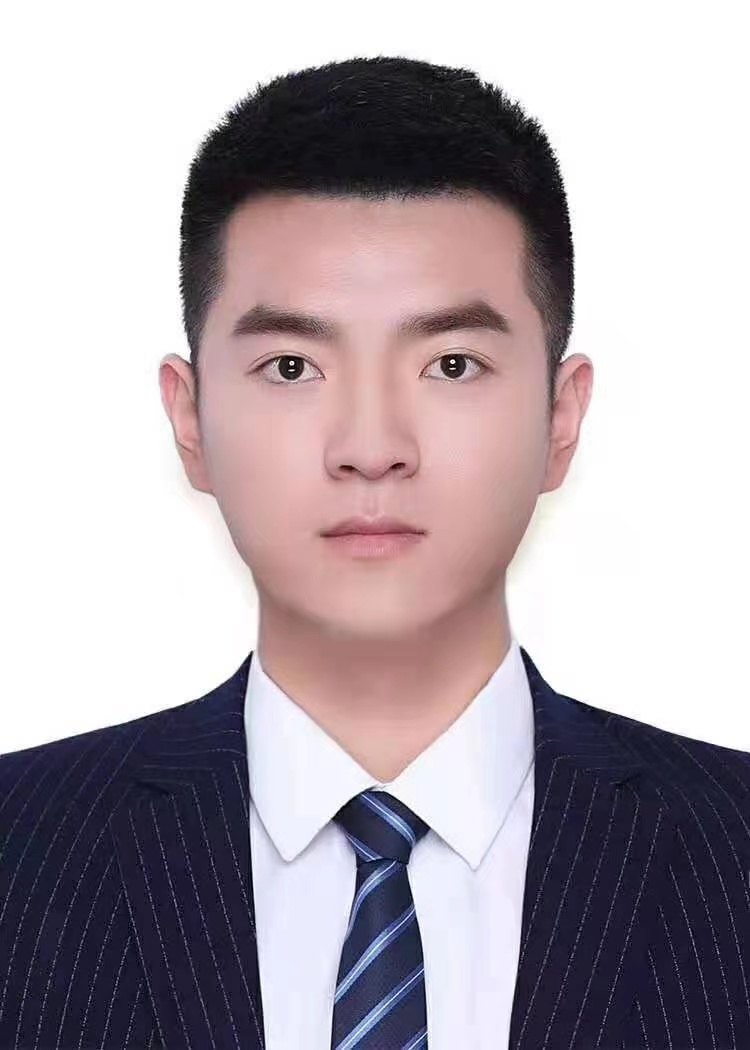}}]
  {Shilong Bao} received the B.S. degree from the College of Computer Science and Technology, Qingdao University in 2019 and the Ph.D. degree from the Institute of Information Engineering, Chinese Academy of Sciences (IIE, CAS) in 2024. He is currently a Post-doc Fellow with the School of Computer Science and Technology, University of Chinese Academy of Sciences (UCAS). His research interests are machine learning and data mining. He has authored or co-authored several academic papers in top-tier international conferences and journals including T-PAMI, NeurIPS, ICML, and ACM Multimedia. He also served as a reviewer for several top-tier conferences and journals, including ICML/NeurIPS/ICLR and IEEE T-MM/T-CSVT.
\end{IEEEbiography}
\begin{IEEEbiography}[{\includegraphics[width=1in,height=1.25in,clip,keepaspectratio]{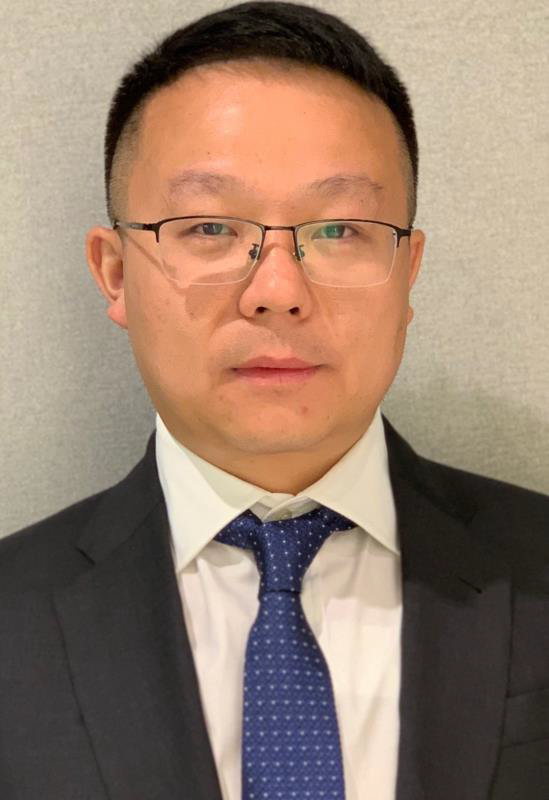}}]
 {Xiaochun Cao} is a Professor of School of Cyber Science and Technology, Shenzhen Campus of Sun Yat-sen University. He received the B.E. and M.E. degrees both in computer science from Beihang University (BUAA), China, and the Ph.D. degree in computer science from the University of Central Florida, USA, with his dissertation nominated for the university level Outstanding Dissertation Award. After graduation, he spent about three years at ObjectVideo Inc. as a Research Scientist. From 2008 to 2012, he was a professor at Tianjin University. Before joining SYSU, he was a professor at Institute of Information Engineering, Chinese Academy of Sciences. He has authored and coauthored over 200 journal and conference papers. In 2004 and 2010, he was the recipients of the Piero Zamperoni best student paper award at the International Conference on Pattern Recognition. He is on the editorial boards of IEEE Transactions on Image Processing and IEEE Transactions on Multimedia, and was on the editorial board of IEEE Transactions on Circuits and Systems for Video Technology.
\end{IEEEbiography}
\begin{IEEEbiography}[{\includegraphics[width=1in,height=1.25in,clip,keepaspectratio]{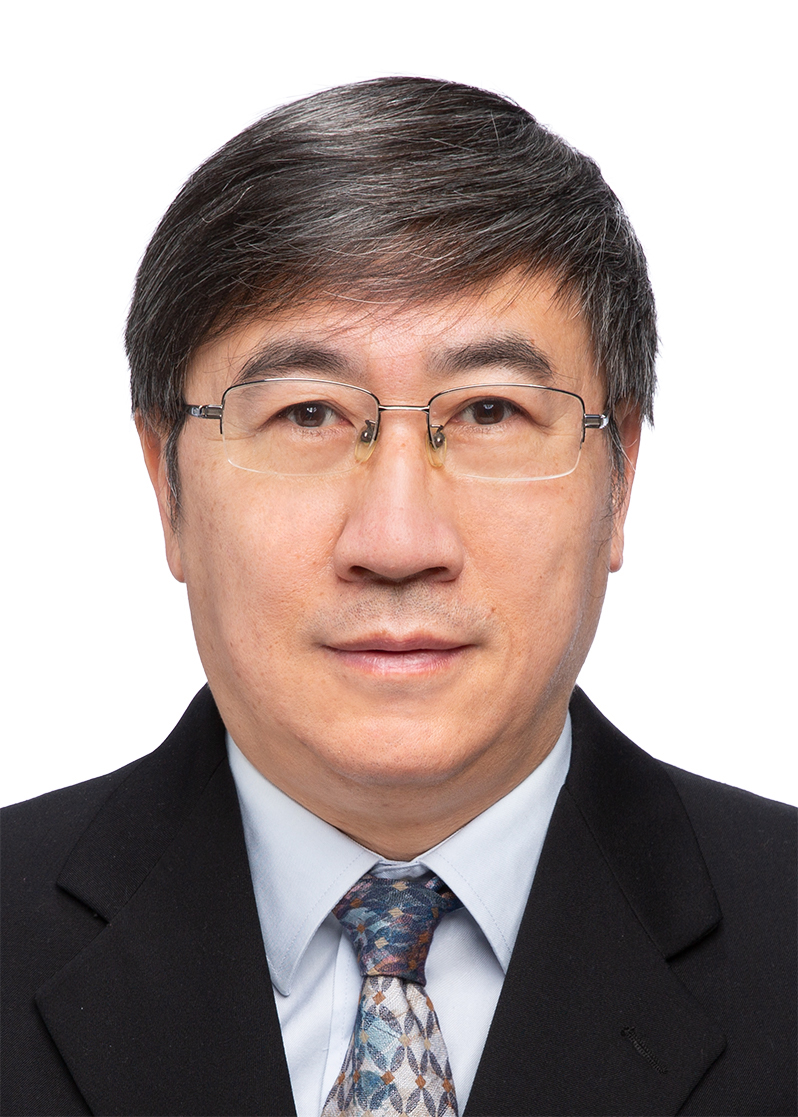}}]
  {Qingming Huang} is a chair professor in University of Chinese Academy of Sciences and an adjunct research professor in the Institute of Computing Technology, Chinese Academy of Sciences. He graduated with a Bachelor degree in Computer Science in 1988 and Ph.D. degree in Computer Engineering in 1994, both from Harbin Institute of Technology, China. His research areas include multimedia computing, image processing, computer vision and pattern recognition. He has authored or coauthored more than 400 academic papers in prestigious international journals and top-level international conferences. He was the associate editor of IEEE Trans. on CSVT and Acta Automatica Sinica, and the reviewer of various international journals including IEEE Trans. on PAMI, IEEE Trans. on Image Processing, IEEE Trans. on Multimedia, etc. He is a Fellow of IEEE and has served as general chair, program chair, area chair and TPC member for various conferences, including ACM Multimedia, CVPR, ICCV, ICME, ICMR, PCM, BigMM, PSIVT, etc.
\end{IEEEbiography}





\clearpage
\onecolumn
\appendices

\settocdepth{section}
\addtocontents{ptc}{\setcounter{tocdepth}{3}}

\LARGE \textbf{List of Appendix}
\normalsize
\startcontents[sections]
\printcontents[sections]{}{1}{}

\newpage

\section{Population Reformulation for OPAUC}\label{sec:op_pop_ver}

\begin{lem}\label{lem:topk}
$\sum_{i=1}^k x_{[i]}$ is a convex function of $(x_1,\cdots,x_n)$ where $x_{[i]}$ is its top-$i$ element. Furthermore, we have
$$\frac{1}{k}\sum_{i=1}^kx_{[i]}=\min_{s}\{s+\frac{1}{k}\sum_{i=1}^n[x_i-s]_+\},$$
where $[a]_+=\max\{0,a\}$.
For the population version, we have $$\E_{x}[x\cdot\indicator{x\geq \eta(\alpha)}]=\min_s \frac{1}{\alpha} \E_{x}[\alpha s+[x-s]_+],$$ where $\eta(\alpha)=\arg\min_{\eta\in\mathbb{R}}[\E_{x}[\indicator{x\geq \eta}]=\alpha]$.
\end{lem}

\textbf{Instance-wise Reformulation.}
With the above lemma, we can obtain
\begin{equation}\label{eq:minmaxopauc1}
    \begin{aligned}
    \underset{f}{\min}~ {\mathcal{R}}_{\beta}(f)
    \Leftrightarrow \underset{\cmin}{\min}\ \underset{\gamma\in[-1,1]}{\max} \
    \colblue{\ezdz} \left[F_{op}(f,a,b,\gamma, \colblue{\efb},\bm{z})\right],
    \end{aligned}
\end{equation}
where $\colblue{\efb}=\arg\min_{\colblue{\eta_{\beta}} \in\mathbb{R}}\left[\E_{\bm{x}'\sim \setD_{\setN}}[\indicator{f(\bm{x}')\geq\colblue{\eta_{\beta}}}]=\beta\right]$.

\textbf{Smoothing Approximation.}
Since
\begin{equation}\label{eq:atk}
  \begin{aligned}
    &\E_{\bm{x}'\sim\setD_\setN}[\indicator{f(\bm{x}')\geq\eta_{\beta}(f)}\cdot\ell_-(\bm{x}')] \\
    &= \min_{s'} \frac{1}{\beta}\cdot  \E_{\bm{x}'\sim\setD_\setN} [\beta s' + [\ell_-(\bm{x}')-s']_+],
  \end{aligned}
\end{equation}
we have
\begin{equation}\label{minmaxmin_op-pop}
\begin{aligned}
  &\underset{\cmin}{\min}\ \underset{\gamma\in[-1,1]}{\max} \
\colblue{\ezdz}[F_{op}(f,a,b,\gamma,\colblue{\efb}, \bm{z})]\\
&\Leftrightarrow \underset{\cmin}{\min}\
\underset{\cmax }{\max}
\ \underset{s'\in\Omega_{s'}}{\min}
\ \colblue{\ezdz}[G_{op}(f,a,b,\gamma,\bm{z},s')],
\end{aligned}
\end{equation}

With the softplus surrogate and regularization, the surrogate optimization problem becomes:
\begin{equation}\label{minmax_op}
  \begin{aligned}
    &\underset{\cmin}{\min}\ \underset{\gamma\in\Omega_{\gamma}}{\max} \min_{s' \in \Omega_{s'}}
    \ \colblue{\ezdz}[G_{op}^{{\colblue{\kappa}},\colbit{\omega}}] \\
    &\Leftrightarrow \underset{\cminl}{\min}\ \underset{\cmax}{\max}
\ \colblue{\ezdz}[G_{op}^{{\colblue{\kappa}},\colbit{\omega}}],
  \end{aligned}
\end{equation}

\textbf{Unbiased Reformulation.}
\begin{equation}\label{minmaxminmax_op-pop}
    \begin{aligned}
      & \underset{\cmin}{\min}\ \underset{\cmax }{\max}
    \ \underset{s'\in\Omega_{s'}}{\min}
    \ \colblue{\ezdz}[G_{op}(f,a,b,\gamma,\bm{z},s')], \\
      \Leftrightarrow&
      \underset{\cmin}{\min}\ \underset{\cmax }{\max}
    \ \underset{s'\in\Omega_{s'}}{\min} \underset{c}{\max}
    \ \colblue{\ezdz}[H_{op}(f,a,b,\gamma,\bm{z},s',c)], \\
        \Leftrightarrow& \underset{\cmin}{\min}\ \underset{\cmax, c}{\max} \ \underset{s'\in\Omega_{s'}}{\min} \colblue{\ezdz}[H_{op}], \\
    \Leftrightarrow& \underset{\cmin, s'\in\Omega_{s'}}{\min}\ \underset{\cmax, c}{\max} \ \colblue{\ezdz}[H_{op}].
    \end{aligned}
\end{equation}

\section{Reformulation for TPAUC}
\label{sec:tpauc_reformulation}
According to Eq.\eqref{TPAUCM}, given a surrogate loss $\ell$ and the finite dataset $S$, maximizing $\tp$ and $\hat{\mathrm{AUC}}_{\alpha,\beta}(f, S)$ is equivalent to solving the following problems, respectively:
\begin{equation}
        \underset{f}{\min}~  \mathcal{R}_{\alpha,\beta}(f)= \E_{\bm{x} \sim \setD_\setP, \bm{x}'\sim \setD_\setN} \left[\indicator{f(\bm{x}) \le \efa} \cdot \indicator{f(\bm{x}') \ge \efb} \cdot \ell(f(\bm{x})- f(\bm{x}'))\right],
        \label{TPAUCO}
    \end{equation}
\begin{equation}
\begin{aligned}
    \underset{f}{\min} \ \hat{\mathcal{R}}_{\alpha,\beta}(f, S)= \sum_{i=1}^{n_+^\alpha} \sum_{j=1}^{n_-^\beta}\frac{\ell{\left(f(\bm{x}_{[i]})- f(\bm{x}'_{[j]})\right)}}{n_+^\alpha n_-^\beta}.
    \label{TPAUCO}
\end{aligned}
\end{equation}
Similar to $\op$, we have the following theorem with an instance-wise reformulation of the $\tp$ optimization problem:
\begin{thm}
\label{thm:7}
Assuming that $f(\bm{x})\in[0,1]$, $\forall \bm{x}\in\mathcal{X}$, $F_{tp}(f,a,b,\gamma, t, t', \bm{z})$ is defined as:
\begin{equation}
    \begin{aligned}
F_{tp}(&f,a,b,\gamma,t, t', \bm{z})=(f(\bm{x})-a)^2y\indicator{f(\bm{x})\leq t}/(\alpha p)+
(f(\bm{x})-b)^2(1-y)\indicator{f(\bm{x}')\geq t'}/[\beta (1-p)]\\
&+2(1+\gamma)f(\bm{x})(1-y)\indicator{f(\bm{x}')\geq t'}/[\beta (1-p)] -
2(1+\gamma)f(\bm{x})y/p\indicator{f(\bm{x})\leq t}/(\alpha p) -\gamma^2,
\end{aligned}
\label{eq:fop}
\end{equation}
where $y=1$ for positive instances, $y=0$ for negative instances and we have the following conclusions:
\begin{enumerate}[leftmargin=20pt]
    \item[(a)] (\textbf{Population Version}.) We have:
    \begin{equation}
        \underset{f}{\min} \ {\mathcal{R}}_{\alpha,\beta}(f) \Leftrightarrow \underset{\cmin}{\min}\ \underset{\gamma\in[-1,1]}{\max} \
    \colblue{\ezdz}
    \left[F_{tp}(f,a,b,\gamma, \colblue{\efa},\colblue{\efb},\bm{z})\right],
    \label{eq:minmaxtpauc1}
    \end{equation}
    where $\colblue{\efa}=\arg\min_{\colblue{\eta_{\alpha}} \in\mathbb{R}}\left[\E_{\bm{x}\sim \setD_{\setP}}[\indicator{f(\bm{x})\leq\ \colblue{\eta_{\alpha}}}]=\alpha\right]$ and $\colblue{\efb}=\arg\min_{\colblue{\eta_{\beta}} \in\mathbb{R}}\left[\E_{\bm{x}'\sim \setD_{\setN}}[\indicator{f(\bm{x}')\geq\ \colblue{\eta_{\beta}}}]=\beta\right]$.
    \item[(b)] (\textbf{Empirical Version}.) Moreover, given a training dataset $S$ with sample size $n$, denote:
    \begin{equation*}
    \colbit{\ehat_{z \sim S}}[F_{tp}(f,a,b,\gamma,\colbit{\hat{\eta}_{\alpha}(f)}, \colbit{\hefb}, \bm{z})] = \frac{1}{n}\sum_{i=1}^n F_{tp}(f,a,b,\gamma,\colbit{\hat{\eta}_{\alpha}(f)}, {\colbit{\hefb}}, \bm{z}),
    \end{equation*}
    where $\colbit{\hat{\eta}_{\alpha}(f)}$ and $\colbit{\hefb}$ are the empirical quantile of the positive and negative instances in $S$, respectively. We have:
    \begin{equation}
        \underset{f}{\min} \ \hat{\mathcal{R}}_{\alpha,\beta}(f, S) \Leftrightarrow \underset{\cmin}{\min}\ \underset{\gamma\in[-1,1]}{\max} \
    \colbit{\hezds}
    \left[F_{tp}(f,a,b,\gamma, \colbit{\hat{\eta}_{\alpha}(f)}, \colbit{\hefb}, \bm{z} ) \right],
    \label{eq:minmaxtpauc2}
    \end{equation}
\end{enumerate}
\end{thm}

Thm.\ref{thm:7} provides a support to convert the pair-wise loss into instance-wise loss for $\tp$. Actually, for $\mathcal{R}_{\alpha, \beta}(f)$, we can just reformulate it as an  Average Top-$k$ (ATk) loss. Here we denote $P(f,a,\gamma,\bm{x})$ and $N(f,b,\gamma,\bm{x}')$ as $\ell_+(\bm{x})$ and $\ell_-(\bm{x}')$ for short respectively when $f,a,b,\gamma$ are not discussed. In the proof of the next theorem, we will show that $\ell_+(\bm{x})$ is a decreasing function and $\ell_-(\bm{x}')$ is an increasing function w.r.t. $f(\bm{x})$ and $f(\bm{x}')$, namely:
\begin{equation}
    \E_{\bm{x}\sim\setD_\setP}[\indicator{f(\bm{x})\leq\eta_{\alpha}(f)}\cdot\ell_+(\bm{x})]= \min_s \frac{1}{\alpha} \cdot \E_{\bm{x}\sim\setD_\setP} [\alpha s + [\ell_+(\bm{x})-s]_+],
\end{equation}
\begin{equation}
    \E_{\bm{x}'\sim\setD_\setN}[\indicator{f(\bm{x}')\geq\eta_{\beta}(f)}\cdot\ell_-(\bm{x}')]= \min_{s'}  \frac{1}{\beta} \cdot \E_{\bm{x}'\sim\setD_\setN} [\beta s' + [\ell_-(\bm{x}')-s']_+],
\end{equation}
The similar result holds for $\hat{\mathcal{R}}_{\alpha, \beta}(f, S)$. Then, we can reach to Thm.\ref{tpthm:step2}.
\begin{thm}\label{tpthm:step2}
Assuming that $f(\bm{x})\in[0,1]$, for all $\bm{x}\in\mathcal{X}$, we have the equivalent optimization for $\tp$:
\begin{equation}
\begin{aligned}
    \underset{\cmin}{\min}\ \underset{\gamma\in[-1,1]}{\max} \
\colblue{\ezdz}[F_{tp}(f,a,b,\gamma,\colblue{\efa}, \colblue{\efb}, \bm{z})]
\\ \Leftrightarrow \underset{\cmin}{\min}\
\underset{\cmax }{\max}
\ \underset{s\in\Omega_{s},s'\in\Omega_{s'}}{\min}
\ \colblue{\ezdz}[G_{tp}(f,a,b,\gamma,\bm{z},s,s')],
\label{minmaxmin_tp}
\end{aligned}
\end{equation}
\begin{equation}
    \begin{aligned}
        \underset{\cmin}{\min}\ \underset{\cmax}{\max} \
    \colbit{\hezds}[F_{tp}(f,a,b,\gamma, \colbit{\hefa}, \colbit{\hefb},\bm{z})]
    \\ \Leftrightarrow \underset{\cmin}{\min}\
    \underset{\gamma\in\Omega_{\gamma} }{\max}
    \ \underset{s\in\Omega_{s},s'\in\Omega_{s'}}{\min}
    \ \colbit{\hezds}[G_{tp}(f,a,b,\gamma,\bm{z},s,s')],
    \label{minmaxmin_tp}
    \end{aligned}
    \end{equation}
where $\Omega_{\gamma}=[\max\{b-1, -a\},1]$, $\Omega_{s}=[-4, 1]$, $\Omega_{s'}=[0,5]$ and
\begin{equation}
\begin{aligned}
G_{tp}(f,a,b,\gamma,&\bm{z},s, s')=\left(\alpha s + \left[(f(\bm{x})-a)^2-
2(1+\gamma)f(\bm{x})-s\right]_+\right)y/(\alpha p)\\
& +
\left(\beta s' +\left[(f(\bm{x})-b)^2+2(1+\gamma) f(\bm{x})-s'\right]_+\right)(1-y)/[\beta (1-p)] -\gamma^2.
\end{aligned}
\label{OPAUC_IB}
\end{equation}
\label{thm:8}
\end{thm}

Similar to $\op$, we can get a regularized non-convex strongly-concave $\tp$ optimization problem:
\begin{equation}
    \underset{\cmin}{\min}\ \underset{\gamma\in\Omega_{\gamma}}{\max} \min_{s \in \Omega_{s}, s' \in \Omega_{s'}}
    \ \colblue{\ezdz}[G_{tp}^{{\colblue{\kappa}},\colbit{\omega}}]\Leftrightarrow \underset{f,(a,b)\in[0,1]^2,s\in\Omega_{s},s'\in\Omega_{s'}}{\min}\ \underset{\cmax}{\max}
\ \colblue{\ezdz}[G_{tp}^{{\colblue{\kappa}},\colbit{\omega}}],
\label{minmax_tp}
\end{equation}

\begin{equation}
    \underset{\cmin}{\min}\ \underset{\gamma\in\Omega_{\gamma}}{\max} \min_{s \in \Omega_{s}, s'\in \Omega_{s'}}
    \ \colbit{\hezds}[{G}_{tp}^{{\colblue{\kappa}},\colbit{\omega}}]\Leftrightarrow \underset{f,(a,b)\in[0,1]^2,s\in\Omega_{s},s'\in\Omega_{s'}}{\min}\ \underset{\cmax}{\max}
\ \colbit{\hezds}[{G}_{tp}^{{\colblue{\kappa}},\colbit{\omega}}],
\end{equation}
where $G_{tp}^{{\colblue{\kappa}},\colbit{\omega}}=G_{tp}^{{\colblue{\kappa}},\colbit{\omega}}(f,a,b,\gamma,\bm{z},s,s')$.

\section{The Constrained Reformulation}

\label{val_constraintval}
In this section, we will prove that the constrained reformulation which is used in the proof of Thm.\ref{thm:instance_difftopk_op} and Thm.\ref{tpthm:step2}. Our proof can be established by Lem.\ref{lem:A}, Lem.\ref{lem:B}, and Thm.\ref{thm:A}. Throughout the proof, we will define:
\begin{equation}
\begin{aligned}
&a^* = \E_{\bm{x}\sim\setD_{\setP}}[f(\bm{x})] &:= E_+\\
&b^* = \E_{\bm{x}'\sim\setD_{\setN}}[f(\bm{x}')|f(\bm{x}')\ge \efb] &:=E_- \\
&b^* - a^* & := \Delta E\\
&\tilde{a}^* = \E_{\bm{x}\sim\setD_{\setP}}[f(\bm{x})|f(\bm{x})\le \efa] &:= \tilde{E}_+\\
&b^* - \tilde{a}^* & := \Delta \tilde{E}\\
& \E_{\bm{x}\sim\setD_\setP}[(f(\bm{x})-a)^2]& :={E}_a\\
& \E_{\bm{x}\sim\setD_\setP}[(f(\bm{x})-a)^2|
f(\bm{x})\leq\efa]& :=\tilde{E}_a\\
& \E_{\bm{x}'\sim\setD_{\setN}}[(f(\bm{x}')-b)^2|f(\bm{x}')\geq\eta_{\beta}(f)]  &:= E_b \\
&\E_{\bm{x}\sim\setD_\setP}[f(\bm{x})^2|
f(\bm{x})\leq\efa] &:=E_{+,2}\\
& \E_{\bm{x}'\sim\setD_{\setN}}[f(\bm{x}')^2|f(\bm{x}')\geq\eta_{\beta}(f)] &:=E_{-,2}
\end{aligned}
\end{equation}

\begin{lem}[The Reformulation for OPAUC]\label{lem:A}
For a fixed scoring function $f$ satisfying $f(\bm{x}) \in [0,1],~ \forall \bm{x}$, the following two problems share the same optimum:
\begin{equation}
    \begin{aligned}
    \boldsymbol{(OP1)} \min_{(a,b)\in[0,1]^2}\max_{\gamma \in [-1,1]}
\E_{\bm{x}\sim\setD_\setP}[(f(\bm{x})-a)^2] +
\E_{\bm{x}'\sim\setD_\setN}[(f(\bm{x}')-b)^2|
f(\bm{x}')\geq\efb] \\
+2\Delta E + 2\gamma \Delta E-\gamma^2,
\\
\boldsymbol{(OP2)} \min_{(a,b)\in[0,1]^2}\max_{\gamma \in [b-1,1]}
\E_{\bm{x}\sim\setD_\setP}[(f(\bm{x})-a)^2] +
\E_{\bm{x}'\sim\setD_\setN}[(f(\bm{x}')-b)^2|
f(\bm{x}')\geq\efb] \\
+2\Delta E + 2\gamma \Delta E-\gamma^2.
    \end{aligned}
\end{equation}
\end{lem}

\begin{rem}
$\boldsymbol{(OP1)}$ and $\boldsymbol{(OP2)}$ have the equivalent formulation:
\begin{equation}
\begin{aligned}
\boldsymbol{(OP1)}\Leftrightarrow \min_{(a,b)\in[0,1]^2}&\max_{\gamma \in [-1,1]} \E_{\bm{z}\sim\setD_{\mathcal{Z}}}\Big[\Large[(f(\bm{x})-a)^2-
2(1+\gamma)f(\bm{x})\Large]y/p-\gamma^2\\
&+\Large[(f(\bm{x})-b)^2+2(1+\gamma)f(\bm{x})\Large]\cdot[(1-y)\indicator{f(\bm{x})\geq \efb}]/[(1-p)\beta]\big], \\
\boldsymbol{(OP2)}\Leftrightarrow \min_{(a,b)\in[0,1]^2}&\max_{\gamma \in [b-1,1]} \E_{\bm{z}\sim\setD_{\mathcal{Z}}}\Big[\Large[(f(\bm{x})-a)^2-
2(1+\gamma)f(\bm{x})\Large]y/p-\gamma^2\\
&+\Large[(f(\bm{x})-b)^2+2(1+\gamma)f(\bm{x})\Large]\cdot[(1-y)\indicator{f(\bm{x})\geq \efb}]/[(1-p)\beta]\big].
\end{aligned}
\end{equation}
\end{rem}

\begin{proof}
From the proof of our main paper, we know that $\boldsymbol{(OP1)}$ has a closed-form minimum:
\begin{equation}
     E_{a^*} + {E}_{b^*} + (\Delta E)^2 + 2 \Delta E.
\end{equation}
Hence, we only need to prove that $\boldsymbol{(OP2)}$ has the same minimum solution. By expanding $\boldsymbol{(OP2)}$, we have:
\begin{equation}
\min_{(a,b)\in[0,1]^2}\max_{\gamma\in[b-1,1]}
\E_{\bm{z}\sim\setD_{\mathcal{Z}}}[
F_{op}(f,a,b,\gamma,\efb,\bm{z})] = 2\Delta E + \min_{a \in [0,1]}E_a  + \min_{b\in[0,1]}\max_{\gamma \in [b-1,1]} F_0,
\end{equation}
where
\begin{equation}
    F_0:= E_b + 2\gamma\Delta E - \gamma^2.
\end{equation}
Obviously since $a$ is decoupled with $b,\gamma$, we have:
\begin{equation}
    \min_{a \in [0,1]} E_a  = E_{a^*}.
\end{equation}
Now, we solve the minimax problem of $F_0$. For any fixed feasible $b$, the inner max problem is a truncated quadratic programming, which has a unique and closed-form solution. Hence, we first solve the inner maximization problem for fixed $b$, and then represent the minimax problem as a minimization problem for $b$. Specifically, we have:
\begin{equation}
    \left(\max_{\gamma\in[b-1,1]} 2\gamma \Delta E-\gamma^2\right) =
\begin{cases}
(\Delta E)^2,  & \Delta E \ge b-1 \\
2(b-1)\Delta E - (b-1)^2, & \text{otherwise}
\end{cases}.
\end{equation}

Thus, we have:
\begin{equation}
    \min_{b\in[0,1]} \max_{\gamma\in[b-1,1]} F_0=
\min_{b \in [0,1]} F_1,
\end{equation}
where
\begin{equation}
    F_1 = \begin{cases}
&F_{1,0}(b) := E_b + (\Delta E)^2,  b-1 \le \Delta E \\
& F_{1,1}(b) := E_{-,2}- 2bE_- + 2b-1 + 2(b-1)\Delta E,  \text{otherwise}
\end{cases}.
\end{equation}
It is easy to see that both cases of $F_1$ are convex functions \wrt $b$. So, we can find the global minimum by comparing the minimum of $F_{1,0}$ and $F_{1,1}$.
\begin{itemize}
\item CASE 1: $\Delta E \ge b-1$.
    It is easy to see that $b^* = E_- \in (-\infty, 1+\Delta E]$, by taking the derivative to zero, we have, the optimum value is obtained at $b= E_-$ for $F_{1,0}$.

\item CASE 2: $\Delta E \leq b-1$.
    Again by taking the derivative, we have:
    \begin{equation}
        F_{1,1}(b)' = -2E_- + 2 + 2 \Delta E = 2-2E_+ \ge 0.
    \end{equation}
    We must have:
    \begin{equation}
        \inf_{b \geq 1 + \Delta E} F_{1,1}(b) ~\ge~ F_{1,1}(1+\Delta E) ~=~ F_{1,0}(1+\Delta E) ~\geq~ F_{1,0}(E_-) ~=~ F_{1,0}(b^*).
    \end{equation}

\item Putting all together
Hence the global minimum of $F_1$ is obtained at $b^*$ with:
\begin{equation}
    F_1(b^*) = F_{1,0}(b^*) = E_{b^*}+ (\Delta E)^2.
\end{equation}
\end{itemize}

Hence, we have $(OP2)$ has the minimum value:
\begin{equation}
     E_{a^*} + E_{b^*} + (\Delta E)^2 + 2 \Delta E.
\end{equation}
\end{proof}

Now, we use a similar trick to prove the result for TPAUC:

\begin{lem}[The Reformulation for TPAUC]\label{lem:B}
    For a fixed scoring function $f$, the following two problems shares the same optimum, given that the scoring function satisfies: $f(\bm{x}) \in [0,1],~ \forall \bm{x}$:
\begin{equation}
    \begin{aligned}
    \boldsymbol{(OP3)} \min_{f,(a,b)\in[0,1]^2}\max_{\gamma \in [-1,1]}
&\E_{\bm{x}\sim\setD_\setP}[(f(\bm{x})-a)^2|
f(\bm{x})\leq\efa]\\+
&\E_{\bm{x}'\sim\setD_\setN}[(f(\bm{x}')-b)^2|
f(\bm{x}')\geq\efb]
+2\Delta \tilde{E} + 2\gamma \Delta \tilde{E}-\gamma^2,
\end{aligned}
\end{equation}
\begin{equation}
\begin{aligned}
\boldsymbol{(OP4)} \min_{f,(a,b)\in[0,1]^2}\max_{\gamma \in [\max\{-a,b-1\},1]}
&\E_{\bm{x}\sim\setD_\setP}[(f(\bm{x})-a)^2|
f(\bm{x})\leq\efa]\\+
&\E_{\bm{x}'\sim\setD_\setN}[(f(\bm{x}')-b)^2|
f(\bm{x}')\geq\efb]
 +2\Delta \tilde{E} + 2\gamma \Delta \tilde{E}-\gamma^2.
\end{aligned}
\end{equation}
\end{lem}

\begin{rem}
$\boldsymbol{(OP3)}$ and $\boldsymbol{(OP4)}$ have the equivalent formulation:
\begin{equation}
   \begin{aligned}
\boldsymbol{(OP3)} \Leftrightarrow
\min_{(a,b)\in[0,1]^2}&\max_{\gamma \in [-1, 1]} \E_{\bm{z}\sim \setD_\mathcal{Z}}\Big[\Large[(f(\bm{x})-a)^2-
2(1+\gamma)f(\bm{x})\Large]\cdot[y\indicator{f(\bm{x})\leq \efa}]/p-\gamma^2\\
&+\Large[(f(\bm{x})-b)^2+2(1+\gamma)f(\bm{x})\Large]\cdot[(1-y)\indicator{f(\bm{x})\geq \efb}]/[(1-p)\beta]\Big],
\end{aligned}
\end{equation}
\begin{equation}
\begin{aligned}
\boldsymbol{(OP4)} \Leftrightarrow
\min_{(a,b)\in[0,1]^2}&\max_{\gamma \in [\max\{-a,b-1\} 1]} \E_{\bm{z}\sim \setD_\mathcal{Z}}\Big[\Large[(f(\bm{x})-a)^2-
2(1+\gamma)f(\bm{x})\Large]\cdot[y\indicator{f(\bm{x})\leq \efa}]/p\\
&+\Large[(f(\bm{x})-b)^2+2(1+\gamma)f(\bm{x})\Large]\cdot[(1-y)\indicator{f(\bm{x})\geq \efb}]/[(1-p)\beta]-\gamma^2\Big].
\end{aligned}
\end{equation}
\end{rem}

\begin{proof}
Again, $\boldsymbol{(OP3)}$ has the minimum value:
\begin{equation}
    \tilde{E}_{\tilde{a}^*} + E_{b^*} + (\Delta \tilde{E})^2 + 2 \Delta \tilde{E}.
\end{equation}

We prove that $\boldsymbol{(OP4)}$ ends up with the minimum value. By expanding $\boldsymbol{(OP4)}$, we have:
\begin{equation}
    \begin{aligned}
(OP4)= 2\Delta \tilde{E} + \min_{(a,b)\in[0,1]^2}\max_{\gamma \in [\max\{-a,b-1\},1]} F_3,
\end{aligned}
\end{equation}
where
\begin{equation}
    F_3:=  \tilde{E}_a+ E_{b} + 2\Delta \tilde{E} + 2\gamma \Delta \tilde{E}-\gamma^2.
\end{equation}

For any fixed feasible $a,b$, the inner max problem is a truncated quadratic programming, which has a unique and closed-form solution. Specifically, define $c = \max\{-a,b-1\}$, we have:
\begin{equation}
   \left(\max_{\gamma \in [c,1]} 2\gamma \Delta \tilde{E}-\gamma^2\right) =
\begin{cases}
(\Delta \tilde{E})^2,  & \Delta \tilde{E} \ge c \\
2c\Delta \tilde{E} - c^2, & \text{otherwise}
\end{cases}.
\end{equation}
Thus, we have:
\begin{equation}
    \min_{(a,b)\in[0,1]^2} \max_{\gamma\in[c,1]} F_3=
\min_{(a,b) \in [0,1]} F_4,
\end{equation}
where
\begin{equation}
    F_4 = \begin{cases}
&F_{4,0}(a,b) := \tilde{E}_a + E_b + (\Delta \tilde{E})^2,  c \le \Delta \tilde{E} \\
& F_{4,1}(a,b) := \tilde{E}_a + E_{-,2} -2b E_- + 2(b-1) \Delta \tilde{E} + 2b -1 ,  b-1 \ge \Delta \tilde{E}, -a \le b-1 \\
& F_{4,2}(a,b) := E_b + E_{+,2}  -2a \tilde{E}_+ -2a \Delta \tilde{E},   -a \ge \Delta \tilde{E}, b-1 \le -a
\end{cases}.
\end{equation}
It is easy to see that both cases of $F_1$ are convex functions w.r.t $b$. So, we can find the global minimum by comparing the minimum of $F_{1,0}$ and $F_{1,1}$.
\begin{itemize}
    \item CASE 1: $\Delta \tilde{E} \ge \max\{-a,b-1\}$.

It is easy to check that when $a = \tilde{E}_+, b = E_- $, we have $-a \le \Delta \tilde{E}$ and $b-1 \le \Delta \tilde{E}$. It is easy to see that $a,b$ are decoupled in the expression of $F_{4,0}(a,b)$. By setting:
\begin{equation}
     \begin{aligned}
\frac{\partial F_{4,0}(a,b)}{\partial a} = 0, \\
\frac{\partial F_{4,0}(a,b)}{\partial b} = 0.
\end{aligned}
\end{equation}
We know that the minimum solution is attained at $a= \tilde{a}^*$, $b= b^*$. Then the minimum value of $F_{4,0}(a,b)$ at this range becomes:
\begin{equation}
    \tilde{E}_{\tilde{a}^*} + E_{b^*} + (\Delta \tilde{E})^2.
\end{equation}
Moreover, we will also use the fact that $E_{\tilde{a}^*}$ and $E_{b^*}$ are also the global minimum for $E_a$ and $E_b$, respectively.

    \item CASE 2: $b-1 \ge \Delta \tilde{E},~ -a \le b-1$.

It is easy to see that $E_a \ge E_{\tilde{a}^*}$ in this case.  According to the same derivation as in Lem.\ref{lem:A} CASE 2, we have:
\begin{equation}
    E_{-,2} -2b E_- + 2(b-1) \Delta \tilde{E} + 2b -1 \ge E_{b^*} + (\Delta \tilde{E})^2
\end{equation}
holds when $b -1 \geq \Delta \tilde{E}$. Recall that CASE 2 is include in the condition $b -1 \geq\Delta \tilde{E}$. So, under the condition of CASE 2:
\begin{equation}
    F_{4,1}(a,b) \ge \tilde{E}_{\tilde{a}^*} + E_{b^*} + (\Delta \tilde{E})^2.
\end{equation}

    \item CASE 3: $-a\geq \Delta \tilde{E}, b-1\leq -a$.

In this case, we have $E_b \ge E_{b^*}$. It remains to check:
\begin{equation}
    g(a) =  -2a \tilde{E}_+ -2a \Delta \tilde{E}.
\end{equation}
By taking derivative, we have:
\begin{equation}
    g'(a) =  -2\tilde{E}_+ - 2\Delta \tilde{E} = -2\tilde{E}_- \le 0.
\end{equation}
Similar as the proof of CASE 2, when $-a \ge \Delta \tilde{E}$, we have:
\begin{equation}
    g(a) \ge \tilde{E}_{\tilde{a}^*} + (\Delta \tilde{E})^2,
\end{equation}
 and thus
\begin{equation}
    F_{4,2}(a,b) \ge \tilde{E}_{\tilde{a}^*} + E_{b^*} + (\Delta \tilde{E})^2
\end{equation}
holds. Since the condition of CASE 3 is included in the set $-a \ge \Delta \tilde{E}$:
\begin{equation}
    F_{4,2}(a,b) \ge \tilde{E}_{\tilde{a}^*} + E_{b^*} + (\Delta \tilde{E})^2
\end{equation}
holds under the condition of CASE 3.
    \item Putting altogether:
The minimum value of $\boldsymbol{(OP4)}$ reads:
\begin{equation}
    \tilde{E}_{\tilde{a}^*} + E_{b^*} + (\Delta \tilde{E})^2 + 2\Delta \tilde{E},
\end{equation}
which is the same as $\boldsymbol{(OP3)}$.
\end{itemize}
\end{proof}

Finally, since for each fixed $f$ $(OP3) = (OP4)$, and $(OP1) = (OP2)$. We can then claim the following theorem.

\begin{thm}[Constrainted Reformulation]\label{thm:A}
\label{Constrainted_Reformulation}
\begin{equation}
    \min_f (OP1) = \min_f (OP2), ~~ \min_f (OP3) = \min_f (OP4)
\end{equation}
\end{thm}

\begin{rem}
Since the calculation is irrelevant to the definition of the expectation, the replace the population-level expectation with the empirical expectation over the training data.
\end{rem}

\begin{rem}
By applying it on Thm.\ref{thm:instance_op}, we can get the reformulation result in Thm.\ref{thm:instance_difftopk_op}
\item  for $\op$
\begin{equation}
    \min_{f,(a,b)\in[0,1]^2} \max_{\gamma\in [b-1,1]}\min_{s'\in\Omega_{s'}}\E_{\bm{z}\sim \setD_\mathcal{Z}} [G_{op}(f,a,b,\gamma,\bm{z},s')]
\end{equation}
where
\begin{equation}
    \begin{aligned}
    G_{op}(f,a,b,\gamma,\bm{z},s')&=[(f(\bm{x})-a)^2-
    2(1+\gamma)f(\bm{x})]y/p-\gamma^2\\
    & +
    \left(\beta s' +\left[(f(\bm{x})-b)^2+2(1+\gamma) f(\bm{x})-s'\right]_+\right)(1-y)/[\beta (1-p)].
    \end{aligned}
\end{equation}
\item for $\tp$
\begin{equation}
    \min_{f,(a,b)\in[0,1]^2 } \max_{\gamma\in [\max\{-a,b-1\}, 1]}\min_{s\in\Omega_{s},s'\in\Omega_{s'}}\E_{\bm{z}\sim \setD_\mathcal{Z}} [G_{tp}(f,a,b,\gamma,\bm{z},s,s')]
\end{equation}
where
\begin{equation}
    \begin{aligned}
    G_{tp}(f,a,b,\gamma,\bm{z},s,s')&=\left(\alpha s + \left[(f(\bm{x})-a)^2-
    2(1+\gamma)f(\bm{x})-s\right]_+\right)y/(\alpha p) -\gamma^2\\
    & +
    \left(\beta s' + \left[(f(\bm{x})-b)^2+2(1+\gamma) f(\bm{x})-s'\right]_+\right)(1-y)/[\beta (1-p)].
    \end{aligned}
\end{equation}
\end{rem}

\section{Proofs for Section 4}
\label{sec:proofs_section4}

\subsection{Proof for Lem.~\ref{lem:topk}}\label{sec:proof_topk}
\begin{proof}
For the summation case, please see Lem.~1 in \cite{fan2017learning} for the proof.  We only prove the expectation case here. Specifically, calculating the sub-differential of the term $\E_{x}[\alpha s+[x-s]_+]$ \wrt $s$, we get:
\begin{equation}
\begin{aligned}
& \alpha-\E_x[\indicator{x\geq s}] \in \partial \left(\E_{x}[\alpha s +[x-s]_+]\right)\\
\end{aligned}
\end{equation}
Since $s$ is convex for $\alpha s +\left[x-s\right]_+$, we can get the optimal $s$ by letting it be 0:
\begin{equation}
    \E_x[\indicator{x\geq s}]=\alpha
\end{equation}
It is clear that optimal $s$ achieves top-$\alpha$ quantile.
\end{proof}

\subsection{Proofs for OPAUC}
\subsubsection{Step 1 (\thmref{thm:instance_op})} \label{sec:proof_instance_op}
\textbf{Restate of Theorem \ref{thm:instance_op}.}
\emph{Assuming that $f(\bm{x})\in[0,1]$, $\forall \bm{x}\in\mathcal{X}$, $F_{op}(f,a,b,\gamma, t, \bm{z})$ is defined as:}

\begin{equation}
    \begin{aligned}
F_{op}(f,a,b,\gamma,t, \bm{z})=&\Large[(f(\bm{x})-a)^2-
2(1+\gamma)f(\bm{x})\Large]y/p-\gamma^2\\
&\Large[(f(\bm{x})-b)^2+2(1+\gamma)f(\bm{x})\Large](1-y)\indicator{f(\bm{x})\geq t}/(1-p)/\beta,
\end{aligned}
\label{eq:fop}
\end{equation}

\emph{where $y=1$ for positive instances, $y=0$ for negative instances and we have the following conclusions:
\begin{enumerate}[leftmargin=20pt]
    \item[(a)] (\textbf{Population Version}.) We have:
    \begin{equation}
        \underset{f}{\min} \ {\mathcal{R}}_{\beta}(f) \Leftrightarrow \underset{\cmin}{\min}\ \underset{\gamma\in[-1,1]}{\max} \
    \colblue{\ezdz}
    \left[F_{op}(f,a,b,\gamma, \colblue{\efb},\bm{z})\right],
    \label{eq:minmaxopauc1}
    \end{equation}
     where  $\colblue{\efb}=\arg\min_{\colblue{\eta_{\beta}} \in\mathbb{R}}\E_{\bm{x}'\sim \setD_{\setN}}[\indicator{f(\bm{x}')\geq\ \colblue{\eta_{\beta}}}=\beta]$.
    \item[(b)] (\textbf{Empirical Version}.) Moreover, given a training dataset $S$ with sample size $n$, denote:
    \begin{equation*}
    \colbit{\underset{\bm{z} \sim S}{\ehat}}[F_{op}(f,a,b,\gamma,\colbit{\hefb}, \bm{z})] = \frac{1}{n}\sum_{i=1}^n F_{op}(f,a,b,\gamma,{\colbit{\hefb}}, \bm{z}_i),
    \end{equation*}
    where $\colbit{\hefb}$ is the empirical quantile of the negative instances in $S$. We have:
    \begin{equation}
        \underset{f}{\min} \ \hat{\mathcal{R}}_{\beta}(f, S) \Leftrightarrow \underset{\cmin}{\min}\ \underset{\cmax}{\max} \
    \colbit{\hezds}
    \left[F_{op}(f,a,b,\gamma, \colbit{\hefb}, \bm{z} ) \right],
    \end{equation}
\end{enumerate}
}

\begin{proof}
Firstly, we give a reformulation of $\op$:
\begin{equation}
\begin{aligned}
    \underset{f}{\min} \ {\mathcal{R}}_{\beta}(f) &=\underset{f}{\min} \ \E_{\bm{x} \sim \setD_\setP, \bm{x}'\sim \setD_\setN} \left[\indicator{f(\bm{x}') \ge \efb} \cdot \ell(f(\bm{x})- f(\bm{x}'))\right]\\
    &=\underset{f}{\min}\ \E_{\bm{x} \sim \setD_\setP, \bm{x}'\sim \setD_\setN} \left[\ell(f(\bm{x})- f(\bm{x}'))|f(\bm{x}') \ge \efb\right]\cdot \underset{\bm{x}'\sim\setD_{\setN}}{\mathbb{P}}[f(\bm{x}') \ge \efb]\\
    &=\underset{f}{\min}\ \E_{\bm{x} \sim \setD_\setP, \bm{x}'\sim \setD_\setN} \left[\ell(f(\bm{x})- f(\bm{x}'))|f(\bm{x}') \ge \efb\right]\cdot\beta\\
    &=\beta\cdot\underset{f}{\min}\ \E_{\bm{x} \sim \setD_\setP, \bm{x}'\sim \setD_\setN} \left[\ell(f(\bm{x})- f(\bm{x}'))|f(\bm{x}') \ge \efb\right].
\end{aligned}
\end{equation}

Applying the surrogate loss $(1-x)^2$ to the estimator of $\op$, we have:
\begin{equation}
\begin{aligned}
&\underset{\bm{x},\bm{x}'\sim\setD_{\setP}, \setD_{\setN}}
{\E}[(1-(f(\bm{x})-f(\bm{x}')))^2|f(\bm{x}')\geq\efb]\\
&=1+\exdp[f(\bm{x})^2]
+\exdn[f(\bm{x}')^2|f(\bm{x}')\geq
\eta_{\beta}(f)] -2\exdp[f(\bm{x})]\\
& \qquad +2\exdn[f(\bm{x}')|f(\bm{x}')\geq
\eta_{\beta}(f)] -2\underset{\bm{x}\sim \setD_\setP}{\E}[f(\bm{x})]
\underset{\bm{x}'\sim \setD_\setN}{\E}[f(\bm{x}')|f(\bm{x}')\geq
\eta_{\beta}(f)]\\
&=1+\underset{\bm{x}\sim \setD_\setP}{\E}[f(\bm{x})^2]
-\underset{\bm{x}\sim \setD_\setP}{\E}[f(\bm{x})]^2+\underset{\bm{x}'\sim \setD_\setN}{\E}[f(\bm{x}')^2|f(\bm{x}') \ge \efb] \\
& \qquad - \underset{\bm{x}'\sim \setD_\setN}{\E}[f(\bm{x}')^2|f(\bm{x}') \ge \efb]^2-2\underset{\bm{x}\sim \setD_\setP}{\E}[f(\bm{x})]
+2\underset{\bm{x}'\sim \setD_\setN}{\E}[f(\bm{x}')|f(\bm{x}') \ge \efb] \\
& \qquad + (\underset{\bm{x}\sim \setD_\setP}{\E}[f(\bm{x})]
-\underset{\bm{x}'\sim \setD_\setN}{\E}[f(\bm{x}')|f(\bm{x}') \ge \efb])^2.
\end{aligned}
\end{equation}
Note that
\begin{equation}
   \underset{\bm{x}\sim \setD_\setP}{\E}[f(\bm{x})^2]-\underset{\bm{x}\sim \setD_\setP}{\E}[f(\bm{x})]^2=\min_{a\in[0,1]}\underset{\bm{x}\sim \setD_\setP}{\E}[(f(\bm{x})-a)^2],
\end{equation}
where the minimization is achieved by:
\begin{equation}
    a^* = \underset{\bm{x}\sim \setD_\setP}{\E}[f(\bm{x})],
\end{equation}
where $a^*\in[0, 1]$. Likewise,
\begin{equation}
\begin{aligned}
    &\underset{\bm{x}'\sim \setD_\setN}{\E}[f(\bm{x}')^2|f(\bm{x}') \ge \efb] - \underset{\bm{x}'\sim \setD_\setN}{\E}[f(\bm{x}')|f(\bm{x}') \ge \efb]^2=\\
    &\qquad \qquad \qquad \min_{b\in[0,1]}\underset{\bm{x}'\sim \setD_\setN}{\E}[(f(\bm{x}')-b)^2|f(\bm{x}') \ge \efb],
\end{aligned}
\end{equation}
where the minimization is get by:
\begin{equation}
    b^* = \underset{\bm{x}'\sim \setD_\setN}{\E}[f(\bm{x}')|f(\bm{x}') \ge \efb].
\end{equation}
where $b^*\in[0, 1]$. It's notable that
\begin{equation}
\begin{aligned}
    &\left(\underset{\bm{x}'\sim \setD_\setN}{\E}[f(\bm{x}')|f(\bm{x}') \ge \efb]-\underset{\bm{x}\sim \setD_\setP}{\E}[f(\bm{x})]\right)^2=\\
&\max_{\gamma}\left\{2\gamma\left(\underset{\bm{x}'\sim \setD_\setN}{\E}[f(\bm{x}')|f(\bm{x}') \ge \efb]-\underset{\bm{x}\sim \setD_\setP}{\E}[f(\bm{x})]\right)-\gamma^2\right\},
\end{aligned}
\end{equation}
where the maximization can be obtained by:
\begin{equation}
    \gamma^* = \underset{\bm{x}'\sim \setD_\setN}{\E}[f(\bm{x}')|f(\bm{x}') \ge \efb]-\underset{\bm{x}\sim \setD_\setP}{\E}[f(\bm{x})].
\end{equation}
It's clear that $\gamma^*=b^*-a^*$. Then we can constraint $\gamma$ with range $[-1, 1]$ and get the equivalent optimization formulation:
\begin{equation}
\begin{aligned}
    &\underset{\bm{x},\bm{x}'\sim\setD_{\setP}, \setD_{\setN}}
{\E}[(1-(f(\bm{x})-f(\bm{x}')))^2|f(\bm{x}')\geq\efb]\Leftrightarrow \\
&\min_{(a,b)\in[0,1]^2}\max_{\gamma\in[-1,1]} \exdp [(f(\bm{x})-a)^2-2(\gamma+1)f(\bm{x})]  -\gamma^2 \\
&+ \exdn [(f(\bm{x}')-b)^2+2(\gamma+1)f(\bm{x}')|f(\bm{x}')\geq \eta_{\beta}(f)].
\end{aligned}
\end{equation}
Taking expectation $\textit{w.r.t.}$, $\bm{z}$, we have:
\begin{equation}
    \begin{aligned}
\underset{f}{\min} \ \mathcal{R}_{\beta}(f)
\Leftrightarrow \underset{f,a,b}{\min}\ \underset{\gamma\in[-1, 1]}{\max} \
\colblue{\ezdz}
[F_{op}(f,a,b,\gamma,\eta_{\beta}(f),\bm{z})],
\end{aligned}
\end{equation}
and the instance-wise function $F_{op}(f,a,b,\gamma,\eta_{\beta}(f),\bm{z})$ is defined by:

\begin{equation}
    \begin{aligned}
F_{op}(f,a,b,\gamma,t, \bm{z})=&\Large[(f(\bm{x})-a)^2-
2(1+\gamma)f(\bm{x})\Large]y/p-\gamma^2\\
&\Large[(f(\bm{x})-b)^2+2(1+\gamma)f(\bm{x})\Large](1-y)\indicator{f(\bm{x})\geq t}/(1-p)/\beta,
\end{aligned}
\end{equation}

where $p=\Pr[y=1]$. The same result holds for empirical version $\colbit{\underset{\bm{z}\sim S}{\hat{\E}}}
[F_{op}(f,a,b,\gamma,\hat{\eta}_{\beta}(f),\bm{z})]$.
\end{proof}

\subsubsection{Step 2 (\thmref{thm:instance_difftopk_op})}\label{sec:proof_instance_difftopk_op}
First we need the following proposition to complete the proof in this subsection.
\begin{prop}\label{prop1}
If $\gamma \in \Omega_{\gamma} =[b-1, 1]$, $\ell_{-}(\bm{x}')=(f(\bm{x}')-b)^2+2(1+\gamma) f(\bm{x}')$ is an increasing function \wrt $f(\bm{x}')$ when $\bm{x}'\sim\setD_\setN$ and $f(\bm{x}')\in[0,1]$.
\end{prop}

\begin{proof}
We have:
\begin{equation}
    \frac{\partial \ell_{-}(\bm{x}')}{\partial f(\bm{x}')}=2(f(\bm{x}')-b+1+\gamma).
\end{equation}
Assuming that $f(\bm{x}')\in[0,1]$, then the feasible solution of $b$ is nonnegative. When $\gamma \in [b-1,1]$, the negative loss function's partial derivative $\partial\ell_{-}(\bm{x}')/\partial f(\bm{x}')\geq 0$. Then $\ell_{-}(\bm{x}')$ is an increasing function \wrt $f(\bm{x}')$.
\end{proof}

\begin{rem}
For negative instances, if the loss function is an increasing function \wrt the score $f(\bm{x}')$, then the top-ranked losses are equivalent to the losses of top-ranked instances.
\end{rem}
\noindent \textbf{Restate of Theorem \ref{thm:instance_difftopk_op}.}
\emph{Assuming that $f(\bm{x})\in[0,1]$, for all $\bm{x}\in\mathcal{X}$, we have the equivalent optimization for $\op$: }
\begin{equation}
\begin{aligned}
    \underset{\cmin}{\min}\ \underset{\gamma \in [-1,1]}{\max} \
\colblue{\ezdz}[F_{op}(f,a,b,\gamma,\colblue{\efb}, \bm{z})]
&\Leftrightarrow \underset{\cmin}{\min}\
\underset{\cmax }{\max}
\ \underset{s'\in\Omega_{s'}}{\min}
\ \colblue{\ezdz}[G_{op}(f,a,b,\gamma,\bm{z},s')],
\end{aligned}
\end{equation}
\begin{equation}
    \begin{aligned}
        \underset{\cmin}{\min}\ \underset{\gamma\in [-1,1]}{\max} \
    \colbit{\hezds}[F_{op}(f,a,b,\gamma, \colbit{\hefb},\bm{z})]
    &\Leftrightarrow \underset{\cmin}{\min}\
    \underset{\gamma\in\Omega_{\gamma} }{\max}
    \ \underset{s'\in\Omega_{s'}}{\min}
    \ \colbit{\hezds}[G_{op}(f,a,b,\gamma,\bm{z},s')],
\end{aligned}
\end{equation}

\emph{where $\Omega_{\gamma}=[b-1,1]$, $\Omega_{s'}=[0,5]$ and}

\begin{equation}
    \begin{aligned}
G_{op}(f,a,b,\gamma,\bm{z},s')&=[(f(\bm{x})-a)^2-
2(1+\gamma)f(\bm{x})]y/p-\gamma^2\\
& +
\left(\beta s' +\left[(f(\bm{x})-b)^2+2(1+\gamma) f(\bm{x})-s'\right]_+\right)(1-y)/[\beta(1-p)].
\end{aligned}
\end{equation}

\begin{proof}


According to the Thm.\ref{Constrainted_Reformulation} in Appx.\ref{val_constraintval},
when we constraint $\gamma$ in range $\Omega_{\gamma}=[b-1, 1]$, we have:

\begin{equation}
    \min_{f,(a,b)\in[0,1]^2}\max_{\gamma\in [-1,1]}\colblue{\E_{z\sim\setD_{\mathcal{Z}}}} [F_{op}]
    \Leftrightarrow
    \min_{f,(a,b)\in[0,1]^2}\max_{\gamma\in [b-1,1]}\colblue{\E_{z\sim\setD_{\mathcal{Z}}}} [F_{op}]
\end{equation}

According to Thm.\ref{thm:instance_op}, we have:

\begin{equation}
    \begin{aligned}
\underset{\bm{z}\sim \setD_\mathcal{Z}}{\E}
[F_{op}(f,a,b,\gamma,\colblue{\eta_{\beta}(f)},&\bm{z})]\Leftrightarrow
\underset{\bm{x}\sim \setD_\setP}{\E}[(f(\bm{x})-a)^2-
2(1+\gamma)f(\bm{x})] -\gamma^2\\
& +\underset{\bm{x}'\sim \setD_\setN}{\E}
\left([(f(\bm{x}')-b)^2+2(1+\gamma)f(\bm{x}')]\cdot \indicator{f(\bm{x})\geq \eta_{\beta}(f)}\right)/\beta.
\end{aligned}
\end{equation}

We denote $\ell_-(\bm{x}')=(f(\bm{x}')-b)^2+2(1+\gamma)f(\bm{x}')$. Prop.\ref{prop1} ensures that the negative loss function $\ell_-(\bm{x}')$ is an increasing function when $\gamma\in[b-1,1]$. Then we can get:

\begin{equation}
    \E_{\bm{x}'\sim\setD_\setN}[\indicator{f(\bm{x}')\geq\eta_{\beta}(f)}\cdot\ell_-(\bm{x}')]= \min_s \frac{1}{\beta} \E_{\bm{x}'\sim\setD_\setN} [\beta s + [\ell_-(\bm{x}')-s]_+],
\end{equation}

Applying Lem.\ref{lem:topk} to the negative loss function, we have:
\begin{equation}
    \begin{aligned}
\colblue{\ezdz}
[F_{op}(f,a,b,\gamma,\colblue{\eta_{\beta}(f)},\bm{z})]&= \underset{s'}{\min}
\underset{\bm{x}\sim \setD_\setP}{\E}[(f(\bm{x})-a)^2-
2(1+\gamma)f(\bm{x})] -\gamma^2\\
&\quad +\underset{\bm{x}'\sim \setD_\setN}{\E}
\left(\beta s' +\left[(f(\bm{x}')-b)^2+2(1+\gamma)f(\bm{x}')-s'\right]_+\right)/\beta.
\end{aligned}
\end{equation}

Then, we get:
\begin{equation}
\begin{aligned}
\colblue{\ezdz}[F_{op}(f,a,b,\gamma,\colblue{\eta_{\beta}(f)},\bm{z})]
&=
\ \underset{s'\in\Omega_{s'}}{\min}
\ \colblue{\ezdz}[G_{op}(f,a,b,\gamma,\bm{z},s')],
\end{aligned}
\end{equation}
where
\begin{equation}
    \begin{aligned}
G_{op}(f,a,b,\gamma,\bm{z},s')&=[(f(\bm{x})-a)^2-
2(1+\gamma)f(\bm{x})]y/p -\gamma^2\\
& +
\left(\beta s' +\left[(f(\bm{x})-b)^2+2(1+\gamma) f(\bm{x})-s'\right]_+\right)(1-y)/[\beta(1-p)].
\end{aligned}
\end{equation}

We have the equivalent optimization problems for $\op$:
\begin{equation}
\begin{aligned}
    &\underset{\cmin}{\min}\ \underset{\gamma\in[-1,1]}{\max} \
\colblue{\ezdz}[F_{op}(f,a,b,\gamma,\colblue{\efb}, \bm{z})]
\Leftrightarrow \\
&\underset{\cmin}{\min}\
\underset{\cmax }{\max}
\ \underset{s'\in\Omega_{s'}}{\min}
\ \colblue{\ezdz}[G_{op}(f,a,b,\gamma,\bm{z},s')],
\end{aligned}
\end{equation}
where $\Omega_{\gamma}=[b-1,1]$, $\Omega_{s'}=[0,5]$, $p=\mathbb{P}[y=1]$. The same result holds for the empirical version $\colbit{\underset{\bm{z}\sim S}{\hat{\E}}}[G_{op}(f,a,b,\gamma,\bm{z},s')]$.

\end{proof}

\subsubsection{Step 3 (\thmref{thm:bias_converge_op})}
\label{sec:uniform_convergence}
\begin{proof}
According to the definition:
\begin{equation*}
\begin{aligned}
\Delta_\kappa^{op} =&\bigg|
\min_{\cmin} \max_{\cmax} \min_{s^{\prime}\in\Omega_{s'}}
\ezdz \Big[
G_{op}^{\kappa}\left(f, a, b, \gamma, \bm{z}, s^{\prime}\right)
\Big]\\
&~-
\min_{\cmin} \max_{\cmax} \min_{s^{\prime}\in\Omega_{s'}}
\ezdz \Big[
G_{op}(f,a, b, \gamma, \bm{z}, s^{\prime})\Big]
\bigg|,
\end{aligned}
\end{equation*}
first we have:
\begin{equation}
\begin{aligned}
& \limsup_{\kappa \rightarrow +\infty}\Delta_\kappa^{op} \leq & \underbrace{
\limsup_{\kappa \rightarrow +\infty}
\sup_{\cmin, \cmax, s^{\prime}\in \Omega_{s'},
\bm{z}\sim\setD_{\setZ}}
\left|
\frac{\log(1+\exp(\kappa\cdot g))}{\kappa} -
[g]_+
\right|}_{(a)},
\end{aligned}
\end{equation}
where $g=(f(\bm{x})-b)^2+2(1+\gamma)f(\bm{x})-s'$ and $[x]_+=\max\{x,0\}$. Since $g \in[-5,5]$ in the feasible set, we have:
\begin{equation}
(a)\le \limsup_{\kappa \rightarrow +\infty}
\sup_{x\in[-5,5]}
\left|
\frac{\log(1+\exp(\kappa\cdot x))}{\kappa} -
[x]_+
\right|.
\end{equation}

Next we prove that
\begin{equation}
\underset{\kappa\to\infty}{\limsup}
\underset{x \in [-5,5]}{\sup}
\left[\left|
\frac{\log(1+\exp(\kappa\cdot x))}{\kappa} -
[x]_+
\right|\right] \le 0.
\end{equation}

For the sake of simplicity, we denote:
\begin{equation}
\ell(x) = \left|
\frac{\log(1+\exp(\kappa\cdot x))}{\kappa} - [x]_+
\right|.
\end{equation}
It is easy to see that, when $x<0$, we have:
\begin{equation}
\nabla \ell(x) = \nabla \left(\frac{\log(1+\exp(\kappa \cdot x))}{\kappa}\right) \ge 0.
\end{equation}
When $x>0$, we have:
\begin{equation}
\nabla \ell(x)= \nabla \left(\frac{\log(1+\exp(\kappa \cdot x))}{\kappa}-x\right) \le 0.
\end{equation}
Hence, the supremum must be attained at $x=0$. We thus have:
\begin{equation}
    (a)\le \limsup_{\kappa\rightarrow +\infty} \frac{\log(1)}{\kappa}= 0 .
\end{equation}

Obviously, the absolute value ensures that:
\begin{equation}
    \liminf_{\kappa\rightarrow +\infty}\Delta_\kappa^{op}\ge 0.
\end{equation}
Then the result follows from the fact:
\begin{equation}
    0 \le \liminf_{\kappa \rightarrow +\infty} \Delta_\kappa^{op} \le \limsup_{\kappa \rightarrow +\infty} \Delta_\kappa^{op} \le 0.
\end{equation}

Moreover, from the proof above, we also obtain a convergence rate:
\begin{equation}
    \Delta_\kappa^{op} = O(1/\kappa).
\end{equation}

\end{proof}

\subsection{Proofs for TPAUC}

\subsubsection{Step 1}
\textbf{Restate of Theorem \ref{thm:7}.}
\emph{Assuming that $f(\bm{x})\in[0,1]$, $\forall \bm{x}\in\mathcal{X}$, $F_{tp}(f,a,b,\gamma, t, t', \bm{z})$ is defined as:}

\begin{equation}
    \begin{aligned}
F_{tp}(f,a,b,&\gamma,t, t', \bm{z})=(f(\bm{x})-a)^2y\indicator{f(\bm{x})\leq t}/(\alpha p)+
(f(\bm{x})-b)^2(1-y)\indicator{f(\bm{x}')\geq t'}/[\beta(1-p)]\\
&+2(1+\gamma)f(\bm{x})(1-y)\indicator{f(\bm{x}')\geq t'}/[\beta(1-p)] -
2(1+\gamma)f(\bm{x})y\indicator{f(\bm{x})\leq t}/(\alpha p) -\gamma^2,
\end{aligned}
\label{eq:fop}
\end{equation}

\emph{where $y=1$ for positive instances, $y=0$ for negative instances and we have the following conclusions:
\begin{enumerate}[leftmargin=20pt]
    \item[(a)] (\textbf{Population Version}.) We have:
    \begin{equation}
        \underset{f}{\min} \ {\mathcal{R}}_{\alpha,\beta}(f) \Leftrightarrow \underset{\cmin}{\min}\ \underset{\gamma\in[-1,1]}{\max} \
    \colblue{\ezdz}
    \left[F_{tp}(f,a,b,\gamma, \colblue{\efa},\colblue{\efb},\bm{z})\right],
    \label{eq:minmaxtpauc1}
    \end{equation}
    where $\colblue{\efa}=\arg\min_{\colblue{\eta_{\alpha}} \in\mathbb{R}}\E_{\bm{x}\sim \setD_{\setP}}[\indicator{f(\bm{x})\leq\ \colblue{\eta_{\alpha}}}=\alpha]$ and $\colblue{\efb}=\arg\min_{\colblue{\eta_{\beta}} \in\mathbb{R}}\E_{\bm{x}'\sim \setD_{\setN}}[\indicator{f(\bm{x}')\geq\ \colblue{\eta_{\beta}}}=\beta]$.
    \item[(b)] (\textbf{Empirical Version}.) Moreover, given a training dataset $S$ with sample size $n$, denote:
    \begin{equation*}
    \colbit{\underset{\bm{z} \sim S}{\ehat}}[F_{tp}(f,a,b,\gamma,\colbit{\hat{\eta}_{\alpha}(f)}, \colbit{\hefb}, \bm{z})] = \frac{1}{n}\sum_{i=1}^n F_{tp}(f,a,b,\gamma,\colbit{\hat{\eta}_{\alpha}(f)}, {\colbit{\hefb}}, \bm{z})
    \end{equation*}
    where $\colbit{\hat{\eta}_{\alpha}(f)}$ and $\colbit{\hefb}$ are the empirical quantile of the positive and negative instances in $S$, respectively. We have:
    \begin{equation}
        \underset{f}{\min} \ \hat{\mathcal{R}}_{\alpha,\beta}(f, S) \Leftrightarrow \underset{\cmin}{\min}\ \underset{\cmax}{\max} \
    \colbit{\hezds}
    \left[F_{tp}(f,a,b,\gamma, \colbit{\hat{\eta}_{\alpha}(f)}, \colbit{\hefb}, \bm{z} ) \right],
    \end{equation}
\end{enumerate}
}
\begin{proof}
Firstly, we give a reformulation of $\tp$:

\begin{equation}
\begin{aligned}
    \underset{f}{\min} \ {\mathcal{R}}_{\alpha,\beta}(f) &=\underset{f}{\min} \ \E_{\bm{x} \sim \setD_\setP, \bm{x}'\sim \setD_\setN} \left[\indicator{f(\bm{x}) \le \efa} \cdot\indicator{f(\bm{x}') \ge \efb} \cdot \ell(f(\bm{x})- f(\bm{x}'))\right]\\
    &=\underset{f}{\min}\ \E_{\bm{x} \sim \setD_\setP, \bm{x}'\sim \setD_\setN} \left[\ell(f(\bm{x})- f(\bm{x}'))|f(\bm{x}') \ge \efb, f(\bm{x}) \le \efa\right]\\
    & \quad \cdot \underset{\bm{x}'\sim\setD_{\setN}}{\mathbb{P}}[f(\bm{x}') \ge \efb]\cdot \underset{\bm{x}\sim\setD_{\setP}}{\mathbb{P}}[f(\bm{x}) \le \efa]\\
    &=\underset{f}{\min}\ \E_{\bm{x} \sim \setD_\setP, \bm{x}'\sim \setD_\setN} \left[\ell(f(\bm{x})- f(\bm{x}'))|f(\bm{x}') \ge \efb,f(\bm{x}) \le \efa \right]\cdot\alpha\beta\\
    &=\alpha\beta\cdot\underset{f}{\min}\ \E_{\bm{x} \sim \setD_\setP, \bm{x}'\sim \setD_\setN} \left[\ell(f(\bm{x})- f(\bm{x}'))|f(\bm{x}') \ge \efb,f(\bm{x}) \le \efa \right].
\end{aligned}
\end{equation}

Similar to the proof of Thm.\ref{thm:instance_op}, using the square surrogate loss, we can get the equivalent optimization formulation:
\begin{equation}
    \begin{aligned}
&\underset{f}{\min} \ \mathcal{R}_{\alpha,\beta}(f)
\Leftrightarrow  \underset{f,(a,b)\in[0,1]^2}{\min}\ \underset{\gamma\in[-1,1]}{\max} \
\colblue{\ezdz}
[F_{tp}(f,a,b,\gamma,\colblue{\eta_{\alpha}(f)}, \colblue{\eta_{\beta}(f)},\bm{z})],
\end{aligned}
\end{equation}
and the instance-wise function $F_{tp}(f,a,b,\gamma,\colblue{\eta_{\alpha}(f)},\colblue{\eta_{\beta}(f)},\bm{z})$ is defined by:

\begin{equation}
    \begin{aligned}
F_{tp}&(f,a,b,\gamma,\colblue{\eta_{\alpha}(f)},\colblue{\eta_{\beta}(f)},\bm{z})\\
&=(f(\bm{x})-a)^2y\indicator{f(\bm{x})\leq\eta_{\alpha}(f)}/(\alpha p)+
(f(\bm{x})-b)^2(1-y)\indicator{f(\bm{x})\geq\eta_{\beta}(f)}/[\beta (1-p)]\\
&+2(1+\gamma)f(\bm{x})(1-y)\indicator{f(\bm{x})\geq\eta_{\beta}(f)}/[\beta (1-p)] -
2(1+\gamma)f(\bm{x})y\indicator{f(\bm{x})\leq\eta_{\alpha}(f)}/(\alpha p) -\gamma^2.
\end{aligned}
\end{equation}

The same result holds for empirical version $\colbit{\underset{\bm{z}\sim S}{\hat{\E}}}
[F_{tp}(f,a,b,\gamma,\colbit{\hat{\eta}_{\alpha}(f)},\colbit{\hat{\eta}_{\beta}(f)},\bm{z})]$.
\end{proof}

\subsubsection{Step 2}
First we need the following proposition to complete the proof in this subsection.


\begin{prop}\label{prop2}
If $\gamma \in \Omega_{\gamma} =[\max\{b-1, -a\}, 1]$, $\ell_+(\bm{x})=(f(\bm{x})-a)^2-2(1+\gamma) f(\bm{x})$ is a decreasing function \wrt $f(\bm{x})$ when $\bm{x}\sim\setD_\setP$ and $f(\bm{x})\in[0,1]$.
\end{prop}

\begin{proof}
We have:
\begin{equation}
    \frac{\partial \ell_{+}(\bm{x})}{\partial f(\bm{x})}=2(f(\bm{x})-a-1-\gamma).
\end{equation}
Assuming that $f(\bm{x})\in [0,1]$, then the feasible solution of $a$ is nonnegative. When $\gamma \in [\max\{b-1, -a\}, 1]$, the positive loss function's partial derivative $\partial\ell_{+}(\bm{x})/\partial f(\bm{x})\leq 0$. Then $\ell_{+}(\bm{x})$ is an decreasing function \wrt $f(\bm{x})$.
\end{proof}

\begin{rem}
For positive instances, if the loss function is an decreasing function \wrt the score $f(\bm{x})$, then the top-ranked losses are equivalent to the losses of bottom-ranked instances.
\end{rem}

\textbf{Restate of Theorem \ref{thm:8}.}
\emph{Assuming that $f(\bm{x})\in[0,1]$ for all $\bm{x} \in\mathcal{X}$, we have the equivalent optimization for $\tp$:}
\begin{equation}
\begin{aligned}
    \underset{\cmin}{\min}\ \underset{\gamma\in[-1,1]}{\max} \
\colblue{\ezdz}[F_{tp}(f,a,b,\gamma,\colblue{\efa}, \colblue{\efb}, \bm{z})]
\\ \Leftrightarrow \underset{\cmin}{\min}\
\underset{\cmax }{\max}
\ \underset{s\in\Omega_{s},s'\in\Omega_{s'}}{\min}
\ \colblue{\ezdz}[G_{tp}(f,a,b,\gamma,\bm{z},s,s')],
\end{aligned}
\end{equation}
\begin{equation}
    \begin{aligned}
        \underset{\cmin}{\min}\ \underset{\gamma\in[-1,1]}{\max} \
    \colbit{\hezds}[F_{tp}(f,a,b,\gamma, \colbit{\hefa}, \colbit{\hefb},\bm{z})]
    \\ \Leftrightarrow \underset{\cmin}{\min}\
    \underset{\gamma\in\Omega_{\gamma} }{\max}
    \ \underset{s\in\Omega_{s}, s'\Omega_{s'}}{\min}
    \ \colbit{\hezds}[G_{tp}(f,a,b,\gamma,\bm{z},s,s')],
    \end{aligned}
\end{equation}

\emph{where $\Omega_{\gamma}=[\max\{b-1, -a\},1]$, $\Omega_{s}=[-4, 1]$, $\Omega_{s'}=[0,5]$ and}

\begin{equation}
    \begin{aligned}
G_{tp}(f,a,b,\gamma,\bm{z},s, s')&=\left(\alpha s + \left[(f(\bm{x})-a)^2-
2(1+\gamma)f(\bm{x})-s\right]_+\right)y/(\alpha p)\\
& +
\left(\beta s' +\left[(f(\bm{x})-b)^2+2(1+\gamma) f(\bm{x})-s'\right]_+\right)(1-y)/[\beta(1-p)] -\gamma^2.
\end{aligned}
\label{OPAUC_IB}
\end{equation}

\begin{proof}

According to the Thm.\ref{Constrainted_Reformulation} in Appx.\ref{val_constraintval},
when we constraint $\gamma$ in range $\Omega_{\gamma}=[\max\{-a,b-1\}, 1]$, we have:
\begin{equation}
    \min_{f,(a,b)\in[0,1]^2}\max_{\gamma\in [-1,1]}\colblue{\underset{\bm{z}\sim\setD_{\mathcal{Z}}}{\E}} [F_{tp}]
    \Leftrightarrow
    \min_{f,(a,b)\in[0,1]^2}\max_{\gamma\in [\max\{-a,b-1\},1]}\colblue{\underset{\bm{z}\sim\setD_{\mathcal{Z}}}{\E}} [F_{tp}]
\end{equation}

According to the Thm.\ref{thm:8}, we have:

\begin{equation}
    \begin{aligned}
\colblue{\ezdz}
[F_{tp}(f,a,b,\gamma,\colblue{\eta_{\alpha}(f)},&\colblue{\eta_{\beta}(f)},\bm{z})]\Leftrightarrow
\underset{\bm{x}\sim \setD_\setP}{\E}\left([(f(\bm{x})-a)^2-
2(1+\gamma)f(\bm{x})]\cdot\indicator{f(\bm{x})\leq \eta_{\alpha}(f)}\right)/\alpha\\
&+\underset{\bm{x}'\sim \setD_\setN}{\E}
\left([(f(\bm{x}')-b)^2+2(1+\gamma)f(\bm{x}')]\cdot \indicator{f(\bm{x}')\geq  \eta_{\beta}(f)}\right)/\beta -\gamma^2.
\end{aligned}
\end{equation}

When we constraint $\gamma$ in range $\Omega_{\gamma}=[\max\{b-1, -a\},1]$, Prop.\ref{prop1} and Prop.\ref{prop2} ensure that the positive and negative loss functions are monotonous. Then we can get:
\begin{equation}
    \E_{\bm{x}\sim\setD_\setP}[\indicator{f(\bm{x})\leq\eta_{\alpha}(f)}\cdot\ell_+(\bm{x})]= \min_s \frac{1}{\alpha} \cdot \E_{\bm{x}\sim\setD_\setP} [\alpha s + [\ell_+(\bm{x})-s]_+],
\end{equation}
\begin{equation}
    \E_{\bm{x}'\sim\setD_\setN}[\indicator{f(\bm{x}')\geq\eta_{\beta}(f)}\cdot\ell_-(\bm{x}')]= \min_{s'} \frac{1}{\beta} \cdot \E_{\bm{x}'\sim\setD_\setN} [\beta s' + [\ell_-(\bm{x}')-s']_+].
\end{equation}
Applying the Lem.\ref{lem:topk} to positive and negative loss, we have:
\begin{equation}
    \begin{aligned}
\colblue{\ezdz}
[F_{tp}(f,a,b,\gamma,\colblue{\eta_{\alpha}(f)},\colblue{\eta_{\beta}(f)},\bm{z})]&\Leftrightarrow \underset{s, s'}{\min}
\underset{\bm{x}\sim \setD_\setP}{\E}\left(\alpha s + \left[(f(\bm{x})-a)^2-
2(1+\gamma)f(\bm{x}) - s\right]_+\right)/\alpha -\gamma^2\\
&+\underset{\bm{x}'\sim \setD_\setN}{\E}
\left(\beta s' +\left[(f(\bm{x}')-b)^2+2(1+\gamma)f(\bm{x}')-s'\right]_+\right)/\beta,
\end{aligned}
\end{equation}
Then, we get:
\begin{equation}
\begin{aligned}
\colblue{\ezdz}[F_{tp}(f,a,b,\gamma, \colblue{\eta_{\alpha}(f)}, \colblue{\eta_{\beta}(f)}, \bm{z})]
&\Leftrightarrow
\ \underset{s\in\Omega_{s},s'\in\Omega_{s'}}{\min}
\ \colblue{\ezdz}[G_{tp}(f,a,b,\gamma,\bm{z},s,s')].
\end{aligned}
\end{equation}
where $\Omega_{\gamma}=[\max\{b-1, -a\},1]$, $\Omega_{s}=[-4,1]$, $\Omega_{s'}=[0,5]$, $p=\mathbb{P}[y=1]$ and

\begin{equation}
    \begin{aligned}
G_{tp}(f,a,b,\gamma,&\bm{z},s,s')=\left(\alpha s + \left[(f(\bm{x})-a)^2-
2(1+\gamma)f(\bm{x})-s\right]_+\right)y/(\alpha p)\\
& +
\left(\beta s' +\left[(f(\bm{x})-b)^2+2(1+\gamma) f(\bm{x})-s'\right]_+\right)(1-y)/[\beta(1-p)] -\gamma^2.
\end{aligned}
\end{equation}

we have the equivalent optimization for $\tp$:
\begin{equation}
\begin{aligned}
    \underset{\cmin}{\min}\ \underset{\gamma\in[-1,1]}{\max} \
\colblue{\ezdz}[F_{tp}(f,a,b,\gamma,\colblue{\efa}, \colblue{\efb}, \bm{z})]
\\ \Leftrightarrow \underset{\cmin}{\min}\
\underset{\cmax }{\max}
\ \underset{s\in\Omega_{s},s'\in\Omega_{s'}}{\min}
\ \colblue{\ezdz}[G_{tp}(f,a,b,\gamma,\bm{z},s,s')],
\end{aligned}
\end{equation}
The same result is hold for empirical version $\colbit{\hezds}[G_{tp}(f,a,b,\gamma,\bm{z},s,s')]$.
\end{proof}

\section{Proof of Generalization Bound}\label{sec:proof_generalization}
The proof is based on the following lemmas.

\begin{lem}
\begin{equation}
\begin{aligned}
    \underset{x}{\max} \ f(x) - \underset{x'}{\max} \ g(x') &\leq \underset{x,x'=x}{\max}\ f(x)-g(x)\\
    \underset{x}{\min} \ f(x) - \underset{x'}{\min} \ g(x') &\leq \underset{x,x'=x}{\max}\ f(x)-g(x).
\end{aligned}
\end{equation}
\label{lem:4}
\end{lem}

\begin{proof}
Since the difference of suprema does not exceed the supremum of the
difference, we have:
\begin{equation}
    \underset{x}{\max} \ f(x) - \underset{x'}{\max} \ g(x')\leq \underset{x}{\max} \ \underset{x'}{\min} f(x)-g(x')\leq \underset{x,x'=x}{\max}\ f(x)-g(x).
\end{equation}
For $\underset{x}{\min} \ f(x) - \underset{x'}{\min} \ g(x')\leq \underset{x,x'=x}{\max}\ f(x)-g(x)$, we have:
\begin{equation}
    \begin{aligned}
\underset{x}{\min} \ f(x) - \underset{x'}{\min} \ g(x')
&\leq\underset{x}{\min}\ \underset{x'}{\max}\ f(x) - g(x')\\
&=\underset{x'}{\max}\ \underset{x}{\min}\ f(x) - g(x')
\leq \underset{x,x'=x}{\max}\ f(x)-g(x).
\end{aligned}
\end{equation}
\end{proof}

\begin{lem}\label{lem:abs_hinge}
    For any $x,y\in\mathbb{R}$, we have that
    \begin{equation}
        \lvert[x]_+ - [y]_+\rvert \leq \lvert x-y\rvert.
    \end{equation}
\end{lem}

\begin{proof}
    We can consider the different cases:
    \begin{itemize}
        \item When $x\geq 0, y\geq 0$, the inequality holds naturally.
        \item When $x\geq 0, y<0$,
        \begin{equation}
            \lvert[x]_+ - [y]_+\rvert = \lvert x\rvert \leq x-y = \lvert x-y\rvert.
        \end{equation}
        \item When $x<0, y\geq 0$,
        \begin{equation}
            \lvert[x]_+ - [y]_+\rvert = \lvert y\rvert \leq y-x = \lvert x-y\rvert.
        \end{equation}
        \item When $x<0, y<0$,
        \begin{equation}
            \lvert[x]_+ - [y]_+\rvert = 0 \leq \lvert x-y\rvert.
        \end{equation}
    \end{itemize}
\end{proof}

\begin{defn}[Sub-root Function \cite{bartlett2005local}]
    A function $\psi: [0,\infty) \to [0,\infty)$ is sub-root if it is nonnegative, nondecreasing, and if $r \mapsto \psi(r)/\sqrt{r}$ is nonincreasing for $r>0$.
\end{defn}



\begin{defn}[Local Rademacher Complexity \cite{bartlett2005local}]\label{defn:localrad_pop}
    Let $\setF_r = \{f\in\setF: \E f^2(\bm{x}) \leq r\}$ be a subset of the function class $\setF$ with radius $r$. The local Rademacher complexity of the function class $\setF$ on the distribution $\setD_\setZ$ is:
    \begin{equation}
        \begin{aligned}
            \Re_r(\setF) = \underset{S\sim \setD_\setZ^n,\bm{\sigma}}{\E}\left[\underset{f\in\setF_r}{\sup} \frac{1}{n}\sum_{i=1}^{n}\sigma_i f(\bm{x}_i)\right],
        \end{aligned}
    \end{equation}
    where $(\sigma_1, \cdots, \sigma_{n})$ are independent uniform random variables taking values in $\{-1, +1\}$.
\end{defn}

\subsection{OPAUC}
\textbf{Restate of Theorem \ref{thm:4}.}
\emph{Assume there exist three positive constants $R, D$ and $h$ such that the following bound holds for the covering number of $\setF$ \wrt\ $\norm{\cdot}_2$ norm:}
\begin{equation}
    \log \setN(\epsilon, \setF, \norm{\cdot}_2) \leq D\log^h(R/\epsilon).
\end{equation}
\emph{Then for any $\delta>0$, with probability at least $1-\delta$ over the draw of an i.i.d. sample set $S$ of size $n$ ($n \geq R^{-2}$), for all $f\in\setF$ and $K>1$ we have:}
\begin{equation*}
\begin{aligned}
\underset{(a,b)\in[0,1]^2}{\min}\
\underset{\cmax }{\max} \underset{s'\in\Omega_{s'}}{\min}
\ \colblue{\ezdz}[G_{op}(f,a,b,\gamma,\bm{z},s')] \le& \frac{K}{K-1}\ \underset{\cmins}{\min}\underset{\cmax }{\max}  \underset{s'\in\Omega_{s'}}{\min}
\ \colbit{\hezds}[G_{op}(f,a,b,\gamma,\bm{z},s')] \\
&~~ + \tilde{O}(n_+^{-1} + \beta^{-1}n_-^{-1}).
\end{aligned}
\end{equation*}

\begin{proof}
Firstly it is easy to see that for any $f\in\setF$, $\text{Var}(\op(f))\leq B\cdot\E[\op(f)], \forall B\geq 4$. Thus, for any $a^*, b^*, \gamma^*$ and $s'^*$ recovering the OPAUC objective, we should also have that $\text{Var}(G_{op}(f, a^*, b^*, \gamma^*, \bm{z}, s'^*))\leq B\cdot\E[G_{op}(f,a^*, b^*, \gamma^*, \bm{z}, s'^*)]$. Then letting $G_f^*$ denote $G_{op}(f, a^*, b^*, \gamma^*, \bm{z}, s'^*)$, according to Lem.\ref{lem:4}, for any $K>1$ we have that
\begin{equation}
\begin{aligned}
    &\sup_{f\in\setF}\bigg(\underset{(a,b)\in[0,1]^2}{\min}\
    \underset{\cmax }{\max} \underset{s'\in\Omega_{s'}}{\min}
    \ \colblue{\ezdz}\left[G_{op}(f,a,b,\gamma,\bm{z},s')\right] \\
    &~~~~~~~~~~ -\frac{K}{K-1}\ \underset{\cmins}{\min}\underset{\cmax }{\max}  \underset{s'\in\Omega_{s'}}{\min} \ \colbit{\hezds} \left[G_{op}(f,a,b,\gamma,\bm{z},s')\right]\bigg) \\
    \leq& \sup_{f\in\setF,\text{Var}(G_f^*)\leq B\cdot\E[G_f^*]}\left(\colblue{\ezdz}\left[G_f^*\right] - \frac{K}{K-1}\ \colbit{\hezds}\left[G_f^*\right]\right). \\
\end{aligned}
\end{equation}

Moreover, it naturally holds that $\text{Var}(G_f^*)\leq\E[G_f^{*^2}]\leq B\cdot\E[G_f^*]$. Applying Thm.3.3 in \cite{bartlett2005local}, for any $K>1$ and $\delta>0$, with probability at least $1-\delta$, it holds that
\begin{equation}\label{eq:local_bound}
\begin{aligned}
    \underset{\bm{z}\sim\setD_\setZ}{\E}\left[G_f^*\right] - \frac{K}{K-1}\underset{\bm{z}\sim S}{\hat{\E}}\left[G_f^*\right] \leq \frac{c_1K}{B}r^*+\frac{(44+c_2BK)\log\frac{1}{\delta}}{n},
\end{aligned}
\end{equation}
where $c_1=704,c_2=26$, $r^*$ is the fixed point\footnote{That is, $r^*$ is the solution of $r=\psi(r)$.} of the sub-root function $\psi(r) = \Re_r(\{G^*_f: f\in\setF \})$ where the local Rademacher complexity $\Re_r(\cdot)$ is defined in Defn.\ref{defn:localrad_pop}.

Now we bound the fixed point $r^*$ based on covering numbers.

For any $f, \tilde{f}\in\setF$, we have the following decomposition
\begin{equation}\label{eq:f_decomp}
\begin{aligned}
    & \lvert G^*_f - G^*_{\tilde{f}}\rvert = \abs{G_{op}(f, a^*, b^*, \gamma^*, \bm{z}, s'^*) - G_{op}(\tilde{f}, \tilde{a}^*, \tilde{b}^*, \tilde{\gamma}^*, \bm{z}, \tilde{s}'^*)}\\
    \leq & \left\lvert (P(f,a^*,\gamma^*, \bm{x}) -  P(\tilde{f},\tilde{a}^*, \tilde{\gamma}^*, \bm{x})) \frac{y}{p} + \gamma^{*2} - \tilde{\gamma}^{*2} + (s'^* - \tilde{s}'^*)\frac{1-y}{\beta(1-p)} \right.\\
    & \left. + \left([N(f,b^*,\gamma^*, \bm{x}) - s'^*]_+ -  [N(\tilde{f},\tilde{b}^*,\tilde{\gamma}^*, \bm{x}) - \tilde{s}'^*]_+\right) \frac{1-y}{\beta(1-p)}\right\rvert\\
    \leq & \left\lvert P(f,a^*,\gamma^*, \bm{x}) -  P(\tilde{f},\tilde{a}^*, \tilde{\gamma}^*, \bm{x}) \right\rvert \frac{y}{p} + \lvert \gamma^{*2} - \tilde{\gamma}^{*2}\rvert + \left\lvert s'^* - \tilde{s}'^*\right\rvert\frac{1-y}{\beta(1-p)} \\
    & + \left\lvert[N(f,b^*,\gamma^*, \bm{x}) - s'^*]_+ -  [N(\tilde{f},\tilde{b}^*,\tilde{\gamma}^*, \bm{x}) - \tilde{s}'^*]_+\right\rvert \frac{1-y}{\beta(1-p)}\\
    \leq & \left\lvert P(f,a^*,\gamma^*, \bm{x}) -  P(\tilde{f},\tilde{a}^*, \tilde{\gamma}^*, \bm{x}) \right\rvert \frac{y}{p} + 2 \lvert \gamma^{*} - \tilde{\gamma}^{*}\rvert + \left\lvert s'^* - \tilde{s}'^*\right\rvert\frac{2(1-y)}{\beta(1-p)} \\
    & + \left\lvert N(f,b^*,\gamma^*, \bm{x}) -  N(\tilde{f},\tilde{b}^*,\tilde{\gamma}^*, \bm{x})\right\rvert \frac{1-y}{\beta(1-p)}\\
    \leq & \left( 2\lvert a^*-\tilde{a}^*\rvert + 2\lvert \gamma^*-\tilde{\gamma}^*\rvert + 6\lvert f(\bm{x})-\tilde{f}(\bm{x})\rvert \right) \frac{y}{p} + 2 \lvert \gamma^{*} - \tilde{\gamma}^{*}\rvert + \left\lvert s'^* - \tilde{s}'^*\right\rvert\frac{2(1-y)}{\beta(1-p)} \\
    & + \left(2\lvert b^* - \tilde{b}^*\rvert + 2\lvert \gamma^*-\tilde{\gamma}^*\rvert + 8\lvert f(\bm{x}) - \tilde{f}(\bm{x})\rvert\right) \frac{1-y}{\beta(1-p)}\\
    \leq & \frac{2y}{p}\lvert a^* - \tilde{a}^*\rvert + \frac{2(1-y)}{\beta(1-p)}\lvert b^*-\tilde{b}^*\rvert + \left(\frac{2y}{p} + \frac{2(1-y)}{\beta(1-p)} + 2\right)\lvert \gamma^*-\tilde{\gamma}^*\rvert\\
    & + \frac{2(1-y)}{\beta(1-p)}\lvert s'^* - \tilde{s}'^*\rvert + \left(\frac{6y}{p} + \frac{8(1-y)}{\beta(1-p)}\right)\lvert f-\tilde{f}\rvert,
\end{aligned}
\end{equation}
where the third inequality follows from Lem.\ref{lem:abs_hinge}.

Let $q(\beta) := \frac{10y}{p} + \frac{14(1-y)}{\beta(1-p)} + 2$ and $\mathcal{G}_{\beta} := \{G^*_f: f\in\setF\}$. We could decompose an $\epsilon$-covering set of $\mathcal{G}_{\beta}$ according to Eq.\eqref{eq:f_decomp}, and then bound its covering number as follows:
\begin{equation}\label{eq:cover_decomp}
\begin{aligned}
    \log \setN(\epsilon, \mathcal{G}_{\beta}, \lVert\cdot\rVert_2)\leq & 2\log\setN\left(\frac{\epsilon}{q(\beta)}, [0,1], \lvert\cdot\rvert\right) + \log\setN\left(\frac{\epsilon}{q(\beta)}, [-1,1], \lvert\cdot\rvert\right) \\
    & + \log\setN\left(\frac{\epsilon}{q(\beta)}, [0,5], \lvert\cdot\rvert\right) + \log\setN\left(\frac{\epsilon}{q(\beta)}, \setF, \lVert\cdot\rVert_2\right) \\
    \leq & 2\log\left(\frac{q(\beta)}{2\epsilon}\right) + \log\left(\frac{q(\beta)}{\epsilon}\right) + \log\left(\frac{3q(\beta)}{2\epsilon}\right) + D\log^h\left(\frac{q(\beta)R}{\epsilon}\right)\\
    \leq & 4\log\left(\frac{3q(\beta)}{\epsilon}\right) + D\log^h\left(\frac{q(\beta)R}{\epsilon}\right),
\end{aligned}
\end{equation}
where the second last line follows from $\setN\left(\epsilon, [a,b], \lvert\cdot\rvert\right) \leq \frac{b-a}{2\epsilon}$.

Then we follow a similar proof as that of Cor.1 in \cite{lei2016local} to bound the local Rademacher complexity by covering numbers. Let $\Re^S_r(\setF)$ be the local Rademacher complexity defined on the set $\setF^S_r = \{f\in\setF: \hat{\E}_S f^2(\bm{x}) \leq r\}$\footnote{It is different from $\Re_r(\setF)$ which is defined on $\setF_r = \{f\in\setF: \E f^2(\bm{x}) \leq r\}$.}. According to Thm.2 in \cite{lei2016local}, for $f\in\setF$ with the upper bounded output $B_f$ it holds that
\begin{equation}
\begin{aligned}
    \Re_r(\setF) \leq \inf_{\epsilon>0}\left[2\Re^S_{\epsilon^2}(\tilde{\setF}) + \frac{8B_f\log\setN(\epsilon/2, \setF, \lVert\cdot\rVert_2)}{n} + \sqrt{\frac{2r\log\setN(\epsilon/2,\setF, \lVert\cdot\rVert_2)}{n}}\right],
\end{aligned}
\end{equation}
where $\tilde{\setF} := \left\{f-g:f,g\in\setF\right\}$. Then with the result in Eq.\eqref{eq:cover_decomp}, we have
\begin{equation}\label{eq:rad_cover}
\begin{aligned}
    \Re_r(\mathcal{G}_{\beta}) \leq \inf_{0<\epsilon\leq 2\tilde{R}}\left[2\Re^S_{\epsilon^2}(\mathcal{\tilde{G}}_{\beta}) + \frac{32\left(4\log\left(\frac{3q(\beta)}{\epsilon}\right) + D\log^h\left(\frac{q(\beta)R}{\epsilon}\right)\right)}{n} + \sqrt{\frac{2r\left(4\log\left(\frac{3q(\beta)}{\epsilon}\right) + D\log^h\left(\frac{q(\beta)R}{\epsilon}\right)\right)}{n}}\right],
\end{aligned}
\end{equation}
where $\tilde{R} := \min\left(3q(\beta), R q(\beta)\right)$.

Now we focus on the first term $\Re^S_{\epsilon^2}(\mathcal{\tilde{G}}_{\beta})$. Set $\epsilon_{k} = 2^{-k}\epsilon$. According to Lem.A.5 in \cite{lei2016local}, we can obtain
\begin{equation}
\begin{aligned}
    \Re^S_{\epsilon^2}(\mathcal{\tilde{G}}_{\beta}) \leq & 4\sum_{k=1}^N\epsilon_{k-1}\sqrt{\frac{\log\setN(\epsilon_k/2, \tilde{\mathcal{G}}_{\beta}, \lVert\cdot\rVert_2)}{n}} + \epsilon_N \\
    \leq & 2^{\frac{7}{2}}\epsilon\sqrt{\frac{D'}{n}}\sum_{k=1}^N2^{-k}\sqrt{\log\left(\frac{3\cdot 2^{k+2}q(\beta)}{\epsilon}\right)+\log^h\left(\frac{2^{k+2}q(\beta)R}{\epsilon}\right)} + \epsilon_N \\
    \leq & 2^{\frac{7}{2}}\epsilon\sqrt{\frac{D'}{n}}\sum_{k=1}^N2^{-k}\left(\sqrt{(k+1)\log 2} + \sqrt{\log\left(\frac{6q(\beta)}{\epsilon}\right)} + \log^{\frac{h}{2}}\left(\frac{2^{k+2}q(\beta)R}{\epsilon}\right)\right) + \epsilon_N \\
    \leq & 2^{\frac{7+h}{2}}\epsilon\sqrt{\frac{D'}{n}}\sum_{k=1}^N2^{-k}\left(\sqrt{(k+1)\log 2} + \sqrt{\log\left(\frac{6q(\beta)}{\epsilon}\right)} + \left((k+1)\log 2\right)^{\frac{h}{2}} + \log^{\frac{h}{2}}\left(\frac{2q(\beta)R}{\epsilon}\right)\right) + \epsilon_N \\
    \leq & 2^{\frac{7+h}{2}}\epsilon\sqrt{\frac{D'}{n}}\left[c(h) + \sqrt{\log\left(\frac{6q(\beta)}{\epsilon}\right)} + \log^{\frac{h}{2}}\left(\frac{2q(\beta)R}{\epsilon}\right)\right] + \epsilon_N,
\end{aligned}
\end{equation}
where $D' = \max(4,D)$, and $c(h)$ is a constant dependent on $h$. The third and forth line follow from the result that $(a+b)^{h/2}\leq (2\max(a,b))^{h/2}\leq 2^{h/2}(a^{h/2}+b^{h/2})$. The last line is due the fact that the infinite series $\sum_{k=1}^{\infty}2^{-k}((k+1)\log2)^{h/2}$ converges.

Substituting this result back into Eq.\eqref{eq:rad_cover} and let $N\to\infty$, we obtain that
\begin{equation}
\begin{aligned}
    \Re_r(\mathcal{G}_{\beta}) \leq \inf_{0<\epsilon\leq 2\tilde{R}}& \left[2^{\frac{9+h}{2}}\epsilon\sqrt{\frac{D'}{n}}\left(c(h) + \sqrt{\log\left(\frac{6q(\beta)}{\epsilon}\right)} + \log^{\frac{h}{2}}\left(\frac{2q(\beta)R}{\epsilon}\right)\right) \right. \\
    & + \left. \frac{32\left(4\log\left(\frac{3q(\beta)}{\epsilon}\right) + D\log^h\left(\frac{q(\beta)R}{\epsilon}\right)\right)}{n} + \sqrt{\frac{2r\left(4\log\left(\frac{3q(\beta)}{\epsilon}\right) + D\log^h\left(\frac{q(\beta)R}{\epsilon}\right)\right)}{n}}\right].
\end{aligned}
\end{equation}

As $n\geq R^{-2}$, by setting $\epsilon = \frac{q(\beta)}{\sqrt{n}}$, we could define the sub-root function in the following form
\begin{equation}
\begin{aligned}
    \psi(r) := & 2^{\frac{9+h}{2}}\frac{q(\beta)\sqrt{D'}}{n}\left(c(h) + \sqrt{\log(6\sqrt{n})} + \log^{\frac{h}{2}}(2R\sqrt{n})\right) \\
    & + \frac{32\left(4\log(3\sqrt{n}) + D\log^h(R\sqrt{n})\right)}{n} + \sqrt{\frac{2r\left(4\log(3\sqrt{n}) + D\log^h(R\sqrt{n})\right)}{n}}.
\end{aligned}
\end{equation}

Notice that the formula $r = \psi(r)$ is in the form of $x = a + \sqrt{bx}$, the solution of which is $\frac{2a + b \pm \sqrt{b^2 + 4ab}}{2} = O(a+b)$. Therefore, for $r^*$ satisfying $r = \psi(r)$, we have that
\begin{equation}
\begin{aligned}
    r^* \lesssim & \frac{4\log(3\sqrt{n}) + 4D\log^{h}(R\sqrt{n})}{n} + \frac{q(\beta)\left(\sqrt{\log(6\sqrt{n})} + \log^{h/2}(2R\sqrt{n})\right)}{n} \\
    \lesssim & \frac{4\log(3\sqrt{n}) + 4D\log^{h}(R\sqrt{n})}{n} + \left(\sqrt{\log(6\sqrt{n})} + \log^{h/2}(2R\sqrt{n})\right)\left(\frac{10y}{np} + \frac{14(1-y)}{\beta(1-p)n} + \frac{2}{n}\right)\\
    \lesssim & \frac{4\log(3\sqrt{n}) + 4D\log^{h}(R\sqrt{n})}{n_+} + \left(\sqrt{\log(6\sqrt{n})} + \log^{h/2}(2R\sqrt{n})\right)\left(\frac{10y}{n_+} + \frac{14(1-y)}{\beta n_-} + \frac{2}{n_+}\right).\\
\end{aligned}
\end{equation}
Substituting it into \Eqref{eq:local_bound} and using the fact $x \le \sup x$ then finish the proof.
\end{proof}

\subsection{TPAUC}\label{sec:proof_gene_tp}
\begin{thm}
Assume there exist three positive constants $R, D$ and $h$ such that the following bound holds for the covering number of $\setF$ \wrt\ $\norm{\cdot}_2$ norm:
\begin{equation}
    \log \setN(\epsilon, \setF, \norm{\cdot}_2) \leq D\log^h(R/\epsilon).
\end{equation}
Then for any $\delta>0$, with probability at least $1-\delta$ over the draw of an i.i.d. sample set $S$ of size $n$ ($n \geq R^{-2}$), for all $f\in\setF$ and $K>1$ we have:
\begin{equation*}
\begin{aligned}
\underset{\cmins}{\min}\
\underset{\cmax }{\max}  \underset{s\in\Omega_{s},s'\in\Omega_{s'}}{\min}
\ \colblue{\ezdz}[G_{tp}(f,a,b,\gamma,\bm{z},s,s')]
\leq& \frac{K}{K-1} \underset{\cmins}{\min}\underset{\cmax }{\max}  \underset{s\in\Omega_{s},s'\in\Omega_{s'}}{\min}
\ \colbit{\hezds}[G_{tp}(f,a,b,\gamma,\bm{z},s,s')] \\
&+ \tilde{O}(\alpha^{-1}n_+^{-1}+\beta^{-1}n_-^{-1}).
\end{aligned}
\end{equation*}
\end{thm}

\begin{proof}
Firstly it is easy to see that for any $f\in\setF$, $\text{Var}(\tp(f))\leq B\cdot\E[\tp(f)], \forall B\geq 4$. Thus, for any $a^*, b^*, \gamma^*, s^*$ and $s'^*$ recovering the TPAUC objective, we should also have that $\text{Var}(G_{tp}(f, a^*, b^*, \gamma^*, \bm{z}, s^*, s'^*))\leq B\cdot\E[G_{tp}(f,a^*, b^*, \gamma^*, \bm{z}, s^*, s'^*)]$. Then letting $G_f^*$ denote $G_{tp}(f, a^*, b^*, \gamma^*, \bm{z}, s^*, s'^*)$, according to Lem.\ref{lem:4}, for any $K>1$ we have that
\begin{equation}
\begin{aligned}
    &\sup_{f\in\setF}\bigg(\underset{(a,b)\in[0,1]^2}{\min}\
    \underset{\cmax }{\max} \underset{s\in\Omega_s, s'\in\Omega_{s'}}{\min}
    \ \colblue{\ezdz}\left[G_{tp}(f,a,b,\gamma,\bm{z},s,s')\right] \\
    &~~~~~~~~~ -\frac{K}{K-1}\ \underset{\cmins}{\min}\underset{\cmax }{\max}  \underset{s\in\Omega_s, s'\in\Omega_{s'}}{\min} \ \colbit{\hezds} \left[G_{tp}(f,a,b,\gamma,\bm{z},s,s')\right]\bigg) \\
    \leq& \sup_{f\in\setF,\text{Var}(G_f^*)\leq B\E[G_f^*]}\left(\colblue{\ezdz}\left[G_f^*\right] - \frac{K}{K-1}\ \colbit{\hezds}\left[G_f^*\right]\right). \\
\end{aligned}
\end{equation}

Moreover, it naturally holds that $\text{Var}(G_f^*)\leq\E[G_f^{*^2}]\leq B\cdot\E[G_f^*]$. Applying Thm.3.3 in \cite{bartlett2005local}, for any $K>1$ and $\delta>0$, with probability at least $1-\delta$, it holds that
\begin{equation}
\begin{aligned}
    \underset{\bm{z}\sim\setD_\setZ}{\E}\left[G_f^*\right] - \frac{K}{K-1}\underset{\bm{z}\sim S}{\hat{\E}}\left[G_f^*\right] \leq \frac{c_1K}{B}r^*+\frac{(44+c_2BK)\log\frac{1}{\delta}}{n},
\end{aligned}
\end{equation}
where $c_1=704,c_2=26$, $r^*$ is the fixed point of the sub-root function $\psi(r) = \Re_r(\{G^*_f: f\in\setF \})$\footnote{That is, $r^*$ is the solution of $r=\psi(r)$.} where $\Re_r(\cdot)$ denotes the local Rademacher complexity.

Now we bound the fixed point $r^*$ based on covering numbers.

For any $f, \tilde{f}\in\setF$, we have the following decomposition
\begin{equation}
\begin{aligned}
    & \lvert G^*_f - G^*_{\tilde{f}} \rvert = \abs{G_{tp}(f, a^*, b^*, \gamma^*, \bm{z}, s^*, s'^*) - G_{tp}(\tilde{f}, \tilde{a}^*, \tilde{b}^*, \tilde{\gamma}^*, \bm{z}, s^*, \tilde{s}'^*)} \\
    \leq & \left\lvert s^* - \tilde{s}^* \right\rvert\frac{y}{p} + \left\lvert \frac{[(f(\bm{x}) - a^*)^2 - 2(1+\gamma^*)f(\bm{x}) - s^*]_+}{\alpha} - \frac{[(\tilde{f}(\bm{x}) - \tilde{a}^*)^2 - 2(1+\tilde{\gamma}^*)\tilde{f}(\bm{x}) - \tilde{s}^*]_+}{\alpha} \right\rvert \frac{y}{p} + \lvert \gamma^{*2} - \tilde{\gamma}^{*2} \rvert \\
    & + \left\lvert s'^* - \tilde{s}'^* \right\rvert \frac{1-y}{1-p} + \left\lvert \frac{[(f(\bm{x}) - b^*)^2 + 2(1+\gamma^*)f(\bm{x}) - s'^*]_+}{\beta} - \frac{[(\tilde{f}(\bm{x}) - \tilde{b}^*)^2 + 2(1+\tilde{\gamma}^*)\tilde{f}(\bm{x}) - \tilde{s}'^*]_+}{\beta} \right\rvert \frac{1-y}{1-p} \\
    \leq & \left(1+\frac{1}{\alpha}\right)\frac{y}{p}\lvert s^* - \tilde{s}^*\rvert + \left(1+\frac{1}{\beta}\right)\frac{1-y}{1-p}\lvert s'^* - \tilde{s}'^*\rvert + \frac{2y}{\alpha p}\lvert a^* - \tilde{a}^*\rvert + \frac{2(1-y)}{\beta(1-p)}\lvert b^* - \tilde{b}^* \rvert \\
    & + 2\left(1 + \frac{y}{\alpha p} + \frac{1-y}{\beta(1-p)}\right)\lvert \gamma^* - \tilde{\gamma}^*\rvert + \left(\frac{6y}{\alpha p} + \frac{8(1-y)}{\beta(1-p)}\right)\left\lvert f(\bm{x}) - \tilde{f}(\bm{x})\right\rvert.
\end{aligned}
\end{equation}

Let $q(\alpha, \beta) := \left(1+\frac{11}{\alpha}\right)\frac{y}{p} + \left(1 + \frac{13}{\beta}\right)\frac{1-y}{1-p} + 2$ and $\mathcal{G}_{\alpha, \beta} := \{G^*_f: f\in\setF\}$. We could decompose an $\epsilon$-covering set of $\mathcal{G}_{\alpha, \beta}$ according to Eq.\eqref{eq:f_decomp}, and then bound its covering number as follows:
\begin{equation}\label{eq:cover_decomp_tp}
\begin{aligned}
    \log \setN(\epsilon, \mathcal{G}_{\alpha, \beta}, \lVert\cdot\rVert_2)\leq & 2\log\setN\left(\frac{\epsilon}{q(\alpha, \beta)}, [0,1], \lvert\cdot\rvert\right) + \log\setN\left(\frac{\epsilon}{q(\alpha, \beta)}, [-1,1], \lvert\cdot\rvert\right) \\
    & + \log\setN\left(\frac{\epsilon}{q(\alpha, \beta)}, [-4,1], \lvert\cdot\rvert\right) + \log\setN\left(\frac{\epsilon}{q(\alpha, \beta)}, [0,5], \lvert\cdot\rvert\right) + \log\setN\left(\frac{\epsilon}{q(\alpha, \beta)}, \setF, \lVert\cdot\rVert_2\right) \\
    \leq & 5\log\left(\frac{3q(\alpha, \beta)}{\epsilon}\right) + D\log^h\left(\frac{q(\alpha, \beta)R}{\epsilon}\right).
\end{aligned}
\end{equation}

Then following the same steps in the proof for the OPAUC generalization bound, we can obtain
\begin{equation}
\begin{aligned}
    \Re_r(\mathcal{G}_{\alpha, \beta}) \leq \inf_{0<\epsilon\leq 2\tilde{R}}& \left[2^{\frac{9+h}{2}}\epsilon\sqrt{\frac{D'}{n}}\left(c(h) + \sqrt{\log\left(\frac{6q(\alpha, \beta)}{\epsilon}\right)} + \log^{\frac{h}{2}}\left(\frac{2q(\alpha, \beta)R}{\epsilon}\right)\right) \right. \\
    & + \left. \frac{32\left(5\log\left(\frac{3q(\alpha, \beta)}{\epsilon}\right) + D\log^h\left(\frac{q(\alpha, \beta)R}{\epsilon}\right)\right)}{n} + \sqrt{\frac{2r\left(5\log\left(\frac{3q(\alpha, \beta)}{\epsilon}\right) + D\log^h\left(\frac{q(\alpha, \beta)R}{\epsilon}\right)\right)}{n}}\right],
\end{aligned}
\end{equation}
where $D' = \max(5,D)$, and $\tilde{R} := \min\left(3q(\alpha, \beta), R q(\alpha, \beta)\right)$.

As $n\geq R^{-2}$, by setting $\epsilon = \frac{q(\alpha, \beta)}{\sqrt{n}}$, we could define the sub-root function in the following form
\begin{equation}
\begin{aligned}
    \psi(r) := & 2^{\frac{9+h}{2}}\frac{q(\alpha, \beta)\sqrt{D'}}{n}\left(c(h) + \sqrt{\log(6\sqrt{n})} + \log^{\frac{h}{2}}(2R\sqrt{n})\right) \\
    & + \frac{32\left(5\log(3\sqrt{n}) + D\log^h(R\sqrt{n})\right)}{n} + \sqrt{\frac{2r\left(5\log(3\sqrt{n}) + D\log^h(R\sqrt{n})\right)}{n}}.
\end{aligned}
\end{equation}

Notice that the formula $r = \psi(r)$ is in the form of $x = a + \sqrt{bx}$, the solution of which is $\frac{2a + b \pm \sqrt{b^2 + 4ab}}{2} = O(a+b)$. Therefore, for $r^*$ satisfying $r = \psi(r)$, we have that
\begin{equation}
\begin{aligned}
    r^* \lesssim & \frac{4\log(3\sqrt{n}) + 4D\log^{h}(R\sqrt{n})}{n} + \frac{q(\alpha, \beta)\left(\sqrt{\log(6\sqrt{n})} + \log^{h/2}(2R\sqrt{n})\right)}{n} \\
    \lesssim & \frac{4\log(3\sqrt{n}) + 4D\log^{h}(R\sqrt{n})}{n} \\
    & + \left(\sqrt{\log(6\sqrt{n})} + \log^{h/2}(2R\sqrt{n})\right)\left(\left(1+\frac{11}{\alpha}\right)\frac{y}{np} + \left(1 + \frac{13}{\beta}\right)\frac{1-y}{n(1-p)} + \frac{2}{n}\right)\\
    \lesssim & \frac{4\log(3\sqrt{n}) + 4D\log^{h}(R\sqrt{n})}{n_+} + \left(\sqrt{\log(6\sqrt{n})} + \log^{h/2}(2R\sqrt{n})\right)\left(\frac{12y}{\alpha n_+} + \frac{14(1-y)}{\beta n_-} + \frac{2}{n_+}\right).\\
\end{aligned}
\end{equation}
Substituting it into \Eqref{eq:local_bound} and using the fact $x \le \sup x$ then finish the proof.

\end{proof}

\section{Experiment Details}
\label{section:experiment_details}

\subsection{Parameter Tuning}
The learning rate of all methods is tuned in $[10^{-2},10^{-5}]$. Weight decay
is tuned in $[10^{-3},10^{-5}]$. Specifically, $E_{k}$ for AUC-poly and AUC-exp is searched in $\{3, 5, 8, 10, 12, 15, 18, 20\}$. For AUC-poly, $\gamma$ is searched in $\{0.03, 0.05, 0.08, 0.1, 1, 3, 5\}$.
For AUC-exp, $\gamma$ is searched in $\{8, 10, 15, 20, 25, 30\}$. For SOPA-S, we tune the KL-regularization parameter $\lambda$ in $\{0.1, 1.0, 10\}$, and we fix $\beta_0 = \beta_1 = 0.9$. For PAUCI and UPAUCI, $k$ is tuned in $[1, 10]$, $\nu$, $\lambda$, $\iota_1$, $\iota_2$ are tuned in $[0, 1]$, $m$ is tuned in $[10, 100]$, $\kappa$ is tuned in $[2, 6]$ and $\omega$ is tuned in $[0, 4]$.

\subsection{Per-iteration Acceleration}
\label{section:convergence}

We conduct some experiments for per-iteration complexity with a fixed epoch with varying $n_+^B$ and $n_-^B$. All experiments are conducted on an Ubuntu 16.04.1 server with an Intel(R) Xeon(R) Silver 4110 CPU. For every method, we repeat running 10000 times and record the average running time. We only record the loss calculation time and use the python package time.time() to calculate the running time. Methods with * stand for the pair-wise estimator, while methods with ** stand for the instance-wise estimator. Here is the result of the experiment. We see the acceleration is significant when the data is large.

\begin{table}[!ht]
    \centering
    \caption{
    Pre-Iteration time complexity experiments for OPAUC ($\mathrm{FPR}\leq0.3$):
    }
    \begin{tabular}{ccccccc}
    \toprule
        unit:ms & $\begin{aligned}n_+^B=64\\n_-^B=64\end{aligned}$ & $\begin{aligned}n_+^B=128\\n_-^B=128\end{aligned}$ & $\begin{aligned}n_+^B=256\\n_-^B=256\end{aligned}$ & $\begin{aligned}n_+^B=512\\n_-^B=512\end{aligned}$ & $\begin{aligned}n_+^B=1024\\n_-^B=1024\end{aligned}$ & $\begin{aligned}n_+^B=2048\\n_-^B=2048\end{aligned}$ \\ \midrule
        SOPA* & 0.075 & 0.205 & 1.427 & 5.053 & 20.132 & 86.779 \\
        SOPA-S* & 0.063 & 0.165 & 0.946 & 4.003 & 15.815 & 62.031 \\
        AUC-poly* & 0.062 & 0.178 & 1.086 & 3.553 & 14.266 & 56.637 \\
        AUC-exp* & 0.063 & 0.182 & 0.985 & 3.513 & 14.155 & 55.689 \\
        AGD-SBCD* & 0.061 & 0.145 & 1.040 & 3.413 & 13.273 & 54.954 \\
        MB* & 0.121 & 0.174 & 0.468 & 1.713 & 6.393 & 25.663 \\
        PAUCI** & 0.026 & 0.029 & 0.033 & 0.043 & 0.072 & 0.107 \\
        AUC-M** & 0.025 & 0.028 & 0.031 & 0.040 & 0.059 & 0.104 \\
        CE** & 0.018 & 0.020 & 0.026 & 0.036 & 0.055 & 0.096 \\
    \bottomrule
    \end{tabular}
\end{table}

\begin{table}[!ht]
    \centering
    \caption{
    Pre-Iteration time complexity experiments for TPAUC ($\mathrm{FPR}\leq0.5,\mathrm{TPR}\geq0.5$):
    }
    \begin{tabular}{ccccccc}
    \toprule
        unit:ms & $\begin{aligned}n_+^B=64\\n_-^B=64\end{aligned}$ & $\begin{aligned}n_+^B=128\\n_-^B=128\end{aligned}$ & $\begin{aligned}n_+^B=256\\n_-^B=256\end{aligned}$ & $\begin{aligned}n_+^B=512\\n_-^B=512\end{aligned}$ & $\begin{aligned}n_+^B=1024\\n_-^B=1024\end{aligned}$ & $\begin{aligned}n_+^B=2048\\n_-^B=2048\end{aligned}$ \\ \midrule
        SOPA* & 0.079 & 0.206 & 1.439 & 5.197 & 20.556 & 88.314 \\
        SOPA-S* & 0.065 & 0.153 & 0.947 & 3.940 & 15.388 & 62.541 \\
        AUC-poly* & 0.062 & 0.180 & 1.175 & 3.573 & 14.440 & 56.469 \\
        AUC-exp* & 0.059 & 0.206 & 1.154 & 3.558 & 14.080 & 56.566 \\
        MB* & 0.173 & 0.198 & 0.491 & 1.955 & 6.554 & 29.369 \\
        PAUCI** & 0.030 & 0.030 & 0.038 & 0.045 & 0.071 & 0.109 \\
        AUC-M** & 0.025 & 0.027 & 0.033 & 0.043 & 0.059 & 0.104 \\
        CE** & 0.018 & 0.021 & 0.026 & 0.037 & 0.0535 & 0.096 \\
    \bottomrule
    \end{tabular}
\end{table}

\subsection{Sensitivity Analysis}\label{sec:add_sensi}
\begin{figure*}[t]
	\centering
    \subfloat[$a$]{\includegraphics[width=0.33\linewidth]{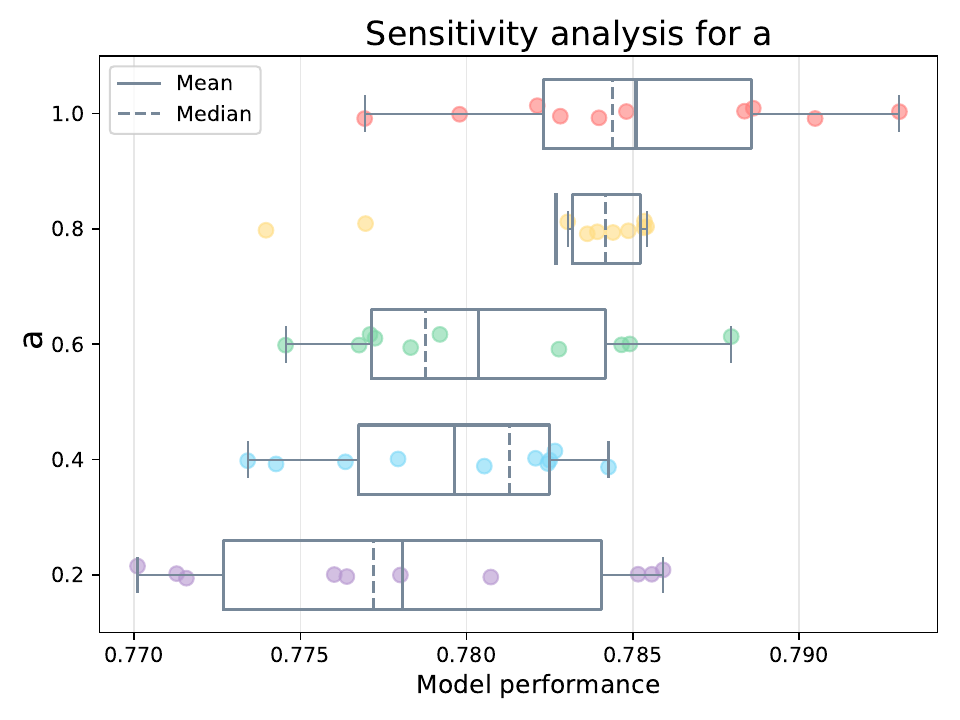}%
    }
    \subfloat[$b$]{\includegraphics[width=0.33\linewidth]{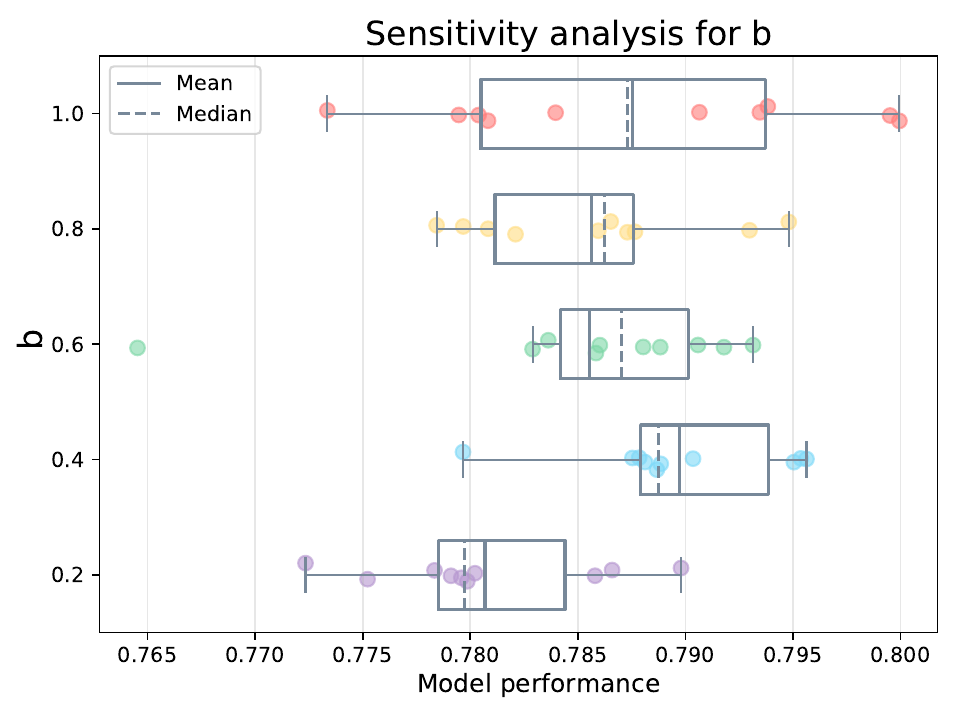}%
    }
    \subfloat[$\gamma$]{\includegraphics[width=0.33\linewidth]{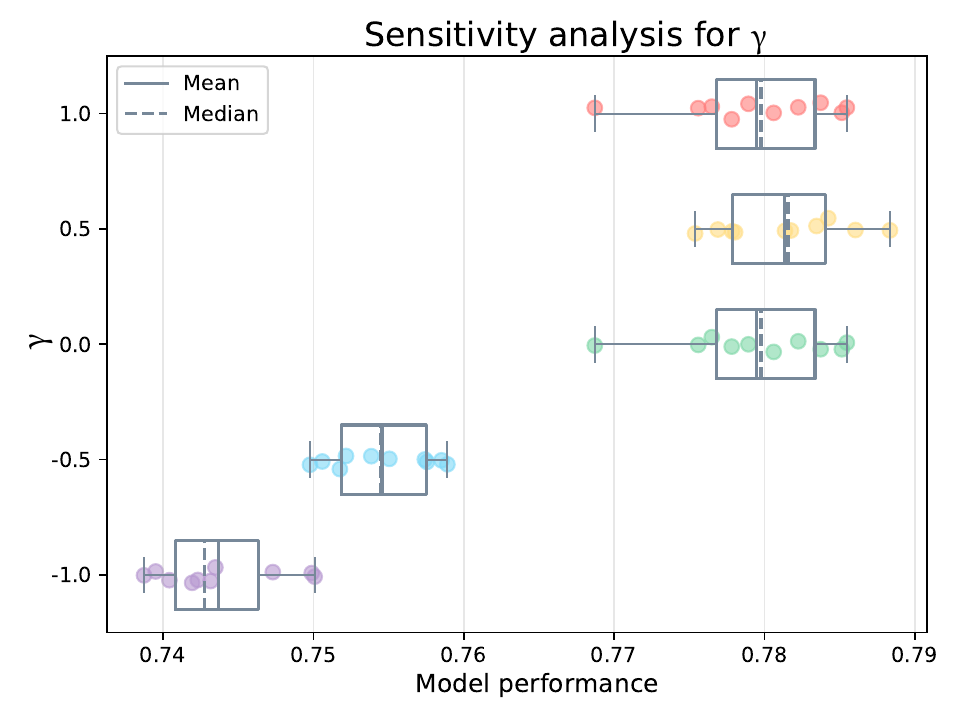}%
    }
    \\
    \subfloat[$s'$]{\includegraphics[width=0.33\linewidth]{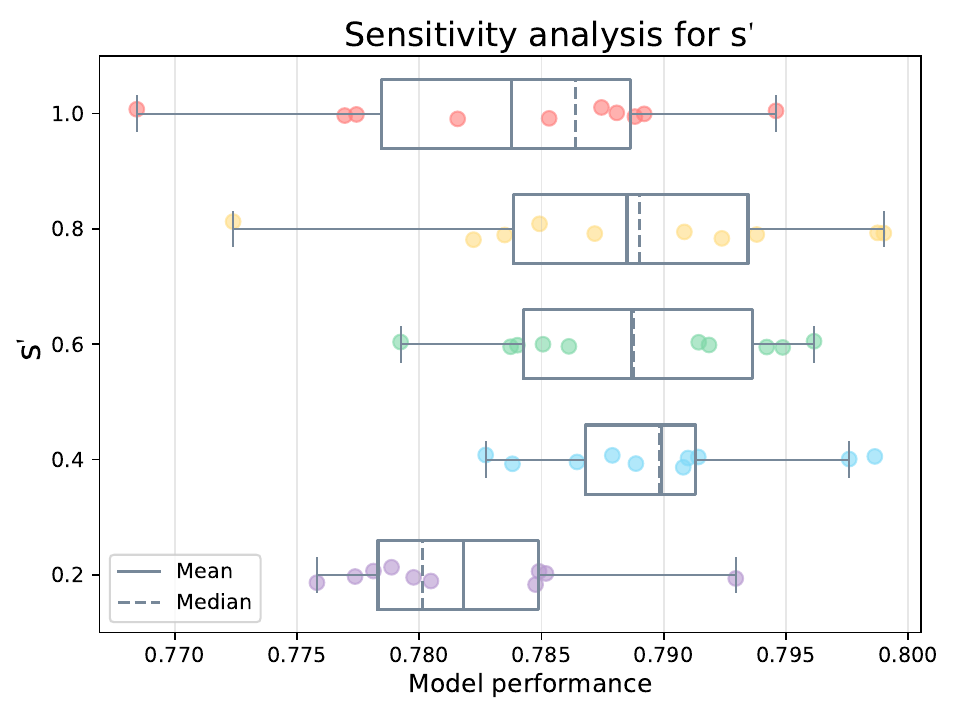}%
    }
    \subfloat[$c$]{\includegraphics[width=0.33\linewidth]{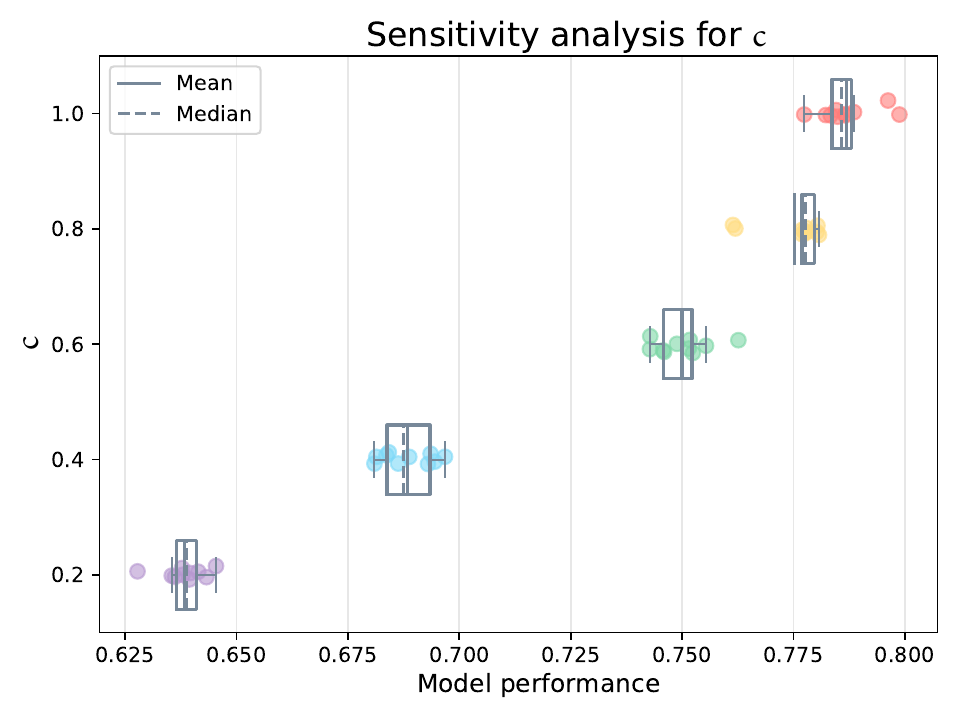}%
    }
	\caption{Sensitivity towards the variable initialization of UPAUCI over OPAUC ($\mathrm{FPR}\leq 0.3$) on CIFAR-10-LT-1.}
	\label{fig:sensi_init_op_0.3_cifar10-1}
\end{figure*}

\begin{figure*}[t]
	\centering
    \subfloat[$a$]{\includegraphics[width=0.33\linewidth]{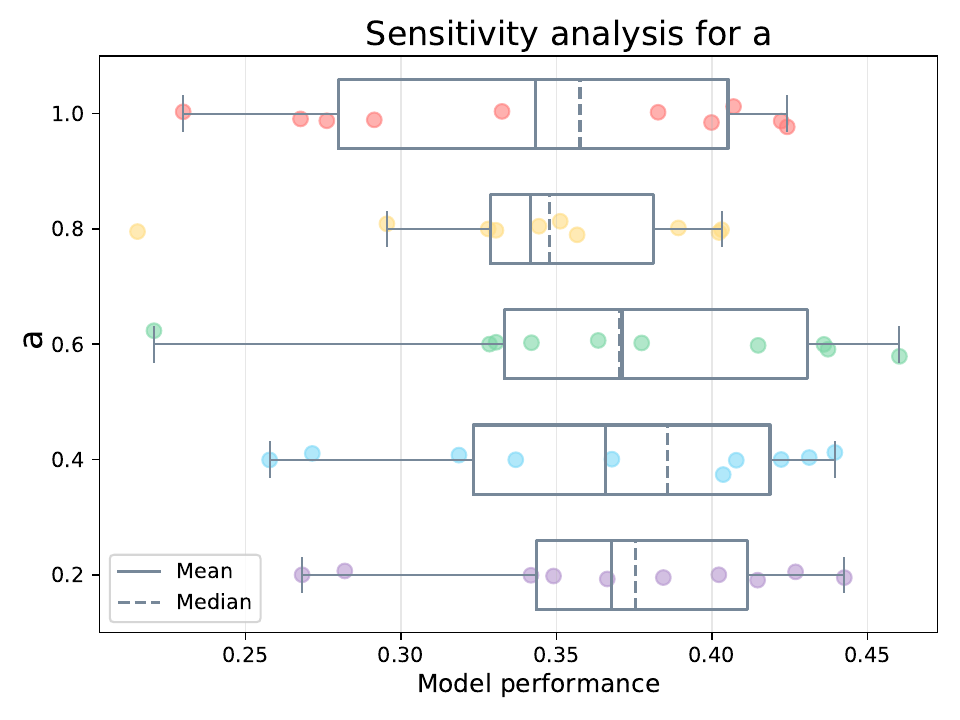}%
    }
    \subfloat[$b$]{\includegraphics[width=0.33\linewidth]{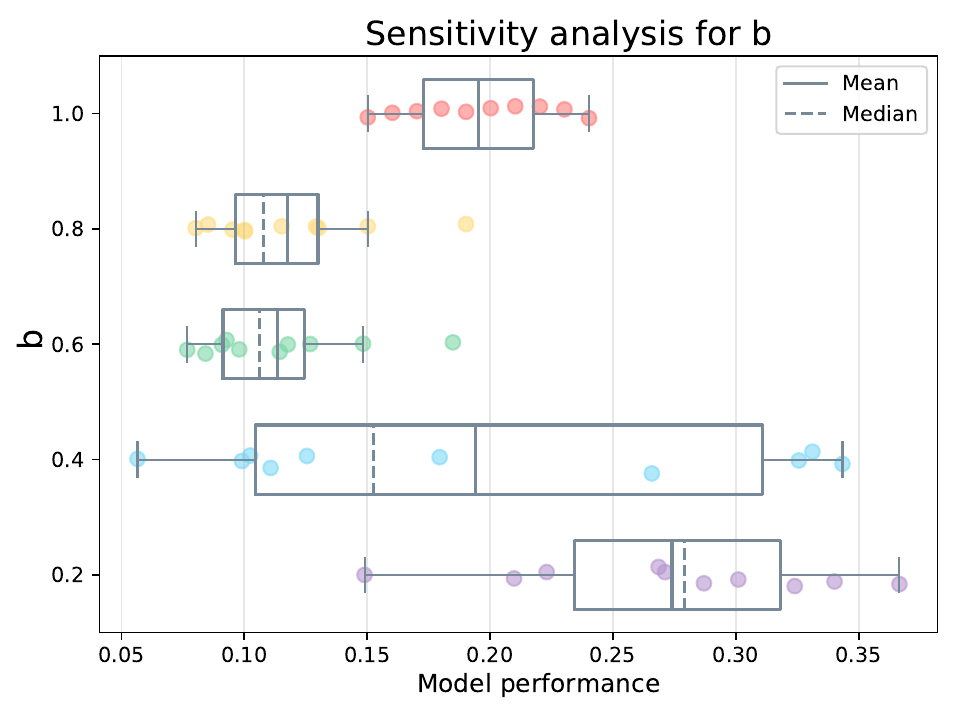}%
    }
    \subfloat[$\gamma$]{\includegraphics[width=0.33\linewidth]{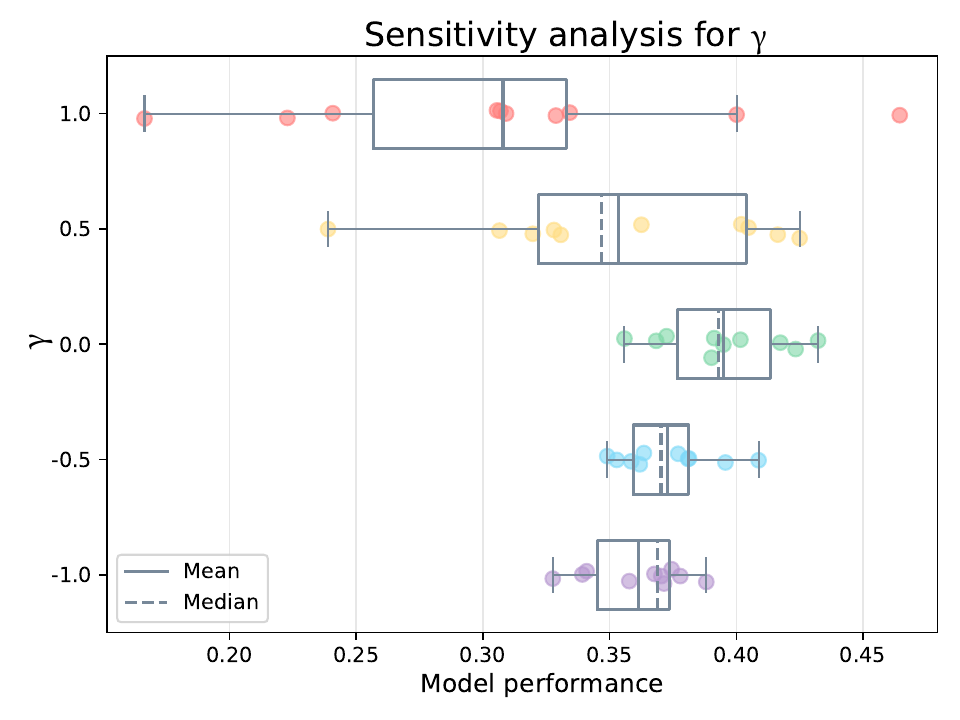}%
    }
    \\
    \subfloat[$s$]{\includegraphics[width=0.33\linewidth]{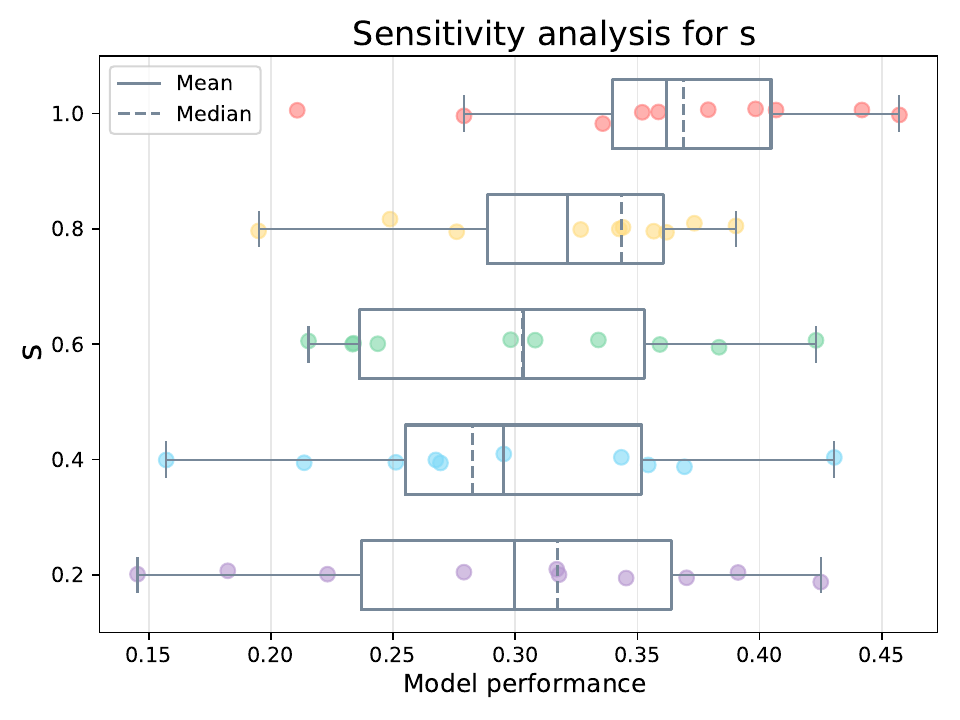}%
    }
    \subfloat[$s'$]{\includegraphics[width=0.33\linewidth]{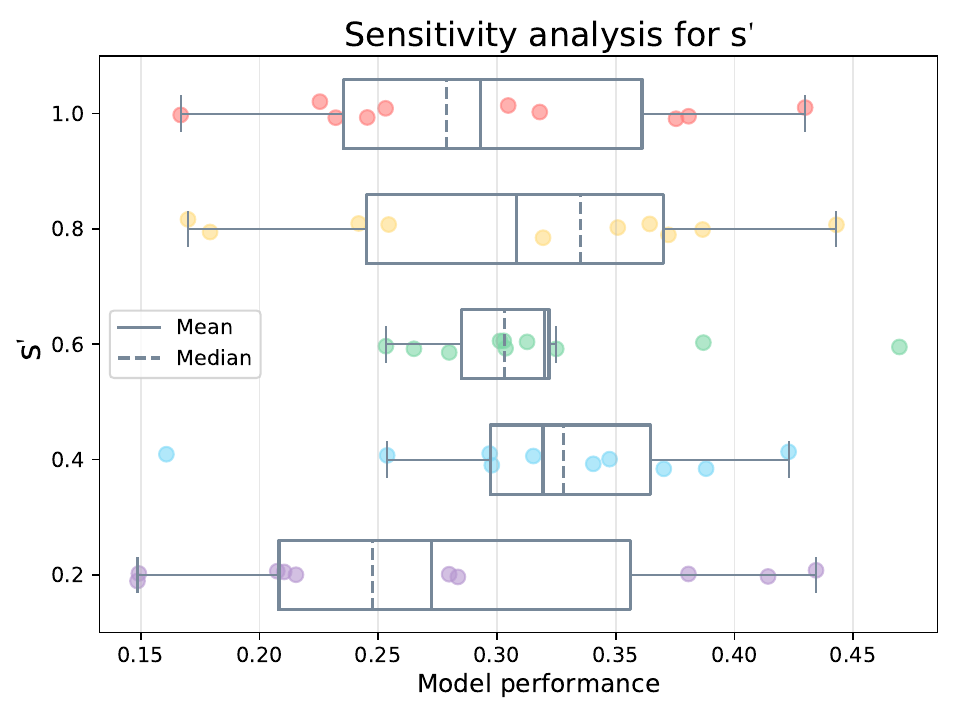}%
    }
    \subfloat[$c$]{\includegraphics[width=0.33\linewidth]{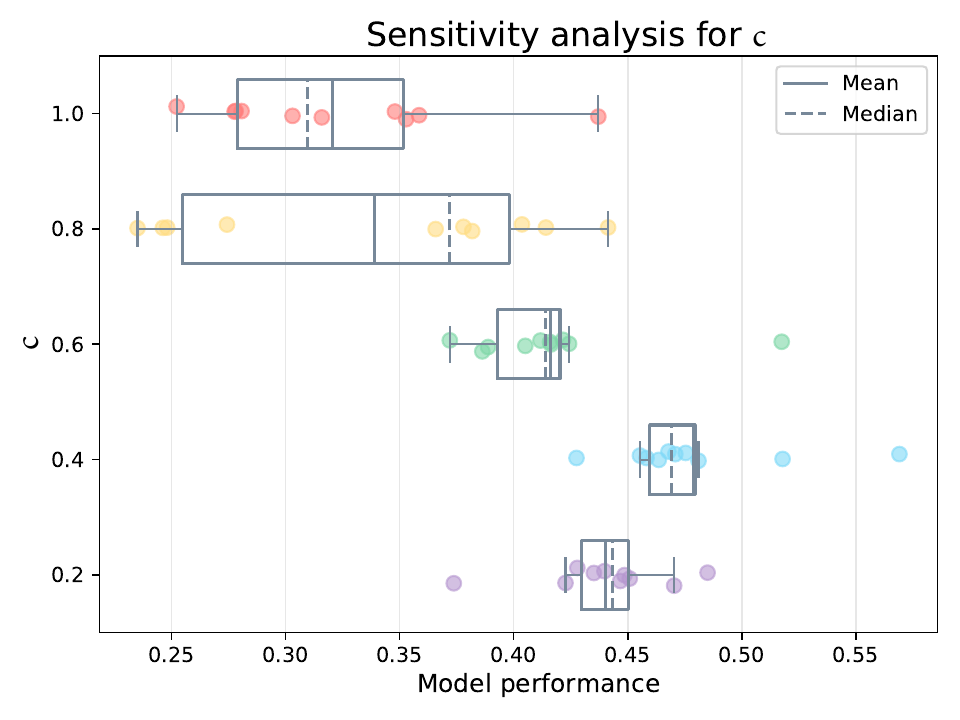}%
    }
	\caption{Sensitivity towards the variable initialization of UPAUCI over TPAUC ($\mathrm{TPR} \geq 0.3$, $\mathrm{FPR}\leq 0.3$) on CIFAR-10-LT-1.}
	\label{fig:sensi_init_tp_0.3_cifar10-1}
\end{figure*}

\end{document}